 \documentclass[final,3p,times]{elsarticle}
\usepackage{amsfonts}
\usepackage{graphicx}
\usepackage{times}
\usepackage{color}
\usepackage{amssymb}
\usepackage{natbib}
\usepackage{amsmath}
\usepackage{bm}
\usepackage[bbgreekl]{mathbbol}
\usepackage{breqn}
\usepackage{algorithm}
\usepackage{algpseudocode}

\font\ppppppcarac=ptmr8y at 4pt
\font\pppppcarac=ptmr8y at 5pt
\font\ppppcarac=ptmr8y at 6pt

\font\bf=ptmb8y at 10pt

\newcommand{\bfA}{{\bm{A}}}

\newcommand{\bfG}{{\bm{G}}}
\newcommand{\bfH}{{\bm{H}}}

\newcommand{\bfL}{{\bm{L}}}

\newcommand{\bfQ}{{\bm{Q}}}

\newcommand{\bfU}{{\bm{U}}}
\newcommand{\bfV}{{\bm{V}}}
\newcommand{\bfW}{{\bm{W}}}
\newcommand{\bfX}{{\bm{X}}}
\newcommand{\bfY}{{\bm{Y}}}
\newcommand{\bfZ}{{\bm{Z}}}

\newcommand{\bfzero}{{ \hbox{\bf 0} }}

\newcommand{\bfa}{{\bm{a}}}
\newcommand{\bfb}{{\bm{b}}}

\newcommand{\bff}{{\bm{f}}}
\newcommand{\bfg}{{\bm{g}}}
\newcommand{\bfh}{{\bm{h}}}

\newcommand{\bfq}{{\bm{q}}}

\newcommand{\bfu}{{\bm{u}}}
\newcommand{\bfv}{{\bm{v}}}
\newcommand{\bfw}{{\bm{w}}}
\newcommand{\bfx}{{\bm{x}}}
\newcommand{\bfy}{{\bm{y}}}
\newcommand{\bfz}{{\bm{z}}}

\newcommand{\bfgamma}{{\bm{\gamma}}}

\newcommand{\bfeta}{{\bm{\eta}}}

\newcommand{\bflambda}{{\bm{\lambda}}}

\newcommand{\bfvarphi}{{\bm{\varphi}}}

\newcommand{\bfxi}{{\bm{\xi}}}
\newcommand{\bfzeta}{{\bm{\zeta}}}

\newcommand{\bfchi}{{\bm{\chi}}}

\newcommand{\CC}{{\mathbb{C}}}

\newcommand{\MM}{{\mathbb{M}}}
\newcommand{\NN}{{\mathbb{N}}}

\newcommand{\RR}{{\mathbb{R}}}

\newcommand{\UU}{{\mathbb{U}}}
\newcommand{\VV}{{\mathbb{V}}}
\newcommand{\WW}{{\mathbb{W}}}
\newcommand{\XX}{{\mathbb{X}}}

\newcommand{\cc}{{\mathbb{c}}}

\DeclareMathAlphabet{\mathonebb}{U}{bbold}{m}{n}
\def\11{{\ensuremath{\mathonebb{1}}}}

\def\aa{{\ensuremath{\mathonebb{a}}}}
\def\bb{{\ensuremath{\mathonebb{b}}}}

\def\hh{{\ensuremath{\mathonebb{h}}}}

\DeclareMathAlphabet{\mathonebb}{U}{bbold}{m}{n}
\def\11{{\ensuremath{\mathonebb{1}}}}

\newcommand{\curB}{{\mathcal{B}}}
\newcommand{\curC}{{\mathcal{C}}}
\newcommand{\curD}{{\mathcal{D}}}

\newcommand{\curH}{{\mathcal{H}}}

\newcommand{\curL}{{\mathcal{L}}}

\newcommand{\curO}{{\mathcal{O}}}
\newcommand{\curP}{{\mathcal{P}}}

\newcommand{\curT}{{\mathcal{T}}}
\newcommand{\curU}{{\mathcal{U}}}
\newcommand{\curV}{{\mathcal{V}}}

\newcommand{\bfcurG}{{\boldsymbol{\mathcal{G}}}}

\newcommand{\bfcurM}{{\boldsymbol{\mathcal{M}}}}
\newcommand{\bfcurN}{{\boldsymbol{\mathcal{N}}}}
\newcommand{\bfcurO}{{\boldsymbol{\mathcal{O}}}}

\newcommand{\bfcurR}{{\boldsymbol{\mathcal{R}}}}

\newcommand{\bfcurY}{{\boldsymbol{\mathcal{Y}}}}

\newcommand{\ad}{{\hbox{{\ppppcarac ad}}}}
\newcommand{\pad}{{\hbox{{\pppppcarac ad}}}}
\newcommand{\st}{{\hbox{{\ppppcarac st}}}}
\newcommand{\SB}{{\hbox{{\pppppcarac SB}}}}
\newcommand{\eff}{{\hbox{{\ppppcarac eff}}}}
\newcommand{\peff}{{\hbox{{\ppppcarac eff}}}}

\newcommand{\pexp}{{\hbox{{\ppppcarac exp}}}}

\newcommand{\cov}{{\hbox{{\textrm cov}}}}

\newcommand{\pmax}{{\hbox{{\ppppcarac max}}}}

\newcommand{\tr}{{\hbox{{\textrm tr}}}}

\newcommand{\error}{{\hbox{{err}}}}

\newcommand{\psol}{{\hbox{{\ppppcarac sol}}}}
\newcommand{\ppsol}{{\hbox{{\ppppppcarac sol}}}}

\newcommand{\pML}{{\hbox{{\pppppcarac ML}}}}

\newcommand{\wien}{{\hbox{{\ppppcarac wien}}}}
\newcommand{\prelax}{{\hbox{{\ppppcarac relax}}}}

\newcommand{\bfHbflambdai}{{\bfH_{\!\bflambda^{\,i}}}}
\newcommand{\bulk}{{\hbox{{\ppppcarac bulk}}}}
\newcommand{\shear}{{\hbox{{\ppppcarac shear}}}}

\newdefinition{definition}{Definition}
\newtheorem{lemma}{Lemma}
\newtheorem{proposition}{Proposition}

\newproof{proof}{Proof}
\newdefinition{remark}{Remark}
\newdefinition{hypothesis}{Hypothesis}
\newdefinition{notation}{Notation}
\newproof{example}{Example}
\numberwithin{equation}{section}

\journal{ArXiv}

\begin{document}

\begin{frontmatter}
\title{Probabilistic learning inference of boundary value problem with uncertainties \\
based on Kullback-Leibler divergence under implicit constraints}

\author[1]{C. Soize \corref{cor1}}
\ead{christian.soize@univ-eiffel.fr}

\cortext[cor1]{Corresponding author: C. Soize, christian.soize@univ-eiffel.fr}
\address[1]{Universit\'e Gustave Eiffel, MSME UMR 8208 CNRS, 5 bd Descartes, 77454 Marne-la-Vall\'ee, France}

\begin{abstract}
In a first part, we present a mathematical analysis of a general methodology of a probabilistic learning inference that allows for estimating a posterior probability model for a stochastic boundary value problem from a prior probability model. The given targets are statistical moments for which the underlying realizations are not available. Under these conditions, the Kullback-Leibler divergence  minimum principle is used for estimating the posterior probability measure. A statistical surrogate model of the implicit mapping, which  represents the constraints, is introduced. The MCMC generator and the necessary numerical elements are given to facilitate the implementation of the methodology in a parallel computing framework.
In a second part, an  application is presented to illustrate the proposed theory and is also, as such, a contribution to the three-dimensional stochastic homogenization of heterogeneous linear elastic media in the case of a non-separation of the microscale and macroscale. For the construction of the posterior probability measure by using the probabilistic learning inference, in addition to the constraints defined by given statistical moments of the random effective elasticity tensor, the second-order moment of the random normalized residue of the stochastic partial differential equation has been added as a constraint. This constraint guarantees that the algorithm seeks to bring the statistical moments closer to their targets while preserving a small residue.
\end{abstract}


\end{frontmatter}

\section{Introduction}
\label{sec:Section1}
We consider a boundary value problem (BVP) with control parameters, for which the  observations (quantities of interest)  are defined by a given transformation of the solution of the BVP. The graph of the functional dependence between the parameters and the observations defines a manifold that is not explicitly described. We consider the statistical inverse problem consisting in identifying the parameters by giving targets for the observations. A prior probabilistic model of the control parameters is introduced what induces random observations. A training set, constituted of independent realizations belonging to the manifold, is constructed by solving the BVP. The high computational numerical cost to solve the stochastic BVP is a limitation requiring that only a small number of points can be computed for constructing the training set.

We consider a probabilistic learning inference for which the prior probability measure on the manifold is estimated with the points of the training set. The posterior probability model on the manifold is  constructed with respect to some statistical moments related to the observations. These statistical moments  are expressed with an implicit function $\bfh^c$ for which an algebraic representation is not available. The underlying realizations that have been used for estimating the targets of the statistical moments (such as those coming from experiments) are unknown (and consequently, are not available).
Consequently, the classical statistical tools to solve this statistical inverse problem, such as the Bayesian inference
(see \cite{Bernardo2000,Kennedy2001,Congdon2007,Carlin2008,Marin2012,Scott2016,Ghanem2017,Soize2020b} for general aspects and  \cite{Marzouk2007,Stuart2010,Soize2011,Matthies2016,Dashti2017,Arnst2017,Perrin2020} for specific aspects related to statistical inverse problems) or the maximum likelihood method (see \cite{Congdon2007,Carlin2008,Spall2005}) cannot easily be used.

The Kullback-Leibler divergence minimum principle \cite{Kullback1951,Kapur1992,Cover2006} is used for estimating the posterior probability measure on the manifold given its prior probability measure and the constraints related to the statistical moments for which targets are given.
It should be noted that this principle has widely been used in many fields (see for instance,
\cite{Kapur1992,Vasconcelos2004,Zhang2007,Cappe2013,Saleem2018}) in particular, for reinforcement learning \cite{Filippi2010} and for probabilistic learning \cite{Soize2020a,Soize2021a}). The posterior probability measure is represented by  an algebraic expression of the prior probability measure  and of a vector-valued  Lagrange multiplier $\bflambda$ associated with the vector-valued function $\bfh^c$. The optimal value $\bflambda^\psol$ of the Lagrange multipliers is obtained as the limit of a sequence $\{\bflambda^{\,i}\}_i$ of Lagrange multipliers allowing for constructing a sequence of probability measures whose limit, for $\bflambda=\bflambda^\psol$ is the searched posterior probability measure. In this paper, we will call "constrained learned set", the set of realizations that results from the learning process in taking into account the constraints.
Since the dimension can be high and the training set small, for each value $\bflambda^{\,i}$ of $\bflambda$, the constrained learned set must be generated with a MCMC algorithm. In order to decrease the numerical cost in the framework of a possible used of parallel computation, the MCMC generator is based on a nonlinear It\^o stochastic differential equation (ISDE) associated with a nonlinear stochastic dissipative Hamiltonian dynamical system. The presence of a dissipative term makes it possible to delete the transient part in order to quickly reach the stationary response associated with the invariant measure.
The evaluation of the drift vector of the ISDE requires to evaluate the gradient of function $\bfh^c$ a large number of times.
In the framework of this work, there is not an available algebraic expression of this gradient. Furthermore, it is not possible to do a direct numerical calculation of it, taking into account the high dimension (function $\bfh^c$ is implicit and for each evaluation of it, the BVP must be used).
For instance, components of $\bfh^c$ can be related to the norm of the random normalized residue of the partial differential equation (including the boundary conditions)  of the BVP and the statistical moments of random observations.
The construction of a surrogate model of implicit function $\bfh^c$ by using a deterministic approach, such as the meshless methods \cite{Nayroles1992,Belytschko1996,Duarte1996,Breitkopf2000,Rassineux2000,Zhang2000}, is not adapted  taking into account a possible high dimension of the space on which $\bfh^c$ is defined. Similarly, the construction of a representation on the chaos (Gaussian or another probability measure) \cite{Ghanem1991,Xiu2002,Soize2004a,Wan2006,Soize2010f,Blatman2011,Perrin2012,Tipireddy2014,Soize2015,Abraham2017,Luthen2021} would not be at all effective in our case for the same reasons related to the possible high dimension.
To circumvent this difficulty, we generalize the approach proposed in \cite{Soize2021a}, which consists in constructing a statistical surrogate model $\hat\bfh^N$ of $\bfh^c$, depending on the number $N$ of points generated in the constrained learned set, for which its gradient has an explicit algebraic representation.
\\

\textit{(ii) Organization and novelties of the paper}. First of all, let us  point out that a neighboring problem has been tackled in \cite{Soize2020a} devoted to take into account constraints in the PLoM (probabilistic learning on manifolds) method  \cite{Soize2016,Soize2020c,Soize2021b}. However, in \cite{Soize2020a}, function $\bfh^c$ is explicit. In this work, $\bfh^c$ is implicit (that is to say the use of the BVP is required for evaluating its value in any given point). The presented methodology is novel and general (note that PLoM is not used, but could be implemented if necessary, but this would be prejudicial to the clarity of the developments). In addition, a mathematical analysis of the methodology is presented. This analysis  is necessary because, due to the use of the statistical surrogate model $\hat\bfh^N$ of $\bfh^c$, the constrained learned realizations of the posterior probability measure on the manifold are generated with the MCMC generator that is the limit of a sequence of MCMC generators whose each one depends on $\hat\bfh^N$ instead of $\bfh^c$. This means that convergence properties with respect to $N$ must be studied.
The organization of the paper is the following.

In Section~\ref{sec:Section2}, we set the problem and summarize the methodology of the probabilistic learning inference that is proposed. The hypotheses used are discussed in order to exhibit the difficulties involved by the choice of a general framework for the developments: small number of points in the training set, targets defined by statistical moments of the observations, implicit description of function $\bfh^c$ and the necessity to construct a statistical surrogate model.

Section~\ref{sec:Section3} deals with the mathematical analysis of the proposed methodology.
Lemma~\ref{lemma:1} proves the convexity of the admissible set $\curC_{\pad,\bflambda}\subset\RR^{n_c}$ of the Lagrange multipliers $\bflambda$ and the integrability properties related to the sequence of posterior probability measures indexed by $\bflambda$. Proposition~\ref{proposition:1} gives the construction of $\bflambda^\psol$ as the unique solution of a convex optimization problem posed in $\curC_{\pad,\bflambda}$.
After building the statistical surrogate model $\hat\bfh^N$ of $\bfh^c$ (Definition~\ref{definition:3}),  Proposition~\ref{proposition:2} proves the convergence of the sequence $\{\hat\bfh^N\}_N$ towards $\bfh^c$ and of its gradient for $\bflambda$ fixed in $\curC_{\pad,\bflambda}$.
Proposition~\ref{proposition:3} proves the existence and uniqueness of the invariant measure of the ISDE for each fixed value of $\bflambda$ and gives the generator of the constrained learned set depending on $\bflambda$.
Proposition~\ref{proposition:4} proves the mean-square convergence, as $N$ goes to infinity, of the stationary stochastic solution of the ISDE calculated with $\hat\bfh^N$ towards the one calculated with $\bfh^c$.
After having given the construction of the iteration algorithm for calculating the sequence $\{\bflambda^{\,i}\}_i$ that converges to $\bflambda^\psol$, Proposition~\ref{proposition:5} gives the rate of convergence of the sequence of posterior probability measures indexed by $\bflambda$.

In Section~\ref{sec:Section4}, a few numerical elements are given for implementing the methodology.
The St\"ormer-Verlet scheme that allows for solving the ISDE is detailed and the explicit expression of the gradient of $\hat\bfh^N$ is given. We present the algorithm for calculating $\bflambda^\psol$ and for generating the constrained learned set $\curD_{\bfH^c}$.

Section~\ref{sec:Section5} is devoted to stochastic homogenization that has given rise to a large number of works (see for instance \cite{Papanicolaou1981,Torquato1985,Nguetseng1989,Allaire1992,Sab1992,Andrews1998,Ostoja1998,Jikov2012}) and for which the analysis of the representative volume element (RVE) size has received a particular attention (see \cite{Drugan1996,Ren2004,Sab2005,Ostoja2006,Ostoja2007,Yin2008,Soize2008,Zhang2020,Soize2021c}).
In this paper, we present an application related to the stochastic homogenization of a random linear elastic medium at mesoscale posed on the domain $\Omega\subset \RR^3$ of the microstructure, which is not a RVE, which means that there is no scale separation between the mesoscale and the macroscale. Consequently, the effective elasticity tensor at macroscale is not deterministic and has statistical fluctuations. This case is obtained when the spatial correlation lengths of the random apparent elasticity field at mesoscale are not sufficiently small with respect to the size of domain $\Omega$. For instance, such a situation is encountered when the size of an experimental specimen, which is measured with the objective to perform an inverse identification of its effective properties, is not sufficiently large compared to the size of the heterogeneities. It is then interesting to identify the probabilistic model of the random apparent elasticity field, without having a scale separation. One can then estimate the effective properties using the identified random apparent elasticity field with a stochastic computational model for which domain $\Omega$ is an RVE, that is to say, has a size that is  largest than the one of the specimen. In order to analyze the scale separation, we will consider several values of the spatial correlation lengths to cover three cases:

\textit{SC1}: partial separation (separated for two directions but not in the third one).

\textit{SC2}: not separated in the three directions.

\textit{SC3}: not strongly separated in the three directions.

\bigskip

\noindent{\textbf{Notations}

\noindent $x,\eta$: lower-case Latin or Greek letters are deterministic real variables.\\
$\bfx,\bfeta$: boldface lower-case Latin or Greek letters are deterministic vectors.\\
$X$: upper-case Latin letters are real-valued random variables.\\
$\bfX$: boldface upper-case Latin letters are vector-valued random variables.\\
$[x]$: lower-case Latin letters between brackets are deterministic matrices.\\
$[\bfX]$: boldface upper-case letters between brackets are matrix-valued random variables.\\
$\curC_{\ad,\bflambda}$: admissible set of $\bflambda\in\RR^{n_c}$.\\
$\CC$: fourth-order tensor-valued random field.\\
$\curD_d$: training set.\\
$\curD_{\bfH^c}$: constrained learned set at convergence for $\bflambda = \bflambda^\psol$.\\
$\curD_{\bfHbflambdai}$: constrained learned set for $\bflambda^{\,i}$.\\
$N$: number of points in the constrained learned set.\\
$N_d$: number of points in the training set.\\
$\NN$, $\RR$: set of all the integers $\{0,1,2,\ldots\}$, set of all the real numbers.\\
$\RR^n$: Euclidean vector space on $\RR$ of dimension $n$.\\
$\MM_{n,m}$: set of all the $(n\times m)$ real matrices.\\
$\MM_n$: set of all the square $(n\times n)$ real matrices.\\
$\MM_n^+$: set of all the positive-definite symmetric $(n\times n)$ real matrices.\\
$[I_{n}]$: identity matrix in $\MM_n$.\\
$\bfx = (x_1,\ldots,x_n)$: point in $\RR^n$.\\
$\langle \bfx,\bfy \rangle = x_1 y_1 + \ldots + x_n y_n$: inner product in $\RR^n$.\\
$\Vert\,\bfx\,\Vert$:  norm in $\RR^n$ such that $\Vert\,\bfx\,\Vert = \langle \bfx,\bfx \rangle$.\\
$[x]^T$: transpose of matrix $[x]$.\\
$\tr \{[x]\}$: trace of the square matrix $[x]$.\\
$\Vert\, [x]\, \Vert_F$: Frobenius norm of matrix  $[x]$.\\
$\delta_{kk'}$: Kronecker's symbol.\\
$\delta_{\bfx_0}$: Dirac measure at point $\bfx_0$.\\
$a.s.$: almost surely.\\
BVP: boundary value problem.\\
$det$: determinant.\\
dof: degree of freedom.\\
$E$: mathematical expectation operator.\\
$\error$: error function.\\
ISDE: It\^o stochastic differential equation.\\
KDE: kernel density estimation.\\
pdf: probability density function.\\
PDE: partial differential equation.\\
\section{Setting the problem and summarizing the methodology}
\label{sec:Section2}
In this paper, all the random variables are defined on a probability space $(\Theta,\curT,\curP)$ in which $\Theta$ is the sample set, $\curT$ is the $\sigma$-field of $\Theta$, and where $\curP$ is a probability measure on the measurable space $(\Theta,\curT)$. A "sample" $\bfX(\theta), \theta\in\Theta$ of a random variable $\bfX$ defined on $(\Theta,\curT,\curP)$ will also be called a "realization" of $\bfX$.
\subsection{Framework of the developments presented in the paper}
\label{sec:Section2.1}
For instance, we consider a stochastic elliptic BVP on an open bounded domain  $\Omega\subset\RR^d$ with $d \geq 1$, whose partial differential equation (PDE) is written as $\bfcurN(\bfY,\bfG,\bfW) = \bfzero\,\, a.s$. The unknown is the non-Gaussian vector-valued field $\{\bfY(\bfxi),\bfxi\in\Omega\}$ defined on  $(\Theta,\curT,\curP)$ and which satisfies the boundary conditions.
The coefficients of the stochastic elliptic operator depend on a non-Gaussian second-order vector-valued random field $\bfG$ defined on $(\Theta,\curT,\curP)$ and on a random vector-valued control parameter $\bfW$ also defined on $(\Theta,\curT,\curP)$. It is assumed that the weak formulation of this stochastic BVP admits a unique strong stochastic solution $\bfY = \bff(\bfG,\bfW)$, which is a second-order random field. The observation (quantity of interest) is, for instance, a second-order vector-valued random variable $\bfQ = \bfcurO(\bfY,\bfG,\bfW)$ in which $\bfcurO$ is a given measurable mapping.

The problem under consideration belongs to the class of the statistical inverse problems. A prior probability model of $\{\bfG,\bfW\}$ is given and we are interested in estimating a posterior model $\{\bfG^c,\bfW^c\}$ of $\{\bfG,\bfW\}$ in order that some statistical moments of the posterior observations $\bfQ^c = \bfcurO(\bfY^c,\bfG^c,\bfW^c)$ with $\bfY^c = \bff(\bfG^c,\bfW^c)$, are equal to some given targets
(the superscript "$c$" is introduced to designate the solution with the constraints, which corresponds to the posterior model). The statistical moments of $\bfQ^c$ are globally written as
$E\{\bfcurM^c(\bfQ^c)\} = \bfb^c$ in which $\bfb^c\in\RR^{n_c}$ is the target, $\bfq\mapsto \bfcurM^c(\bfq)$ is a given measurable mapping, and $E$ is the mathematical expectation operator.  As explained in Section~\ref{sec:Section1}, the realizations that have allowed $\bfb^c$ to be estimated  are not available. Only $\bfb^c$ is known (it can be the case when one or several components of $\bfb^c$ have been estimated with experimental realizations that are not available). However, as we will see in Section~\ref{sec:Section5}, a component of the vector of statistical moments can also be related to the random normalized residue of the PDE of the stochastic boundary value problem.
\subsection{Hypotheses concerning the problem to be solved}
\label{sec:Section2.2}
\textit{(i) Small dimension $N_d$ of the training set}. The subscript "$d$" is introduced to designate the quantities related to the training set that is constructed by using the Monte Carlo numerical simulation method.  Let $\{ \bfg_d^1,\ldots , \bfg_d^{N_d}\}$ and $\{ \bfw_d^1,\ldots , \bfw_d^{N_d}\}$ be $N_d$ independent realizations of random variable $\{\bfG ,\bfW\}$, generated by using the prior probability model of $\{\bfG,\bfW\}$. Each realization of the BVP defined as the PDE
$\bfcurN(\bfy_d^j,\bfg_d^j,\bfw_d^j) =\bfzero$ with its boundary conditions, is solved. Consequently,   $N_d$ independent realizations $\{\bfy_d^j , j=1,\ldots, N_d\}$ of random field $\bfY$ are computed and are such that $\bfy_d^j=\bff(\bfg_d^j,\bfw_d^j)$.
The $N_d$ independent realizations $\{\bfq_d^j, j=1,\ldots, N_d\}$ of random observation $\bfQ$ are thus deduced such that
$\bfq_d^j = \bfcurO(\bfy_d^j,\bfg_d^j,\bfw_d^j)$.
The training set is then made up of a small number $N_d$ of points $\bfx_d^j = \{\bfy_d^j ,\bfg_d^j,\bfw_d^j \}$ for $j=1,\ldots, N_d$, which are $N_d$ independent realizations of $\bfX = \{\bfY,\bfG,\bfW\}$. Note that a strong hypothesis in the present work is that the BVP can only be solved a small number of times. This means that the training set is a small data set (as opposed to a big data set).
Consequently, for constructing the posterior probability measure on the manifold, a learning tool must be used for generating the constrained learned realizations of $\bfX$ without solving the BVP, but using only the training set.

\textit{(ii) Statistical inverse problem}. The available information consists of the targets  of statistical moments related to observation $\bfQ^c$ and represented by $\bfb^c$. As we have explained in Section~\ref{sec:Section1}, the realizations that have been used for estimating $\bfb^c$ are not available. This is the second strong hypothesis used in this work. Consequently the classical statistical tools such as the Bayesian inference or the maximum likelihood method cannot easily be used for solving the statistical inverse problem under consideration.

\textit{(iii) Finite reduced-order representation}. The second-order random variable $\bfX=\{\bfY,\bfG,\bfW\}$, defined on $(\Theta,\curT,\curP)$, is assumed to be with values in a real Hilbert space $\XX$ equipped with the inner product $\langle \bfX\, ,\bfX' \rangle_\XX$ and its associated norm $\Vert\,\bfX\,\Vert_\XX = \langle \bfX\, ,\bfX \rangle_\XX^{1/2}$. Consequently, $\bfX$ belongs the Hilbert space $L^2(\Theta,\XX)$ of the equivalent class of all the second-order random variables with values in $\XX$, equipped with the inner product $\langle\langle \bfX\, ,\bfX' \rangle\rangle = E\{\langle \bfX\, ,\bfX' \rangle_\XX\}$ for which the square of the associated norm is $|||\,\bfX\,|||^2 = E\{\Vert\,\bfX\,\Vert_\XX^2\} = \int_\Theta \Vert\,\bfX(\theta)\,\Vert_\XX^2\, d\,\curP(\theta)$.
Since the problem is in infinite dimension, in order to implement the probabilistic learning inference that we propose, we need to introduce a finite representation $\bfX^{(\nu)}$  of dimension $\nu$ of random variable $\bfX$ in $L^2(\Theta,\XX)$. Assuming that the covariance operator is a Hilbert-Schmidt \cite{Guelfand1967}, symmetric, positive operator in $\XX$, $\bfX^{(\nu)}$ can be represented using the truncated Karhunen-Lo\`eve expansion \cite{Karhunen1947,Loeve1948} of $\bfX$,
\begin{equation}\label{eq:eq1}
\bfX^{(\nu)} = \underline\bfx + \sum_{\alpha=1}^\nu \sqrt{\kappa_\alpha} \, \bfvarphi^\alpha \, H_\alpha \, ,
\end{equation}
in which the eigenvalues of the covariance operator are $\kappa_1 \geq \ldots \geq \kappa_\nu \geq \ldots = 0$ with $\sum_{\alpha=1}^{+\infty} \kappa_\alpha^2 < +\infty$,
where the family of the eigenfunctions $\{\bfvarphi^\alpha\}_\alpha$ is a Hilbertian basis of $\XX$, where $\underline\bfx = E\{\bfX\}$, and where $\bfH=(H_1,\ldots, H_\nu)$ is a second-order, centered, $\RR^\nu$-valued random variable whose covariance matrix is the identity matrix $[I_\nu]$ in $\MM_\nu$. It should be noted that $\bfcurN$ being a deterministic mapping, the realizations of $\bfX$ belong to a manifold and consequently, the covariance operator of $\bfX$ is not positive definite  but only positive (its kernel is not reduced to zero) and therefore, there is a zero eigenvalue with a finite multiplicity, which is not taken into account in the truncated representation of $\bfX$, defined by Eq.~\eqref{eq:eq1}, in which $\kappa_\nu > 0$.
For $\alpha\in\{1,\ldots , \nu\}$, the component $H_\alpha$ is written as
$H_\alpha  = \kappa_\alpha^{-1/2}\, \langle\bfX-\underline\bfx\, , \bfvarphi^\alpha \rangle_\XX$.
The training set $\curD_d$ related to $\bfH$ is made up of the $N_d$ independent realizations $\{\bfeta_d^j,j=1,\ldots, N_d\}$ such that
$\bfeta_d^j = \kappa_\alpha^{-1/2}\, \langle \bfx_d^j -\underline\bfx\, , \bfvarphi^\alpha \rangle_\XX$.
If the kernel of the covariance operator was explicitly known, then $\nu$ would be chosen in order that
$|||\,\bfX-\bfX^{(\nu)}\,||| \,\, \leq \varepsilon$ for a sufficiently small value of $\varepsilon$. In this paper, we assume that the kernel of the covariance operator is unknown. Therefore, we can only obtain an approximation of the covariance operator using an empirical estimator built with the $N_d$ points $\{\bfx_d^j, j=1,\ldots, N_d\}$ (for instance, see Section~\ref{sec:Section5.6}). Under these conditions the largest value of $\nu$ will be $N_d-1$ and the discretization of Eq.~\eqref{eq:eq1} will simply correspond to a normalization of the points that constitute the training set $\curD_d$.
\subsection{Formulation using the Kullback-Leibler divergence minimum principle}
\label{sec:Section2.3}
Taking into account Section~\ref{sec:Section2.2}-(ii), we use the Kullback-Leibler divergence minimum principle  \cite{Kullback1951,Kapur1992,Cover2006} for estimating the posterior probability measure $P_{\bfH^c}(d\bfeta) = p_{\bfH^c}(\bfeta)\, d\bfeta$ on $\RR^\nu$ of the $\RR^\nu$-valued random variable $\bfH^c=(H^c_1,\ldots , H^c_\nu)$. This estimation of $P_{\bfH^c}$ is performed (1) using the prior  probability measure $P_{\bfH}(d\bfeta) = p_{\bfH}(\bfeta)\, d\bfeta$ on $\RR^\nu$ in which $p_\bfH$ is estimated with the Gaussian KDE method using the points $\{\bfeta_d^1,\ldots , \bfeta_d^{N_d}\}$ of the training set $\curD_d$, and (2) using the constraint defined by the given statistical moments,
\begin{equation}\label{eq:eq2}
E\{\bfh^c(\bfH^c)\} = \bfb^c \in\RR^{n_c} \, .
\end{equation}
In Eq.~\eqref{eq:eq2}, $\bfeta\mapsto\bfh^c(\bfeta)$ is a function from $\RR^\nu$ into $\RR^{n_c}$ with $n_c > 1$, such that $\bfh^c(\bfH^c)$ is a random variable equal to $\bfcurM^c(\bfQ^c)$ that is expressed as a function of $\bfH^c$.
We recall that the Kullback-Leibler divergence between two probability measures $p(\bfeta)\, d\bfeta$ and $p_\bfH(\bfeta)\, d\bfeta$ on $\RR^\nu$ is defined by
$D(p,p_\bfH) = \int_{\RR^\nu} p(\bfeta)\, \log(p(\bfeta)/p_\bfH(\bfeta)) \, d\bfeta$,
and is such that $D(p,p_\bfH) \geq 0$ (that can be proven by applying the Jensen inequality \cite{Jensen1906,Durrett2019})
and $D(p,p_\bfH)= 0$ if and only if $p=p_\bfH$. Note that $(p,p_\bfH)\mapsto D(p,p_\bfH)$ is not a distance because the symmetry property and the triangle inequality are not verified. It can easily be seen that the cross entropy
$S(p,p_\bfH)=-\int_{\RR^\nu} p(\bfeta)\, \log(p_\bfH(\bfeta))\, d\bfeta$ and the entropy
$S(p)=-\int_{\RR^\nu} p(\bfeta)\, \log(p(\bfeta))\, d\bfeta$ are related to $D(p,p_\bfH)$ by
$S(p,p_\bfH)=S(p)+D(p,p_\bfH)$. Finally, it can be proven (see for instance \cite{Cover2006}) that $(p,p_\bfH)\mapsto D(p,p_\bfH)$ is a convex function in the pair $(p,p_\bfH)$.
The probability density function $p_{\bfH^c}$ on $\RR^\nu$, which satisfies the constraint defined by Eq.~\eqref{eq:eq2} and which is closest to $p_\bfH$, is the solution of the optimization problem (see for instance \cite{Kapur1992,Cover2006,Soize2020a}),
\begin{equation}\label{eq:eq3}
p_{\bfH^c} = \arg\,\min_{p\in\curC_{\pad,p}} \int_{\RR^\nu} p(\bfeta)\, \log\left (\frac{p(\bfeta)}{p_\bfH(\bfeta)}\right) \, d\bfeta \, ,
\end{equation}
in which the admissible set $\curC_{\ad,p}$ is defined by
\begin{equation}\label{eq:eq4}
\curC_{\ad,p} = \left\{\bfeta\mapsto p(\bfeta):\RR^\nu\rightarrow \RR^+ \, ,\int_{\RR^\nu} p(\bfeta)\, d\bfeta =1\, ,
\int_{\RR^\nu} \!\bfh^c(\bfeta)\, p(\bfeta)\, d\bfeta = \bfb^c \right\} \, .
\end{equation}
It should be noted that, since we are interested in divergence (from $p_\bfH(\bfeta)\, d\bfeta$) of  probability measure
$p(\bfeta)\, d\bfeta$ that satisfies the constraints expressed in $\curC_{\ad,p}$, we are in fact interested in the divergence and therefore, the symmetry condition is irrelevant in this case.
\subsection{Methodology used for solving the optimization problem and MCMC generator}
\label{sec:Section2.4}
The constraints, which are defined in  admissible set $\curC_{\ad,p}$, are taken into account by introducing the Lagrange multipliers $\lambda_0-1$ with $\lambda_0\in\RR^+$, which is associated with the normalization condition,
and $\bflambda\in\curC_{\ad,\bflambda}\subset\RR^{n_c}$, which is associated with the moments constraints. The admissible set
$\curC_{\ad,\bflambda}$ of $\bflambda$ is a subset of $\RR^{n_c}$, which is completely defined by Definition~\ref{definition:2} in Section~\ref{sec:Section3}.
The Lagrange multiplier  $\lambda_0$ is eliminated as a function of $\bflambda$. In Eq.~\eqref{eq:eq3}, the posterior pdf $p_{\bfH^c}$  is constructed as the limit of a sequence $\{p_{\bfH_\bflambda}\}_\bflambda$ of probability density functions of a $\RR^\nu$-valued random variable
$\bfH_\bflambda = (H_{\bflambda, 1}, \ldots , H_{\bflambda, \nu})$ that depends on $\bflambda$. The construction of  $\{p_{\bfH_\bflambda}\}_\bflambda$  requires to generate with a MCMC algorithm a constrained learned set $\curD_{\bfH_\bflambda} = \{\bfeta^1_\bflambda,\ldots \bfeta^N_\bflambda\}$ constituted of $N\gg N_d$ independent realizations $\{\bfeta_\bflambda^\ell, \ell=1,\ldots , N\}$ of $\bfH_\bflambda$. When the convergence is reached with respect to $\bflambda$, the constrained learned set $\curD_{\bfH^c} = \{\bfeta_c^1,\ldots, \bfeta_c^N\}$ is generated. This set is made up of $N$ independent realizations $\{\bfeta_c^\ell, \ell=1,\ldots , N\}$ of $\bfH^c$ whose probability measure is $p_{\bfH^c}(\bfeta)\, d\bfeta$ (subscript or superscript "$c$" is introduced to designate the quantities related to the constrained learned set (posterior model)). For $\bflambda$ fixed in $\curC_{\ad,\bflambda}$, $\curD_{\bfH_\bflambda}$ is generated using a MCMC generator based on a nonlinear It\^o stochastic differential equation (ISDE) associated with the nonlinear stochastic dissipative Hamiltonian dynamical system proposed in \cite{Soize2008b} and based on \cite{Soize1994}. This MCMC generator allows for deleting the transient part to rapidly reach the stationary response associated with the invariant measure for which measure $p_{\bfH^c}(\bfeta)\, d\bfeta$ is a marginal  measure of this invariant measure. The ISDE is solved by using the St\"ormer-Verlet algorithm, which yields an efficient and accurate MCMC algorithm. This algorithm can then easily be parallelized for strongly decreasing the elapsed time on a multicore computer.

\subsection{Formal construction of the optimal solution using the sequence $\{p_{\bfH_\bflambda}\}_\bflambda$}
\label{sec:Section2.5}
Let us assumed that the optimization problem defined by Eq.~\eqref{eq:eq3} has almost one solution $p_{\bfH^c}$ and that $p=p_{\bfH^c}$ is a regular point of the continuously differentiable functional $p\mapsto \int_{\RR^\nu} \bfh^c(\bfeta)\, p(\bfeta)\, d\bfeta - \bfb^c$. For $\lambda_0\in\RR^+$ and $\bflambda\in\curC_{\ad,\bflambda}$, we define the Lagrangian,
\begin{equation}\nonumber 
\curL ag(p,\lambda_0,\bflambda) = \!\int_{\RR^\nu} p(\bfeta)\, \log\left (\frac{p(\bfeta)}{p_\bfH(\bfeta)}\right)  d\bfeta
                                 + (\lambda_0 - 1)\,(\! \int_{\RR^\nu}\! p(\bfeta)\, d\bfeta -1)
                                 + \langle \bflambda \, ,\! \int_{\RR^\nu} \!\bfh^c(\bfeta)\, p(\bfeta)\, d\bfeta - \bfb^c \rangle \, .
\end{equation}
For all $\bfeta$ in $\RR^\nu$, the pdf of $\bfH$ is written as
$p_\bfH(\bfeta) = c_\nu\, \zeta(\bfeta)$
in which the positive-valued function $\zeta$ is integrable on $\RR^\nu$, is such that $\hbox{supp}\, \zeta = \RR^\nu$, and where
the positive constant $c_\nu$ is such that $c_\nu^{-1} = \int_{\RR^\nu} \zeta(\bfeta)\, d\bfeta$. We define the sequence $\{p_{\bfH_\bflambda}\}_\bflambda$ of pdf $\bfeta\mapsto p_{\bfH_\bflambda}(\bfeta\,;\bflambda)$ on $\RR^\nu$, indexed by $\bflambda$, such that $p_{\bfH_\bflambda}(.\, ; \bflambda)$ is an extremum of functional $p\mapsto \curL ag(p,\lambda_0,\bflambda)$. Using the calculus of variations yields
\begin{equation}\label{eq:eq7}
 p_{\bfH_\bflambda}(\bfeta\,;\bflambda) = c_0(\bflambda)\, \zeta(\bfeta)\, \exp\{-\langle\bflambda\, ,\bfh^c(\bfeta) \rangle\} \quad , \quad \forall\, \bfeta\in \RR^\nu\, ,
\end{equation}
in which $c_0(\bflambda)$ is the constant of normalization that depends on $\bflambda$ (note that $\lambda_0$ is eliminated and we have $c_0(\bflambda) = c_\nu\, \exp\{-\lambda_0\}$). The existence of a unique solution requires that the constraints be algebraically independent in the following sense: given any bounded positive measure $P(d\bfeta)$ on $\RR^\nu$ with support $\RR^\nu$, there exists a bounded set $\curB$ in $\RR^\nu$ with $P(\curB) > 0$ such that
\begin{equation}\label{eq:eq7bis}
 \forall\,\bfv\in\RR^{n_c} \,\, , \,\, \Vert\,\bfv\,\Vert \neq 0\,\, , \,\, \int_\curB \langle\bfh^c(\bfeta)\, ,
 \bfv \rangle^2  P(d\bfeta) > 0 \, .
\end{equation}
Under the condition defined by Eq.~\eqref{eq:eq7bis}, there exists (see \cite{Luenberger2009})
$\bflambda^\psol$ in $\curC_{ad,\bflambda}$ such that the functional $(p,\lambda_0,\bflambda)\mapsto \curL ag(p,\lambda_0,\bflambda)$
is stationary at point $p=p_{\bfH^c}$ for $\bflambda=\bflambda^\psol$ and $\lambda_0 = -\log(c_0(\bflambda^\psol)/c_\nu)$.
Consequently, $p_{\bfH^c} = p_{\bfH_{\!\bflambda^{\,\ppsol}}}(.\, ;\bflambda^\psol)$ and Eq.~\eqref{eq:eq7} yield
\begin{equation}\label{eq:eq8}
 p_{\bfH^c}(\bfeta) = c_0(\bflambda^\psol)\, \zeta(\bfeta)\, \exp\{-\langle\bflambda^\psol ,\bfh^c(\bfeta)\rangle\} \quad , \quad \forall\, \bfeta\in \RR^\nu\, .
\end{equation}
Taking into account the introduced hypotheses, $p_{\bfH^c}$ is the unique solution of the optimization problem defined by Eq.~\eqref{eq:eq3}, in which $\bflambda^\psol$ is the unique solution of a convex optimization problem that will be defined by Proposition~\ref{proposition:1} in Section~\ref{sec:Section3}) and which will be, under required hypotheses, the solution
of the following nonlinear algebraic equation in $\bflambda$,
$\int_{\RR^\nu} \bfh^c(\bfeta)\, p_{\bfH_\bflambda}(\bfeta\,;\bflambda) \, d\bfeta = \bfb^c$.
\subsection{Implicit definition of function $\bfh^c$, resulting difficulties, and necessity to construct a statistical surrogate model}
\label{sec:Section2.6}
Using the methodology presented in Section~\ref{sec:Section2.4}, the drift vector of the ISDE will involve the matrix $[\nabla_{\!\bfeta}\bfh^c(\bfeta)]\in\MM_{\nu,n_c}$ (the transpose of the Jacobian matrix of $\bfh^c$). As we have explained in Section~\ref{sec:Section1}, an additional strong assumption used in this paper is that function $\bfeta\mapsto\bfh^c(\bfeta)$ from $\RR^\nu$ into $\RR^{n_c}$ is not explicitly defined by an algebraic expression. It is assumed that we can only compute $\bfa^\ell=\bfh^c(\bfeta^\ell)\in\RR^{n_c}$ for any point $\bfeta^\ell$ given in $\RR^\nu$
(for instance and as previously underlined, a component of $\bfh^c$ can be related to the square of a norm of the random normalized residue of the stochastic PDE). The MCMC generator requires the evaluation of $[\nabla_{\!\bfeta}\bfh^c(\bfeta)]$ for a large number of values of $\bfeta$. Consequently, a statistical surrogate model $\hat\bfh^N$ of $\bfh^c$ is constructed and allows for deducing an algebraic representation $\bfeta\mapsto [\nabla_{\!\bfeta}\hat\bfh^N\bfeta)]$ of function $\bfeta\mapsto [\nabla_{\!\bfeta}\bfh^c(\bfeta)]$ from $\RR^\nu$ into $\MM_{\nu,n_c}$. Such statistical surrogate model is an approximation whose convergence with respect to $N$ will be given by Proposition~\ref{proposition:2} in Section~\ref{sec:Section3}.

\section{Mathematical analysis of the proposed methodology}
\label{sec:Section3}
In this section, there are some repetitions with respect to Section~\ref{sec:Section2}, but we have preferred to do them so that Section~\ref{sec:Section3} be mathematically coherent and self contained.
%
\begin{definition}[Training set $\curD_d$ and probability measure of $\bfH$] \label{definition:1}
Let $\nu$ and $N_d$ be integers such that $N_d > \nu$. Let $\curD_d$ be the set of $N_d$ points $\bfeta_d^1,\ldots,\bfeta_d^{N_d}$ given in $\RR^\nu$ such that
\begin{equation}\label{eq:eq10}
 \underline{\widehat\bfeta} = \frac{1}{N_d}\sum_{j=1}^{N_d} \bfeta_d^j =\bfzero_\nu\quad , \quad
 [\widehat C_\bfH]= \frac{1}{N_d-1}\sum_{j=1}^{N_d} (\bfeta_d^j- \underline{\widehat\bfeta})\otimes (\bfeta_d^j- \underline{\widehat\bfeta}) = [I_\nu]\, .
\end{equation}
Let $\bfH=(H_1,\ldots,H_\nu)$ be the $\RR^\nu$-valued random variable defined on the probability space $(\Theta,\curT,\curP)$ whose probability measure $P_\bfH(d\bfeta)=p_\bfH(\bfeta)\, d\bfeta$ is defined by the probability density function
$\bfeta\mapsto p_\bfH(\bfeta): \RR^\nu\rightarrow \RR^+$ with respect to the Lebesgue measure $d\bfeta$ on $\RR^\nu$,
\begin{equation}\label{eq:eq11}
 p_\bfH(\bfeta) = c_\nu\,\zeta(\bfeta)\quad , \quad  \forall\bfeta\in\RR^\nu\quad  , \quad  c_\nu = (\sqrt{2\pi}\, \hat s)^{-\nu} \, ,
\end{equation}
in which $\hat s = s\, \left( s^2 + (N_d-1)/N_d \right)^{-1/2}$ with $s = \left( 4/( N_d (2+\nu) ) \right)^{1/(\nu+4)}$, and where $\bfeta\mapsto\zeta(\bfeta): \RR^\nu\rightarrow \RR^+$ is written as
\begin{equation}\label{eq:eq12}
 \zeta(\bfeta) = \frac{1}{N_d}\sum_{j=1}^{N_d} \exp\left\{ -\frac{1}{2\hat s^2}\,\Vert\,\frac{\hat s }{s}
  \, \bfeta^j_d - \bfeta\,\Vert^2 \right\} \, .
\end{equation}
We define the potential function $\bfeta\mapsto\phi(\bfeta): \RR^\nu\rightarrow \RR$, related to $p_\bfH$, such that
\begin{equation}\label{eq:eq13}
\zeta(\bfeta) = \exp\{-\phi(\bfeta)\}\, .
\end{equation}
\end{definition}
%
\begin{remark} \label{remark:1}
Definition~\ref{definition:1} of $p_\bfH$ corresponds to a Gaussian kernel-density estimation (KDE) using the training set $\curD_d$, involving the modification proposed in \cite{Soize2015} of the classical formulation \cite{Bowman1997} for which $s$ is the Sylverman bandwidth. With such a modification, the normalization of $\bfH$ is preserved for any value of $N_d$,
\begin{equation}\label{eq:eq14}
E\{\bfH\} = \int_{\RR^\nu} \bfeta\, p_\bfH(\bfeta)\, d\bfeta = \frac{1}{2\hat s^2}\,\underline{\widehat\bfeta} = \bfzero_\nu\, ,
\end{equation}
\begin{equation}\label{eq:eq15}
E\{\bfH\otimes\bfH\} = \int_{\RR^\nu} \bfeta\otimes\bfeta\, p_\bfH(\bfeta)\, d\bfeta = \hat s^2 \,[I_\nu] +
     \frac{\hat s^2}{s^2} \frac{(N_d-1)}{N_d}\,[\widehat C_\bfH] = [I_\nu]\, .
\end{equation}
Theorem~3.1 in \cite{Soize2020c} proves that, for all $\bfeta$ fixed in $\RR^\nu$, Eq.~\eqref{eq:eq11} with Eq.~\eqref{eq:eq12} is a consistent estimation of the sequence $\{p_\bfH\}_{N_d}$ for $N_d\rightarrow +\infty$.
\end{remark}

%
\begin{hypothesis}[Concerning function $\bfh^c$] \label{hypothesis:1}
\textit{Let $n_c$ be the integer  such that $1\leq n_c\leq \nu$.
It is assumed that $\bfeta\mapsto\bfh^c(\bfeta)$ verifies the property defined by Eq.~\eqref{eq:eq7bis}, is continuously differentiable from $\RR^\nu$ into $\RR^{n_c}$,
\begin{equation}\label{eq:eq16}
\bfh^c \in C^1(\RR^\nu,\RR^{n_c}) \, ,
\end{equation}
and there exist constants $\alpha>0$, $\beta >0$, $c_\alpha >0$, and $c_\beta>0$,  independent of $\bfeta$, such that for $\Vert\,\bfeta\,\Vert\rightarrow +\infty$,
\begin{equation}\label{eq:eq17}
\Vert\,\bfh^c(\bfeta)\,\Vert \,\, \leq \,c_\alpha\, \Vert\,\bfeta\,\Vert^{\,\alpha} \quad , \quad
 \Vert \, [\nabla_{\!\bfeta}\bfh^c(\bfeta)]\, \Vert_F \,\,\leq \,c_\beta\, \Vert\,\bfeta\,\Vert^{\,\beta} \, ,
\end{equation}
in which $[\nabla_{\!\bfeta}\bfh^c(\bfeta)]\in \MM_{\nu,n_c}$ with $[\nabla_{\!\bfeta}\bfh^c(\bfeta)]_{\alpha k}
         = \partial h^c_k(\bfeta)/ \partial\eta_\alpha$, and where $\Vert \, . \,\Vert_F$ is the Frobenius norm.
}
\end{hypothesis}
%
\begin{definition}[Admissible subset $\curC_{\ad,\bflambda}$ of $\RR^{n_c}$] \label{definition:2}
Under Hypothesis~\ref{hypothesis:1}, the admissible set $\curC_{\ad,\bflambda}$ of Lagrange multiplier $\bflambda$
is defined as the open subset of $\RR^{n_c}$ such that
\begin{equation}\label{eq:eq18}
\curC_{\ad,\bflambda} = \left\{ \bflambda\in\RR^{n_c} \,\, \vert \,\, 0 < E\{\, \exp\{-\langle\bflambda\, ,\bfh^c(\bfH)\rangle\}
                    \,\right\}\,\, < +\infty \, ,
\end{equation}
in which the pdf of the $\RR^\nu$-valued random variable $\bfH$ is defined by Eq.~\eqref{eq:eq11}. It is also assumed that $\bfh^c$ is such that $\curC_{\ad,\bflambda}$ is not reduced to the empty set,
\begin{equation}\label{eq:eq19}
\curC_{\ad,\bflambda} \neq \emptyset \, .
\end{equation}
\end{definition}
%
%
\begin{lemma}[Convexity of $\curC_{\ad,\bflambda}$ and integrability properties] \label{lemma:1}
Under Hypothesis~\ref{hypothesis:1} and with Definition~\ref{definition:2},

\noindent (a) $\curC_{\ad,\bflambda}$ defined by Eq.~\eqref{eq:eq18} is a convex open subset of $\RR^{n_c}$.

\noindent (b) $\forall \bflambda\in \curC_{\ad,\bflambda}$, let $\bfeta\mapsto\curV_{\!\bflambda}(\bfeta): \RR^\nu\rightarrow\RR$ be the function defined by
\begin{equation}\label{eq:eq20}
\curV_{\!\bflambda}(\bfeta) = \phi(\bfeta) + \langle\bflambda\, , \bfh^c(\bfeta)\rangle \, ,
\end{equation}
in which $\phi(\bfeta) = -\log\bfzeta(\bfeta)$ (see Eq.~\eqref{eq:eq13}). One then has
\begin{equation}\label{eq:eq21}
0 < \int_{\RR^\nu} \exp\{-\curV_{\!\bflambda}(\bfeta)\}\, d\bfeta \,\, < \, +\infty \, .
\end{equation}
\noindent (c) The pdf $\bfeta\mapsto p_{\bfH_\bflambda}(\bfeta\,;\bflambda)$ with respect to $d\bfeta$, defined by Eq.~\eqref{eq:eq7}, which can be written as
\begin{equation}\label{eq:eq22}
p_{\bfH_\bflambda}(\bfeta\,;\bflambda) = c_0(\bflambda)\, \exp\{-\curV_{\!\bflambda}(\bfeta)\} \quad , \quad \forall\bfeta\in\RR^\nu \, ,
\end{equation}
is such that the constant $c_0(\bflambda)$ of normalization verifies
\begin{equation}\label{eq:eq23}
0 < c_0(\bflambda) < +\infty \quad , \quad \forall\bflambda \in \curC_{\ad,\bflambda}\, .
\end{equation}
\noindent (d) $\forall\bflambda \in \curC_{\ad,\bflambda}$, we have
$\curV_{\!\bflambda}(\bfeta) \rightarrow +\infty$ if $\Vert\,\bfeta\,\Vert \rightarrow +\infty$, and we have,
\begin{equation}\label{eq:eq24bis}
\int_{\RR^\nu} \Vert\,\bfh^c(\bfeta)\,\Vert^2\, \exp\{-\curV_{\!\bflambda}(\bfeta)\}\, d\bfeta \,\, < \, +\infty      \quad , \quad
\int_{\RR^\nu} \Vert\,[\nabla_{\!\bfeta}\bfh^c(\bfeta) ]\,\Vert_F\, \exp\{-\curV_{\!\bflambda}(\bfeta)\}\, d\bfeta \,\, < \, +\infty\, .
\end{equation}
\end{lemma}
%
\begin{proof} (Lemma~\ref{lemma:1}).

\noindent (a) $\curC_{\ad,\bflambda}$ is a convex subset if $\forall\mu\in [0,1]$, $\forall\bflambda$ and $\bflambda'$ in
$\curC_{\ad,\bflambda}$, $\mu\bflambda + (1-\mu)\bflambda'\in \curC_{\ad,\bflambda}$. Let $A= \exp\{-\langle\bflambda\, ,\bfh^c(\bfH)\rangle\}$ and $B= \exp\{-\langle\bflambda'\, ,\bfh^c(\bfH)\rangle\}$  be $\RR^+$-valued random variables. Since $\bflambda$ and $\bflambda'$ are in $\curC_{\ad,\bflambda}$, we have $E\{A\} < +\infty$ and $E\{B\} < +\infty$. We have to prove that
$E\{ \,\exp\{-\langle\mu\bflambda + (1-\mu)\bflambda'\, ,\bfh^c(\bfH)\rangle\}\, \} < +\infty$, that is to say, $E\{A^\mu\, B^{1-\mu}\} < +\infty$. For $\mu=0$ and $\mu=1$, it is verified. For $\mu\in ]0,1[$, since $A$ and $B$ are  almost-surely positive and using the H\"older inequality yield $E\{A^\mu\, B^{1-\mu}\} \leq (E\{A\})^{\,\mu} \times (E\{B\})^{1-\mu} < +\infty$, which finishes the proof of the convexity of $\curC_{\ad,\bflambda}$.

\noindent (b) Using Eqs.~\eqref{eq:eq11}, \eqref{eq:eq13}, and \eqref{eq:eq20}, yields
$\int_{\RR^\nu} \exp\{-\curV_{\!\bflambda}(\bfeta)\}\, d\bfeta  =\frac{1}{c_\nu} \int_{\RR^\nu} \exp\{-\langle\bflambda\, , \bfh^c(\bfeta)\rangle\}\,
p_\bfH(\bfeta)\, d\bfeta  = \frac{1}{c_\nu} E\{\,\exp\{-\langle\bflambda\, ,$ $\bfh^c(\bfeta)\rangle\} \,\}$, which is positive and finite due to Eq.~\eqref{eq:eq18} and to $0 < c_\nu < +\infty$. We have thus proven Eq.~\eqref{eq:eq21}.

\noindent (c) Using Eqs.~\eqref{eq:eq21} and \eqref{eq:eq22}, and since we must have $\int_{\RR^\nu} p_{\bfH_\bflambda}(\bfeta)\, d\bfeta = 1$, we deduce Eq.~\eqref{eq:eq23}.

\noindent (d)  Since $\bfh^c$ is continuous on $\RR^\nu$, (see Eq.~\eqref{eq:eq16}), $\forall\bflambda\in\curC_{\ad,\bflambda}$, $\bfeta\mapsto \exp\{-\curV_{\!\bflambda}(\bfeta)\}$ is continuous on $\RR^\nu$ and then is locally integrable on $\RR^\nu$. Eq.~\eqref{eq:eq21} implies the integrability at infinity of $\bfeta\mapsto \exp\{-\curV_{\!\bflambda}(\bfeta)\}$. Since $\bfeta\mapsto \curV_{\!\bflambda}(\bfeta)$ is continuous on $\RR^\nu$, it can  be deduced that $\curV_{\!\bflambda}(\bfeta) \rightarrow +\infty$ if $\Vert\,\bfeta\,\Vert \rightarrow +\infty$.
Using Eq.~\eqref{eq:eq17}, for $\Vert\,\bfeta\,\Vert \rightarrow +\infty$, one has
$\Vert\,\bfh^c(\bfeta)\,\Vert^2\, \exp\{-\curV_{\!\bflambda}(\bfeta)\} \leq c_\alpha^2 \, \Vert\,\bfeta\,\Vert^{2\alpha}\, \exp\{-\curV_{\!\bflambda}(\bfeta)\}$ and $\Vert\,[\nabla_{\!\bfeta}\bfh^c(\bfeta) ]\,\Vert_F\, \exp\{-\curV_{\!\bflambda}(\bfeta)\} \leq c_\beta\, \Vert\,\bfeta\,\Vert^{\beta}\, \exp\{-\curV_{\!\bflambda}(\bfeta)\}$, which allow for proving the integrability at infinity and then proving Eq.~\eqref{eq:eq24bis}.
\end{proof}
%
%
\begin{proposition}[Construction of the probability measure of $\bfH_\bflambda$] \label{proposition:1}
We consider Hypothesis~\ref{hypothesis:1} and  Definition~\ref{definition:2}.
For all $\bflambda$ in $\curC_{\ad,\bflambda}$, let
\begin{equation} \label{eq:eq26bis}
p_{\bfH_\bflambda}(\bfeta\,;\bflambda)= c_0(\bflambda)\,\zeta(\bflambda) \exp\{- \langle \bflambda\, , \bfh^c(\bfeta) \rangle\}
\end{equation}
be the pdf of $\bfH_\bflambda$ (see Eq.~\eqref{eq:eq7}) with
$c_0(\bflambda)$ satisfying Eq.~\eqref{eq:eq23}).

\noindent (a) The $\RR^{n_c}$-valued random variable $\bfh^c(\bfH_\bflambda)$ is of second-order,
\begin{equation}\label{eq:eq27}
E\{\Vert\,\bfh^c(\bfH_\bflambda)\,\Vert^2\} < +\infty \, .
\end{equation}
\noindent (b) Let $\bflambda\mapsto \Gamma(\bflambda): \curC_{\ad,\bflambda} \rightarrow \RR$ be defined by
\begin{equation}\label{eq:eq28}
\Gamma(\bflambda) = \langle\bflambda\, , \bfb^c\rangle -\log c_0(\bflambda)\, ,
\end{equation}
in which $\bfb^c$ is given in $\RR^{n_c}$. For all $\bflambda$ in $\curC_{\ad,\bflambda}$, we have
\begin{equation}\label{eq:eq29}
\nabla_{\!\bflambda}\Gamma(\bflambda) = \bfb^c - E\{\bfh^c(\bfH_\bflambda) \} \in \RR^{n_c}\, ,
\end{equation}
\begin{equation}\label{eq:eq30}
[\Gamma{\,''}(\bflambda)] = [\cov\{\bfh^c(\bfH_\bflambda)\}] \in \MM_{n_c}^+ \, ,
\end{equation}
where the positive-definite covariance matrix $[\Gamma{\,''}(\bflambda)]$ of $\bfh^c(\bfH_\bflambda)$ is such that
$[\Gamma{\,''}(\bflambda)]_{kk'} = \partial^2\Gamma(\bflambda)/\partial\lambda_k\partial\lambda_{k'}$.

\noindent (c) $\Gamma$ is a strictly convex function on $\curC_{\ad,\bflambda}$. There is a unique solution $\bflambda^\psol$ in $\curC_{\ad,\bflambda}$ of the convex optimization problem,
\begin{equation}\label{eq:eq31}
\bflambda^\psol = \arg\, \min_{\bflambda\in\curC_{\ad,\bflambda}} \Gamma(\bflambda) \, .
\end{equation}
If the following equation in $\bflambda$,
\begin{equation}\label{eq:eq31bis}
\nabla_{\!\bflambda} \Gamma(\bflambda) = \bfzero_{n_c} \, ,
\end{equation}
has a solution $\tilde\bflambda$ that belongs to $\curC_{\ad,\bflambda}$, then this solution is unique and we have $\bflambda^\psol = \tilde\bflambda$.
The pdf $p_{\bfH^c}$ of $\bfH^c$, which satisfies the constraint $E\{\bfh^c(\bfH^c)\} = \bfb^c$  is written (see Eq.~\eqref{eq:eq22} or \eqref{eq:eq26bis}) as
\begin{equation}\label{eq:eq32}
p_{\bfH^c}(\bfeta) = p_{\bfH_{\!\bflambda^\ppsol}}(\bfeta\,;\bflambda^\psol) \quad, \quad \forall\bfeta\in\RR^\nu\, .
\end{equation}
\end{proposition}
%
\begin{proof} (Proposition~\ref{proposition:1}).

\noindent (a) Using Eq.~\eqref{eq:eq22}, Eq.~\eqref{eq:eq23}, and the first equation Eq.~\eqref{eq:eq24bis} yield
\begin{equation} \nonumber
E\{\Vert\,\bfh^c(\bfH_\bflambda)\,\Vert^2\} =\int_{\RR^\nu} \Vert\,\bfh^c(\bfeta)\,\Vert^2\, c_0(\bflambda)\, \exp\{-\curV_{\!\bflambda}(\bfeta)\} \, d\bfeta < +\infty\, .
\end{equation}

\noindent (b) The definition of function $\Gamma$ given by Eq.~\eqref{eq:eq28} is similar to the one introduced in the discrete case for finding the probability measure of maximal entropy \cite{Agmon1979,Kapur1992}. Let us prove Eqs.~\eqref{eq:eq29} and \eqref{eq:eq30}.  Eq.~\eqref{eq:eq20} yields $\nabla_{\!\bflambda}\curV_{\!\bflambda}(\bfeta) = \bfh^c(\bfeta)$ and from Eq.~\eqref{eq:eq22}, it can be deduced that
\begin{equation}\label{eq:eq33}
\nabla_{\!\bflambda} p_{\bfH_\bflambda}(\bfeta\,;\bflambda) = \left (c_0(\bflambda)^{-1}\,\nabla_{\!\bflambda} c_0(\bflambda) - \bfh^c(\bfeta)\right)\,p_{\bfH_{\bflambda}}(\bfeta\,;\bflambda) \, .
\end{equation}
Integrating Eq.~\eqref{eq:eq22} on $\RR^\nu$ and taking the logarithm  yields $\log\, c_0(\bflambda) = -\log \int_{\RR^\nu} \exp\{-\curV_{\!\bflambda}(\bfeta)\}\, d\bfeta$ and consequently,
\begin{equation}\label{eq:eq34}
c_0(\bflambda)^{-1}\,\nabla_{\!\bflambda} c_0(\bflambda) = \int_{\RR^\nu} \bfh^c(\bfeta)\, p_{\bfH_{\bflambda}}(\bfeta\,;\bflambda)\, d\bfeta = E\{\bfh^c(\bfH_\bflambda)\} \, .
\end{equation}
Eq.~\eqref{eq:eq28} yields $\nabla_{\!\bflambda}\Gamma(\bflambda) = \bfb^c - c_0(\bflambda)^{-1}\,\nabla_{\!\bflambda} c_0(\bflambda)$, which proves Eq.~\eqref{eq:eq29} by using Eq.~\eqref{eq:eq34}. It should be noted that Eq.~\eqref{eq:eq27} implies the existence of the mean value $E\{\bfh^c(\bfH_\bflambda)\}$. Taking the derivative of Eq.~\eqref{eq:eq29} with respect to $\bflambda$ yields
\begin{equation}\label{eq:eq35}
[\Gamma{\,''}(\bflambda)] = - \int_{\RR^\nu} \bfh^c(\bfeta) \otimes \nabla_{\!\bflambda} p_{\bfH_\bflambda}(\bfeta\,;\bflambda)
\, d\bfeta\, .
\end{equation}
Substituting Eq.~\eqref{eq:eq34} into Eq.~\eqref{eq:eq33} yields
$\nabla_{\!\bflambda} p_{\bfH_\bflambda}(\bfeta\,;\bflambda) = ( E\{\bfh^c(\bfH_\bflambda)\} -\bfh^c(\bfeta)\,) \, p_{\bfH_\bflambda}(\bfeta\,;\bflambda)$,
which with Eq.~\eqref{eq:eq35}, gives
$[\Gamma{\,''}(\bflambda)] = E\{ \bfh^c(\bfH_\bflambda) \otimes \bfh^c(\bfH_\bflambda)\} - (E\{\bfh^c(\bfH_\bflambda)\})\otimes (E\{\bfh^c(\bfH_\bflambda)\})$
that is the covariance matrix of the $\RR^{n_c}$-valued random variable $\bfh^c(\bfH_\bflambda)$. Again Eq.~\eqref{eq:eq27} proves the existence of matrix $[\Gamma{\,''}(\bflambda)]$ as a covariance matrix, which is semi-positive definite. We have to prove that this matrix is positive definite, which will be true if the matrix $[M_\bflambda] = E\{ \bfh^c(\bfH_\bflambda) \otimes \bfh^c(\bfH_\bflambda)\}$ belongs to $\MM_{n_c}^+$, that is to say if $\langle [M_\bflambda]\, \bfv\, , \bfv\rangle^2 > 0$ for all $\bfv$ in $\RR^{n_c}$ with $\Vert\,\bfv\,\Vert\neq 0$. Since $p_{\bfH_\bflambda}(\bfeta\,;\bflambda)\, d\bfeta$ is a probability measure, this will be true if $\bfeta\mapsto \langle\bfh^c(\bfeta)\, ,\bfv\rangle^2$ is not zero on a set $\curB$ such that $\int_\curB p_{\bfH_\bflambda}(\bfeta\,;\bflambda)\, d\bfeta > 0$  that is the case due to the hypothesis defined by Eq.~\eqref{eq:eq7bis}.

\noindent (c) From Lemma~\ref{lemma:1}-(a), $\curC_{\ad,\bflambda}$ is a convex set and from Eq.~\eqref{eq:eq30}, $[\Gamma{\,''}(\bflambda)]$ is a positive-definite matrix for all $\bflambda$ in $\curC_{\ad,\bflambda}$. It can then be deduced that $\bflambda\mapsto \Gamma(\bflambda)$ is strictly convex on $\curC_{\ad,\bflambda}$ and therefore, Eq.~\eqref{eq:eq31} holds and $\bflambda^\psol$ is unique. Note that the existence of a solution of Eq.~\eqref{eq:eq31bis}, which would then be a global minimum, could be not in $\curC_{\ad,\bflambda}$. However, if the equation $\nabla_{\!\bflambda} \Gamma(\bflambda) = \bfzero_{n_c}$ admits a solution $\bflambda =\tilde\bflambda\in\RR^{n_c}$ that belongs to $\curC_{\ad,\bflambda}$, then this solution is unique and we have $\bflambda^\psol= \tilde\bflambda$, which  is the solution of the convex optimization problem defined by Eq.~\eqref{eq:eq31}, and then Eq.~\eqref{eq:eq31bis} holds for $\tilde\bflambda=\bflambda^\psol$. Finally, under the condition that $\tilde\bflambda$ belongs to $\curC_{\ad,\bflambda}$, Eq.~\eqref{eq:eq29} shows that $E\{\bfh^c(\bfH_{\bflambda^\psol})\}=\bfb^c$. Taking into account Eq.~\eqref{eq:eq8}, the solution is given by Eq.~\eqref{eq:eq32} and is unique due to the uniqueness  of solution  $\bflambda^\psol$ of $\nabla_{\!\bflambda} \Gamma(\bflambda) = \bfzero_{n_c}$.
\end{proof}
%
\begin{definition}[Surrogate model $\hat\bfh^N$ of $\bfh^c$] \label{definition:3}
Let $\bflambda$ be fixed in $\curC_{\ad,\bflambda}$ and let $\curD_{\bfH_\bflambda} =\{\bfeta_\bflambda^1,\ldots ,\bfeta_\bflambda^N\}$ be the constrained learned set whose points are $N\gg N_d$ independent realizations of the $\RR^\nu$-valued random variable $\bfH_\bflambda$ for which the pdf $\bfeta\mapsto p_{\bfH_\bflambda}(\bfeta\, ; \bflambda)$ is defined by Eq.~\eqref{eq:eq22}. Let
$\bfA_\bflambda = \bfh^c(\bfH_\bflambda)$
be the $\RR^{n_c}$-valued random variable defined on $(\Theta,\curT,\curP)$ whose $N$ independent realizations $\bfa_\bflambda^1,\ldots , \bfa_\bflambda^N$ are such that
$\bfa_\bflambda^\ell = \bfh^c(\bfeta_\bflambda^\ell) \, \in \, \RR^{n_c}$ for $\ell = 1,\ldots , N$.
The surrogate model $\bfeta \mapsto \hat\bfh^N(\bfeta\, ; \bflambda): \RR^\nu\rightarrow\RR^{n_c}$ of $\bfh^c$ is defined, for all $\bfeta$ in $\RR^\nu$, by
\begin{equation}\label{eq:eq38}
\hat\bfh^N(\bfeta\, ; \bflambda) = \sum_{\ell=1}^N \bfa_\bflambda^\ell \, \frac{\beta_\bfeta^N(\bfeta_\bflambda^\ell)}{\sum_{\ell'=1}^N\beta_\bfeta^N(\bfeta_\bflambda^{\ell'})} \, ,
\end{equation}
in which for all $\bfeta$ and $\tilde\bfeta$ in $\RR^\nu$,
\begin{equation}\label{eq:eq39}
\beta_\bfeta^N(\tilde\bfeta) = \exp \{ -\frac{1}{2s_\SB^2} \Vert \,\tilde\bfeta - \bfeta\,\Vert^2_H  \}\, ,
\end{equation}
\begin{equation}\label{eq:eq40}
\Vert \,\tilde\bfeta - \bfeta\,\Vert^2_H = \langle [\sigma_{\bfH_\bflambda}]^{-2}(\tilde\bfeta - \bfeta)\, , \tilde\bfeta - \bfeta\rangle \, ,
\end{equation}
in which $[\sigma_{\bfH_\bflambda}]$ is the diagonal positive-definite matrix in $\MM_\nu^+$ such that $[\sigma_{\bfH_\bflambda}]_{\alpha\alpha}$ is the standard deviation of the real-valued random variable $H_{\bflambda , \alpha}$, estimated using $\curD_{\bfH_\bflambda}$, and where $s_\SB$ is the Sylverman bandwidth that depends on $N$ and written as
\begin{equation}\label{eq:eq41}
s_\SB = \left( \frac{4}{N(2+n_c+\nu)}\right )^{1/(n_c+\nu+4)} \, .
\end{equation}
\end{definition}
%
%
\begin{remark}  [Rationale of Definition~\ref{definition:3}] \label{remark:2}
Let us assume that $\bflambda$ is fixed in $\curC_{\ad,\bflambda}$.
For all $\bfeta$ in $\RR^\nu$, we have the following identity,
\begin{equation}\label{eq:eq42}
\bfh^c(\bfeta) = E\{\bfh^c(\bfH_\bflambda)\, \vert \, \bfH_\bflambda = \bfeta \}\, ,
\end{equation}
in which $E\{\bfh^c(\bfH_\bflambda)\, \vert \, \bfH_\bflambda = \bfeta \}$ is the conditional mathematical expectation of the $\RR^{n_c}$-valued second-order random variable $\bfh^c(\bfH_\bflambda)$ given $\bfH_\bflambda = \bfeta$.
Let $P_{\bfA_\bflambda  ,  \bfH_\bflambda}(d\bfa,d\bfeta\, ;\bflambda)$ be the probability measure on $\RR^{n_c}\times \RR^\nu$ of the
$\RR^{n_c}\times \RR^\nu$-valued random variable $(\bfA_\bflambda  ,  \bfH_\bflambda)$. Since $\bfA_\bflambda = \bfh^c(\bfH_\bflambda)$, the support of $P_{\bfA_\bflambda  ,  \bfH_\bflambda}$ is the manifold defined by the graph $\{(\bfa,\bfeta)\in\RR^{n_c}\times \RR^\nu \, ; \bfa = \bfh^c(\bfeta)\}$. Let $P_{\bfH_\bflambda}(d\bfeta\, ;\bflambda) = p_{\bfH_\bflambda}(\bfeta\, ;\bflambda)\, d\bfeta$ be the probability measure of $\bfH_\bflambda$, for which the density $p_{\bfH_\bflambda}$ is defined by Eq.~\eqref{eq:eq22} with $\hbox{supp}\, p_{\bfH_\bflambda} =\RR^\nu$. Eq.~\eqref{eq:eq42} can then be rewritten, for all $\bfeta$ in $\RR^\nu$, as
\begin{equation}\label{eq:eq43}
\bfh^c(\bfeta) = \int_{\bfa\in\RR^{n_c}} \bfa\, P_{\bfA_\bflambda \, \vert \, \bfH_\bflambda}(d\bfa \, \vert \, \bfeta\, ;\bflambda) \, ,
\end{equation}
in which $P_{\bfA_\bflambda \, \vert \, \bfH_\bflambda}(d\bfa \, \vert \, \bfeta\, ;\bflambda)$ is the conditional probability measure of $\bfA_\bflambda$ given $\bfH_\bflambda = \bfeta$ in $\RR^\nu$, which could also be written as $\delta_0(\bfa-\bfh^c(\bfeta))$ in which $\delta_0$ is the Dirac measure on $\RR^{n_c}$ at point $\bfa=\bfzero_{n_c}$. Note that if the joint probability measure $P_{\bfA_\bflambda  ,  \bfH_\bflambda}(d\bfa,d\bfeta\, ;\bflambda)$ of $(\bfA_\bflambda  ,  \bfH_\bflambda)$ had a density $p_{\bfA_\bflambda  ,  \bfH_\bflambda}(\bfa,\bfeta\, ;\bflambda)$ with respect to $d\bfa\otimes d\bfeta$ (that is not the case), then Eq.~\eqref{eq:eq43} could be written as $\bfh^c(\bfeta) = \int_{\bfa\in\RR^{n_c}} \bfa\, p_{\bfA_\bflambda \, \vert \, \bfH_\bflambda}(\bfa \, \vert \, \bfeta\, ;\bflambda)\, d\bfa$ with $p_{\bfA_\bflambda \, \vert \, \bfH_\bflambda}(\bfa \, \vert \, \bfeta\, ;\bflambda) =
p_{\bfA_\bflambda , \bfH_\bflambda}(\bfa \, \vert \, \bfeta\, ;\bflambda) / p_{\bfH_\bflambda}(\bfeta\, ;\bflambda)$.
The statistical surrogate model $\bfeta\mapsto\hat\bfh^N(\bfeta\, ; \bflambda): \RR^\nu\rightarrow \RR^{n_c}$ of $\bfh^c$ is then defined by approximating (regularizing) $P_{\bfA_\bflambda \, \vert \, \bfH_\bflambda}(d\bfa \, \vert \, \bfeta\, ;\bflambda)$ by the conditional probability measure,
\begin{equation}\label{eq:eq44}
\hat p^N_{\bfA_\bflambda \, \vert \, \bfH_\bflambda}(\bfa \, \vert \, \bfeta\, ;\bflambda)\, d\bfa = ( \hat p^N_{\bfH_\bflambda}(\bfeta\, ;\bflambda))^{-1} \, \hat p^N_{\bfA_\bflambda , \bfH_\bflambda}(\bfa , \bfeta\, ;\bflambda)\, d\bfa\, ,
\end{equation}
in which the pdf $\hat p^N_{\bfA_\bflambda , \bfH_\bflambda}(\bfa , \bfeta\, ;\bflambda)$ on $\RR^{n_c}\times\RR^\nu$ with respect to $d\bfa\otimes d\bfeta$ is defined by the following Gaussian kernel density representation, based on the $N$ independent realizations $\{(\bfa_\bflambda^\ell,\bfeta_\bflambda^\ell), \ell=1,\ldots, N\}$ of $(\bfA_\bflambda,\bfH_\bflambda)$, and where
$\hat p^N_{\bfH_\bflambda}(\bfeta\, ;\bflambda) = \int_{\RR^{n_c}} \hat p^N_{\bfA_\bflambda , \bfH_\bflambda}(\bfa , \bfeta\, ;\bflambda)\, d\bfa$.
Therefore, we have
\begin{equation}\label{eq:eq46}
 \hat p^N_{\bfA_\bflambda , \bfH_\bflambda}(\bfa , \bfeta\, ;\bflambda) = \frac{1}{N}\sum_{\ell=1}^N
 \left (  (\sqrt{2\pi}\,s_\SB)^{n_c+\nu}\det  [\sigma_{\bfA_\bflambda}]\,\det [\sigma_{\bfH_\bflambda}] \right )^{-1}\,
 \exp\left \{  -\frac{1}{2s_\SB^2} \left( \Vert \,\bfa_\bflambda^\ell - \bfa\,\Vert^2_A + \Vert \,\bfeta_\bflambda^\ell - \bfeta\,\Vert^2_H \right )\right \}\, ,
\end{equation}
in which $s_\SB$ is defined by Eq.~\eqref{eq:eq41}, where $[\sigma_{\bfA_\bflambda}]$ is the diagonal positive-definite matrix in $\MM_{n_c}^+$ such that $[\sigma_{\bfA_\bflambda}]_{kk}$ is the standard deviation of the real-valued random variable $A_{\bflambda ,k}$, estimated using the realizations $\{\bfa_\bflambda^\ell,\ell=1,\ldots , N\}$, and where for all $\tilde \bfa$ and $\bfa$ in $\RR^{n_c}$,
$\Vert \,\tilde\bfa - \bfa\,\Vert^2_A = \langle [\sigma_{\bfA_\bflambda}]^{-2}(\tilde\bfa - \bfa)\, , \tilde\bfa - \bfa\rangle$.
From Eq.~\eqref{eq:eq43} and using the approximation of $P_{\bfA_\bflambda \,\vert\, \bfH_\bflambda}(d\bfa \,\vert\, \bfeta\, ;\bflambda)$
defined by Eq.~\eqref{eq:eq44}, we have
$\hat \bfh^N(\bfeta\, ;\bflambda)  = ( \hat p^N_{\bfH_\bflambda}(\bfeta\, ;\bflambda))^{-1}\int_{\bfa\in\RR^{n_c}} \bfa\,
\hat p^N_{\bfA_\bflambda , \bfH_\bflambda}(\bfa , \bfeta\, ;\bflambda)\, d\bfa$,
which using Eq.~\eqref{eq:eq46} yields Eq.~\eqref{eq:eq38}.
\end{remark}
%
%
\begin{proposition}[Convergence of sequences $\{\hat\bfh^N(\bfeta\,;\bflambda)\}_N$
 and $\{ \hbox{[}\nabla_{\!\bfeta} \hat\bfh^N(\bfeta\,;\bflambda)\hbox{]}\}_N$] \label{proposition:2}
Let $\bflambda$ be fixed in $\curC_{\ad,\bflambda}$ and let $\bfeta$ be fixed in $\RR^\nu$. Under the hypothesis
$\bfh^c \in C^1(\RR^\nu,\RR^{n_c})$ (see Eq.~\eqref{eq:eq16}), $\forall\varepsilon > 0$, there exists a finite integer $N_\varepsilon(\bfeta,\bflambda)$ depending on $\varepsilon$, $\bfeta$, and $\bflambda$, such that $\forall N \geq N_\varepsilon(\bfeta,\bflambda)$,
\begin{equation}\label{eq:eq49}
 \Vert\,\hat\bfh^N(\bfeta\, ; \bflambda) - \bfh^c(\bfeta)\,\Vert \,\,\,\leq \,\,\varepsilon
\quad , \quad
 \Vert\,[\nabla_{\!\bfeta}\hat\bfh^N(\bfeta\, ; \bflambda)] - [\nabla_{\!\bfeta}\bfh^c(\bfeta)]\,\Vert_F \,\,\leq \,\,\,\varepsilon \, .
\end{equation}
\end{proposition}
%
\begin{proof} (Proposition~\ref{proposition:2}).
The probability measure $P_{\bfH_\bflambda}(d\bfeta\, ; \bflambda)$  admits a continuous density $p_{\bfH_\bflambda}(\, . \, ; \bflambda)$ with respect to $d\bfeta$ on $\RR^\nu$ (see Eq.~\eqref{eq:eq22}). Using the notation of Definition~\ref{definition:3}, for $s > 0$, let
$\tilde\bfeta\mapsto\kappa_\bfeta(\tilde\bfeta\, ;s,\bflambda) = (\sqrt{2\pi}\, s)^{-\nu}\,(det [\sigma_{\bfH_\bflambda}])^{-1}$
$\exp \{  -\frac{1}{2s^2} \Vert \,\tilde\bfeta- \bfeta\,\Vert^2_H \}$
 be the function defined on $\RR^\nu$ with values in $\RR^+$. Since $\forall\alpha\in\{1,\ldots, \nu\}$,  $\lim_{s\rightarrow 0_+} \{s \,[\sigma_{\bfH_\bflambda}]_{\alpha\alpha}\} = 0$, it can be seen that
\begin{equation}\label{eq:eq51}
 \lim_{s\rightarrow 0_+} \kappa_\bfeta(\tilde\bfeta\, ;s,\bflambda)\, d\tilde\bfeta  = \delta_0(\tilde\bfeta-\bfeta)\, ,
\end{equation}
in the vector space of bounded measure, in which $\delta_0(\tilde\bfeta)$ is the Dirac measure on $\RR^\nu$ at point $\tilde\bfeta=\bfzero_\nu$. Eq.~\eqref{eq:eq38} with Eqs.~\eqref{eq:eq39} and \eqref{eq:eq40} can be rewritten as
$\hat\bfh^N(\bfeta\, ; \bflambda) = \widehat\hh^N(\bfeta\, ;s_\SB , \bflambda)$ in which for all $s > 0$,
$\widehat\hh^N(\bfeta\, ;s , \bflambda) = \left ( (1/N) \sum_{\ell=1}^N \bfh^c(\bfeta_\bflambda^\ell)\,\kappa_\bfeta(\bfeta_\bflambda^\ell\, ;s,\bflambda)\right ) \left ( (1/N) \sum_{\ell'=1}^N \kappa_\bfeta(\bfeta_\bflambda^{\ell'}\, ;s,\bflambda) \right )^{-1}$.
Since $\bfh^c$, $p_{\bfH_\bflambda}(\, .\, ;\bflambda)$, and $\kappa_\bfeta(\, . \, ;s,\bflambda)$ are continuous on $\RR^\nu$, and since
for any value of $N$, $\bfeta_\bflambda^1,\ldots ,\bfeta_\bflambda^N$ are $N$ independent realizations of $\RR^\nu$-valued random variable $\bfH_\bflambda$, for $s > 0$ fixed,
$\lim_{N\rightarrow +\infty} \widehat\hh^N(\bfeta\, ;s , \bflambda) =
\left (E\{\bfh^c(\bfH_\bflambda)\,\kappa_\bfeta(\bfH_\bflambda\, ;s,\bflambda)\}\right )
\left ( E\{\kappa_\bfeta(\bfH_\bflambda\, ;s,\bflambda)\} \right )^{-1}$,
which can be rewritten as
$\lim_{N\rightarrow +\infty} \widehat\hh^N(\bfeta\, ;s , \bflambda) =
\left (\int_{\RR^\nu}\bfh^c(\tilde\bfeta)\, p_{\bfH_\bflambda}(\tilde\bfeta\, ;\bflambda)\,\kappa_\bfeta(\tilde\bfeta\, ;s,\bflambda)\right ) \left ( \int_{\RR^\nu} p_{\bfH_\bflambda}(\tilde\bfeta\, ;\bflambda)\,\kappa_\bfeta(\tilde\bfeta\, ;s,\bflambda)\right )^{-1}$.
Since $\bfh^c$ belongs to $C^1(\RR^\nu,\RR^{n_c})$ and $p_{\bfH_\bflambda}(\, . \,  ;\bflambda)$ to $C^0(\RR^\nu,\RR^+)$, using Eq.~\eqref{eq:eq51} yields, for $\bfeta$ fixed in $\RR^\nu$,
$\lim_{s\rightarrow 0_+} \lim_{N\rightarrow +\infty} \widehat\hh^N(\bfeta\, ;s , \bflambda) = \bfh^c(\bfeta)\,\in\,\RR^\nu$
and
$\lim_{s\rightarrow 0_+} \lim_{N\rightarrow +\infty} [\nabla_{\!\bfeta}\widehat\hh^N(\bfeta\, ;s , \bflambda)] = [\nabla_{\!\bfeta}\bfh^c(\bfeta)]\,\in\,\MM_{\nu,n_c}$. Consequently, $\forall\bflambda\in\curC_{\ad,\bflambda}$, $\forall\bfeta\in\RR^\nu$, $\forall\varepsilon >0$, there exists $s_\varepsilon > 0$ and a finite integer $N_\varepsilon(\bfeta,\bflambda)$ depending on $\varepsilon$, $\bfeta$, and $\bflambda$, such that
$\Vert\, \widehat\hh^{N_\varepsilon}(\bfeta\, ;s_\varepsilon , \bflambda) - \bfh^c(\bfeta)\, \Vert \,\, \leq  \varepsilon$ and
$\Vert\, [\nabla_{\!\bfeta}\widehat\hh^{N_\varepsilon}(\bfeta\, ;s_\varepsilon , \bflambda)] - [\nabla_{\!\bfeta}\bfh^c(\bfeta)]\, \Vert \,\, \leq  \varepsilon$.
The Sylverman bandwidth $s_\SB = s_\SB(N)$ defined by Eq.~\eqref{eq:eq41}, goes to $0$ when $N\rightarrow +\infty$.
Therefore, choosing $N\geq N_\varepsilon(\bfeta,\bflambda)$  such that $s_\SB(N) < s_\varepsilon$ (that is always possible) yields
$\Vert\, \widehat\hh^N(\bfeta\, ;s_\SB(N), \bflambda) - \bfh^c(\bfeta)\, \Vert \,\, \leq  \varepsilon$ and
$\Vert\, [\nabla_{\!\bfeta}\widehat\hh^N(\bfeta\, ;s_\SB(N) , \bflambda)] - [\nabla_{\!\bfeta}\bfh^c(\bfeta)]\, \Vert \,\, \leq  \varepsilon$,
which proves Eq.~\eqref{eq:eq49}.
\end{proof}

\noindent {\textbf{Notation} (\textit{Normalized $\RR^\nu$-valued Wiener stochastic process}).
Let $\{\bfW^\wien(t) = (W_1^\wien(t),\ldots ,W_\nu^\wien(t)), t\geq 0\}$ be the Wiener process, defined on $(\Theta,\curT,\curP)$, indexed by $\RR^+$, with values in $\RR^\nu$, such that $W_1^\wien,\ldots ,W_\nu^\wien$ are mutually independent, $\bfW^\wien(0) =\bfzero_\nu$ a.s., $\bfW^\wien$ is a process with independent increments such that, for all $0\leq t'  < t < +\infty$, the increment $\bfW^\wien(t)-\bfW^\wien(t')$ is a $\RR^\nu$-valued second-order, Gaussian, centered random variable whose covariance matrix is $(t-t')\,[I_\nu]$.
%
%
\begin{proposition}[MCMC generator of $\bfH_\bflambda$] \label{proposition:3}
Let $\bfh^c$ be the function satisfying Hypothesis~\ref{hypothesis:1}. Let $\bflambda$ be fixed in $\curC_{\ad,\bflambda}$. Consequently,
Lemma~\ref{lemma:1}-(d) holds. Let $\{ ( \bfU_{\!\bflambda}(t),\bfV_{\!\bflambda}(t) ), t\geq 0\}$ be the stochastic process, defined on $(\Theta,\curT,\curP)$, indexed by $\RR^+$, with values in $\RR^\nu\times\RR^\nu$, which verifies the following ISDE for $t > 0$, with the initial condition $(\bfu_0,\bfv_0)$ given in $\RR^\nu\times\RR^\nu$,
\begin{align}
d\bfU_{\!\bflambda}(t) & = \bfV_{\!\bflambda}(t)\, dt \, , \label{eq:eq60} \\
d\bfV_{\!\bflambda}(t) & = \bfL_{\bflambda}(\bfU_\bflambda(t))\, dt - \frac{1}{2} f_0\,\bfV_{\!\bflambda}(t) \, dt
                        + \sqrt{f_0}\, d\bfW^\wien(t) \, ,\label{eq:eq61} \\
 \bfU_{\!\bflambda}(0) & = \bfu_0 \,\, , \,\, \bfV_{\!\bflambda}(0)= \bfv_0 \,\,  a.s. \label{eq:eq62}
\end{align}
\noindent (a) The initial condition $\bfu_0\in\RR^\nu$ is chosen as any point in training set $\curD_d= \{\bfeta_d^1,\ldots , \bfeta_d^{N_d}\}$  while the initial condition $\bfv_0$ is chosen as any realization of a normalized Gaussian $\RR^\nu$-valued random variable $\bfV_G$, independent of $\bfW^\wien$, whose probability density function with respect to $d\bfv$ is
$p_{\bfV_G}(\bfv) = (2\pi)^{-\nu/2}\, \exp\{-\Vert\,\bfv\,\Vert^2 /2\}$.

\noindent (b) The parameter $f_0 > 0$ allows the dissipation term in the dissipative Hamiltonian system to be controlled and to rapidly reach the stationary response associated with the invariant measure (the value $f_0=4$ is generally a good choice).

\noindent (c) For all $\bfu =(u_1,\ldots ,u_\nu)$ in $\RR^\nu$, the vector $\bfL_\bflambda(\bfu)$ in $\RR^\nu$ is defined by
$\bfL_\bflambda(\bfu) = -\nabla_{\!\bfu}\curV_{\!\bflambda}(\bfu)$ that can be written as
\begin{equation}\label{eq:eq63}
 \bfL_\bflambda(\bfu) = \frac{1}{\zeta(\bfu)}\,\nabla_{\!\bfu}\zeta(\bfu) - [\nabla_{\!\bfu} \bfh^c(\bfu)]\,\bflambda\, .
\end{equation}
\noindent (d) The stochastic solution $\{ ( \bfU_{\!\bflambda}(t),\bfV_{\!\bflambda}(t) ), t\geq 0\}$ of the ISDE defined by
Eqs.~\eqref{eq:eq60} to \eqref{eq:eq62} is unique,  has almost-surely continuous trajectories, and is a second-order diffusion stochastic process. For $t\rightarrow +\infty$, this diffusion process converges to a stationary second-order diffusion stochastic process
$\{ ( \bfU^\st_{\!\bflambda}(\tau),\bfV^\st_{\!\bflambda}(\tau) ), \tau\geq 0\}$  associated with the unique invariant probability measure on $\RR^\nu\times\RR^\nu$,
\begin{equation}\label{eq:eq64}
 p_{\bfH_\bflambda,\bfV_G}(\bfeta,\bfv\, ; \bflambda)\, d\bfeta \otimes d\bfv =
 (p_{\bfH_\bflambda}(\bfeta \, ; \bflambda)\, d\bfeta) \otimes (p_{\bfV_G}(\bfv)\, d\bfv) \, ,
\end{equation}
in which $p_{\bfH_\bflambda}(\bfeta \, ; \bflambda)$ is the pdf defined by Eq.~\eqref{eq:eq22}.

\noindent (e)  For $t_s$ sufficiently large, we can choose $\bfH_\bflambda$ as $\bfU_{\!\bflambda}(t_s)$. The generation of the constrained learned set $\curD_{\bfH_\bflambda} = \{\bfeta_\bflambda^1,\ldots , \bfeta_\bflambda^N\}$, made up of $N\gg N_d$ independent realizations of $\bfH_\bflambda$ whose probability measure is $p_{\bfH_\bflambda}(\bfeta \, ; \bflambda)\, d\bfeta$, consists in solving Eqs.~\eqref{eq:eq60} to \eqref{eq:eq62} for $t\in[0\, ,t_s]$ and then using the realizations of $\bfU_{\!\bflambda}(t_s)$ (see the numerical aspects in Section~\ref{sec:Section4}).
\end{proposition}
%
\begin{proof} (Proposition~\ref{proposition:3}).
Since $\bfh^c\in C^1(\RR^\nu,\RR^{n_c})$ and $\phi(\bfu)= -\log\zeta(\bfu)$ with $\zeta(\bfu)$ given by Eq.~\eqref{eq:eq12}, function $\bfu\mapsto \curV_{\!\bflambda}(\bfu)$ defined by Eq.~\eqref{eq:eq20} belongs to $C^1(\RR^\nu,\RR)$. Therefore,
$\bfu\mapsto \Vert\,\nabla_{\!\bfu}\curV_{\!\bflambda}(\bfu)\,\Vert$ is locally bounded on $\RR^\nu$. Using Eqs.~\eqref{eq:eq20} and \eqref{eq:eq21}, it can be seen that, for all $\bflambda\in\curC_{\ad,\bflambda}$,  $\inf_{\Vert\,\bfu\,\Vert > R}  \curV_{\!\bflambda}(\bfu) \rightarrow +\infty$ if $R\rightarrow +\infty$, and $\inf_{\bfu\in\RR^\nu} \curV_{\!\bflambda}(\bfu)$ is a finite real number. Using Eqs.~\eqref{eq:eq12}, \eqref{eq:eq13}, and \eqref{eq:eq20} yields
\begin{equation}\label{eq:eq65}
 \int_{\RR^\nu} \Vert\,\nabla_{\!\bfu}\curV_{\!\bflambda}(\bfu)\,\Vert \, p_{\bfH_\bflambda}(\bfu \, ;\bflambda)\, d\bfu \leq
 \int_{\RR^\nu} \frac{1}{\zeta(\bfu)}\Vert\,\nabla_{\!\bfu}\zeta(\bfu)\,\Vert \, p_{\bfH_\bflambda}(\bfu \, ; \bflambda)\, d\bfu
 + \int_{\RR^\nu} \Vert\, [\nabla_{\!\bfu}\bfh^c(\bfu)]\,\Vert_F \, \Vert\, \bflambda\, \Vert\, p_{\bfH_\bflambda}(\bfu \, ;\bflambda)\, d\bfu \, ,
\end{equation}
because $\Vert\, [\nabla_{\!\bfu}\bfh^c(\bfu)]\,\bflambda\, \Vert
\,\,\leq \Vert\, [\nabla_{\!\bfu}\bfh^c(\bfu)]\,\Vert \, \Vert\, \bflambda\, \Vert$
and  $\Vert\, [\nabla_{\!\bfu}\bfh^c(\bfu)]\,\Vert  \,\,\leq \Vert\, [\nabla_{\!\bfu}\bfh^c(\bfu)]\,\Vert_F$.
From Eqs.~\eqref{eq:eq21} and \eqref{eq:eq12}, the first term in the right-hand side member of Eq.~\eqref{eq:eq65} is finite, while from the second equation~\eqref{eq:eq24bis}, the second term is also finite. It can then be deduced that the left-hand side member of Eq.~\eqref{eq:eq65} is finite.
Consequently, Theorems 6, 7, and 9 in Pages 214 to 216 of \cite{Soize1994}, and the expression of the invariant measure given by Theorem~4 in Page 211 of the same reference, for which the Hamiltonian is $\curH(\bfu,\bfv)=\Vert\,\bfv\,\Vert^2/2 + \curV_{\!\bflambda}(\bfu)$, prove that the solution of  Eqs.~\eqref{eq:eq60} to \eqref{eq:eq62} is unique and is a second-order diffusion stochastic process with almost-surely continuous trajectories, which converges for $t\rightarrow +\infty$ to a second-order stationary diffusion process with almost surely continuous trajectories $\{ ( \bfU^\st_{\!\bflambda}(\tau),\bfV^\st_{\!\bflambda}(\tau) ), \tau\geq 0\}$  associated with the invariant probability measure given by Eq.~\eqref{eq:eq64}. For any $\tau > 0$,
$\bfU_{\!\bflambda}^\st(\tau) = \lim_{t\rightarrow +\infty}\bfU_{\!\bflambda}(t+\tau)$ in probability measure.
\end{proof}
%
\begin{proposition}[Convergence of the sequence of MCMC generator using the statistical surrogate model] \label{proposition:4}
Let $\bflambda$ be fixed in $\curC_{\ad,\bflambda}$ and let us use Proposition~\ref{proposition:3}. For all $\bfeta$ in $\RR^\nu$, let $\hat\bfh^N(\bfeta\, ; \bflambda)$ be the approximation of $\bfh^c(\bfeta)$ defined by Eq.~\eqref{eq:eq38} and let $\bfu\mapsto \hat\bfL_\bflambda(\bfu)$ be  the twice continuously differentiable function on $\RR^\nu$ with values in $\RR^\nu$ such that, for all $\bfu$ in $\RR^\nu$,
\begin{equation}\label{eq:eq66}
 \hat\bfL^N_\bflambda(\bfu) = \frac{1}{\zeta(\bfu)}\,\nabla_{\!\bfu}\zeta(\bfu)
              - [\nabla_{\!\bfu} \hat\bfh^N(\bfu\, ;\bflambda)]\,\bflambda\, ,
\end{equation}
in which $\bfzeta$ is defined by Eq.~\eqref{eq:eq12} and where $[\nabla_{\!\bfu} \hat\bfh^N(\bfu\, ;\bflambda)]$ is explicitly given (see Eq.~\eqref{eq:eq104bis}) by differentiating function $\bfu\mapsto \hat\bfh^N(\bfu\, ;\bflambda)$ defined by Eq.~\eqref{eq:eq38}. Let $\{ ( \bfU^N_{\!\bflambda}(t),\bfV^N_{\!\bflambda}(t) ), t\geq 0\}$ be the stochastic process solution of the ISDE defined by  Eqs.~\eqref{eq:eq60} to \eqref{eq:eq62} in which $\bfL_\bflambda$ is replaced by $\hat\bfL^N_\bflambda$,
\begin{align}
d\bfU^N_{\!\bflambda}(t) & = \bfV^N_{\!\bflambda}(t)\, dt \, , \label{eq:eq67} \\
d\bfV^N_{\!\bflambda}(t) & = \hat\bfL^N_{\bflambda}(\bfU^N_\bflambda(t))\, dt - \frac{1}{2} f_0\,\bfV^N_{\!\bflambda}(t) \, dt
                        + \sqrt{f_0}\, d\bfW^\wien(t) \, ,\label{eq:eq68} \\
 \bfU^N_{\!\bflambda}(0) & = \bfu_0 \,\, , \,\, \bfV^N_{\!\bflambda}(0)= \bfv_0 \,\,  a.s. \, , \label{eq:eq69}
\end{align}
and where $\bfu_0$, $\bfv_0$, $f_0$, and $\bfW^\wien$ are the quantities defined in Proposition~\ref{proposition:3}.
Then the stochastic solution $\{ ( \bfU^N_{\!\bflambda}(t),\bfV^N_{\!\bflambda}(t) ), $ $t\geq 0\}$ of
Eqs.~\eqref{eq:eq67} to \eqref{eq:eq69} is unique,  has almost-surely continuous trajectories, and is a second-order diffusion stochastic process, which converges to a stationary second-order diffusion stochastic process for $t\rightarrow +\infty$, associated with the unique invariant probability measure on $\RR^\nu\times\RR^\nu$,
$ \hat p^N_{\bfH_\bflambda,\bfV_G}(\bfeta,\bfv\, ; \bflambda)\, d\bfeta \otimes d\bfv =
 (\hat p^N_{\bfH_\bflambda}(\bfeta \, ; \bflambda)\, d\bfeta) \otimes (p_{\bfV_G}(\bfv)\, d\bfv)$,
in which $\hat p^N_{\bfH_\bflambda}(\bfeta \, ; \bflambda)= \hat c_0^N(\bflambda)\, \exp\{-\widehat\curV^N_{\!\bflambda}(\bfeta)\}$
with $\widehat\curV^N_{\!\bflambda}(\bfeta)= -\log\zeta(\bfzeta)
+ \langle\bflambda\, , \hat\bfh^N(\bfeta\, ;\bflambda)\rangle$. Then for all $t\in [0\, , t_s]$ with $t_s < +\infty$, the  sequence
$\{\bfU_{\!\bflambda}^N(t)\}_N$ of second-order $\RR^\nu$-valued random variables converges in mean-square to the second-order $\RR^\nu$-valued random variable  $\bfU_{\!\bflambda}(t)$ of Proposition~\ref{proposition:3},
\begin{equation}\label{eq:eq73}
\lim_{N\rightarrow +\infty} E\{ \Vert\,\bfU^N_{\!\bflambda}(t) - \bfU_{\!\bflambda}(t) \, \Vert^2\} = 0 \quad , \quad \forall t\in [0\, , t_s]\, .
\end{equation}
\end{proposition}
%
\begin{proof} (Proposition~\ref{proposition:4}).
The classical theorem, such as Theorem 5.1 Page 118 of \cite{Friedman2006}, cannot directly be used because the required  hypotheses are not satisfied and consequently, an adapted proof of this Proposition~\ref{proposition:4} must be done.
For $\bflambda$ be fixed in $\curC_{\ad,\bflambda}$, the unique second-order stochastic process with almost-surely continuous trajectories
$\{ ( \bfU_{\!\bflambda}(t),\bfV_{\!\bflambda}(t) ), t\geq 0\}$ of Proposition~\ref{proposition:3} can be written as
\begin{equation}\label{eq:eq74}
\bfZ_\bflambda(t) = \bfz_0 +\int_0^{\, t} \aa_\bflambda(\bfZ_\bflambda(\tau))\,d\tau +\int_0^{\, t} [\bb]\, d\bfW^\wien(\tau)\, ,
\end{equation}
in which $\bfz_0=(\bfu_0,\bfv_0)$, $\aa_\bflambda(\bfz) = (\bfv\, , \bfL_{\bflambda}(\bfu) -(1/2)\,f_0\,\bfv )$ and where
$\bfz=(\bfu,\bfv)$ with $\bfu$ and $\bfv$ in $\RR^\nu$, where $\bfz_0$, $\bfz$, and $\aa_\bflambda(\bfz)$ are in $\RR^{2\nu}=\RR^\nu \times \RR^\nu$, and where $[\bb] = [ \, [0_{\nu}] \,\, \sqrt{f_0}\,[I_\nu]\, ]^T \in\MM_{2\nu,\nu}$. Reusing the proof of Proposition~\ref{proposition:3}, it can be seen that Eqs.~\eqref{eq:eq67} to \eqref{eq:eq69} admits a unique  solution
$\{ ( \bfU^N_{\!\bflambda}(t),\bfV^N_{\!\bflambda}(t) ), t\geq 0\}$  (with the properties given in Proposition~\ref{proposition:3}), which can be written as,
\begin{equation}\label{eq:eq75}
\bfZ^N_\bflambda(t) = \bfz_0 +\int_0^{\, t} \widehat\aa^N_\bflambda(\bfZ^N_\bflambda(\tau))\,d\tau +\int_0^{\, t} [\bb]\, d\bfW^\wien(\tau)\, ,
\end{equation}
in which $\widehat\aa^N_\bflambda(\bfz) = (\bfv\, , \hat\bfL^N_{\bflambda}(\bfu) -(1/2)\,f_0\,\bfv )$.
Subtracting the two equations Eqs.~\eqref{eq:eq74} and \eqref{eq:eq75} yields
\begin{equation}\nonumber 
\bfZ^N_\bflambda(t) - \bfZ_\bflambda(t) = \bfchi_\bflambda^N(t)
  +\int_0^{\, t} \left ( \widehat\aa^N_\bflambda(\bfZ^N_\bflambda(\tau)) - \widehat\aa^N_\bflambda(\bfZ_\bflambda(\tau)) \right ) \,d\tau \, ,
\end{equation}
\begin{equation}\label{eq:eq77}
\bfchi_\bflambda^N(t) =
  \int_0^{\, t} \left ( \widehat\aa^N_\bflambda(\bfZ_\bflambda(\tau)) - \aa_\bflambda(\bfZ_\bflambda(\tau)) \right ) \,d\tau \, .
\end{equation}
Let $L^2(\Theta,\RR^m)$ be the Hilbert space of the equivalent classes of second-order $\RR^m$-valued random variables defined on $(\Theta,\curT,\curP)$, equipped with the inner product $\langle\langle \bfA\, ,\bfA'\rangle\rangle = E\{\langle \bfA\, , \bfA'\rangle\}$ and the associated norm $|||\, \bfA\, ||| = (E\{\Vert\,\bfA\,\Vert^2\})^{1/2}$.
Let us define $\bfG(\tau) = \widehat\aa^N_\bflambda(\bfZ^N_\bflambda(\tau)) - \widehat\aa^N_\bflambda(\bfZ_\bflambda(\tau))$ and $\11(\tau) = 1$.
The Schwarz inequality
$\int_0^t \11(\tau) \times |||\, \bfG(\tau)\,|||\, d\tau  \leq (\int_0^t \11(\tau)^2 d\tau)^{1/2} \, (\int_0^t |||\, \bfG(\tau)\,|||^2\, d\tau)^{1/2}$
yields
$\left (\int_0^{\, t} |||\, \bfG(\tau)\,|||\, d\tau \right )^2 \leq t\int_0^{\, t} |||\, \bfG(\tau)\,|||^2\, d\tau$.
Using the  inequality $2\alpha\beta \leq 2\alpha^2 + \beta^2/2$ for all $\alpha > 0$ and $\beta > 0$ and
since
$|||\, \widehat\aa^N_\bflambda(\bfZ^N_\bflambda(\tau)) - \widehat\aa^N_\bflambda(\bfZ_\bflambda(\tau))\, |||^2 = (1+f_0^2/4)\,
|||\, \bfV^N_{\!\bflambda}(\tau) - \bfV_{\!\bflambda}(\tau)\,|||^2
+ |||\, \hat\bfL^N_{\bflambda}(\bfU^N_{\!\bflambda}(\tau)) - \hat\bfL^N_{\bflambda}(\bfU_{\!\bflambda}(\tau))\,|||^2$, we obtain
\begin{equation}\label{eq:eq78}
|||\,\bfZ^N_\bflambda(t) - \bfZ_\bflambda(t)\, |||^2 \,\, \leq \, 3\, |||\,\bfchi_\bflambda^N(t)\,|||^2
  + \frac{3}{2} t\!\! \int_0^{\, t}\left ( (1\!+\!\frac{f_0^2}{4})\, |||\, \bfV^N_{\!\bflambda}(\tau) - \bfV_{\!\bflambda}(\tau)\,|||^2
  + |||\, \hat\bfL^N_{\bflambda}(\bfU^N_{\!\bflambda}(\tau)) - \hat\bfL^N_{\bflambda}(\bfU_{\!\bflambda}(\tau))\,|||^2\right )\, d\tau .
\end{equation}
Let $t$ be fixed such that $0 < t \leq t_s < +\infty$. Since $\{\bfU_{\!\bflambda}^N(\tau),\tau \geq 0\}$ and
$\{\bfU_{\!\bflambda}(\tau),\tau \geq 0\}$ are dependent second-order $\RR^\nu$-valued  stochastic process with almost-surely continuous trajectories, there exists a finite positive constant $r_\bflambda(t_s)$ depending on $\bflambda$ and $t_s$, $0< r_\bflambda(t_s) < +\infty$, such that
$\sup_{0\leq \tau\leq t_s}\,|||\, \bfU^N_{\!\bflambda}(\tau)\,||| \,\, < \, r_\bflambda(t_s)$
and
$\sup_{0\leq \tau\leq t_s}\,|||\, \bfU_{\!\bflambda}(\tau)\,||| \,\, < \, r_\bflambda(t_s)$.
Let $\curU_{t_s}$ be the open ball of $L^2(\Theta,\RR^\nu)$ such that
$\curU_{t_s} = \left\{ \bfU\in L^2(\Theta,\RR^\nu) \, ; |||\,\bfU\, ||| \,\, < \, r_\bflambda(t_s) \right \}$.
Due to the convexity of the open ball in a normed vector space, $\curU_{t_s}$ is a convex open set of $L^2(\Theta,\RR^\nu)$ ($\forall\bfU, \bfU' \in \curU_{t_s}$,
$\forall\mu\in [0\, , 1]$, we have
$|||\,(1-\mu)\,\bfU + \mu\, \bfU'\,||| \,\, \leq \, (1-\mu)\, |||\,\bfU\,||| + \mu\, |||\bfU'\,||| \,\, \leq \, (1-\mu)\,r_\bflambda(t_s)  +\mu\, r_\bflambda(t_s) = r_\bflambda(t_s)$, which shows that $(1-\mu)\,\bfU + \mu\, \bfU' \in \curU_{t_s}$).
Let $\bfU\mapsto \hat\bfL^N_{\bflambda}(\bfU)$ be the mapping from $L^2(\Theta,\RR^\nu)$ into $L^2(\Theta,\RR^\nu)$, in which $\hat\bfL^N_{\bflambda}$ is defined by Eq.~\eqref{eq:eq66}. Since $\bfu\mapsto \hat\bfh^N(\bfu\, ;\bflambda)$ is twice continuously differentiable on $\RR^\nu$, then function $\bfu\mapsto \hat\bfL^N_{\bflambda}(\bfu)$ is continuously differentiable on $\RR^\nu$.
It can easily be verified that
$\forall\bfU\in\curU_{t_s}$, $\left ( E\{\Vert\,[\nabla_\bfu \hat\bfL^N_{\bflambda}(\bfU)] \,\Vert_F^2 \}\right )^{1/2}  \leq \, k_\bflambda(t_s)$,
in which $k_\bflambda(t_s)$ is a finite positive constant depending on $\bflambda$ and $t_s$. Consequently, using Theorem 3.3.2 Page 45 of \cite{Cartan1985} for Banach spaces, for all $0\leq \tau\leq t\leq t_s$, we have
\begin{equation}\label{eq:eq80}
|||\, \hat\bfL^N_{\bflambda}(\bfU^N_{\!\bflambda}(\tau)) - \hat\bfL^N_{\bflambda}(\bfU_{\!\bflambda}(\tau))\,|||
\,\, \leq \, k_\bflambda(t_s)\, |||\, \bfU^N_{\!\bflambda}(\tau) - \bfU_{\!\bflambda}(\tau)\,|||\, .
\end{equation}
From Eqs.~\eqref{eq:eq78} and \eqref{eq:eq80}, it can be deduced using the Gronwall Lemma \cite{Gronwall1919} that, for all $0\leq t\leq t_s < +\infty$,
\begin{equation}\label{eq:eq81}
|||\, \bfZ^N_{\bflambda}(t) - \bfZ_{\bflambda}(t)\,|||^2 \,\, \leq \, c_\bflambda(t_s)\, |||\,\bfchi_\bflambda^N(t) \, |||^2 \, ,
\end{equation}
with $c_\bflambda(t_s) = (9/2)\,t_s^2\, \max\{ (1 + f_0^2/4) \, , k_\bflambda(t_s)^2\} \, <  +\infty$.
Eq.~\eqref{eq:eq77} yields
$|||\, \bfchi_\bflambda^N(t)\, |||\,\, \leq \, \int_0^{\, t}|||\, \hat\bfL^N_{\bflambda}(\bfU_{\!\bflambda}(\tau)) - \bfL_{\bflambda}(\bfU_{\!\bflambda}(\tau))\,||| \, d\tau$.
Using Eqs.~\eqref{eq:eq63} and \eqref{eq:eq66} yields
\begin{equation}\label{eq:eq83}
|||\, \bfchi_\bflambda^N(t)\, |||\,\, \leq \, \Vert\,\bflambda\, \Vert \int_0^{\, t}
\left ( E\{ \Vert\, [\nabla_\bfu \hat\bfh^N(\bfU_{\!\bflambda}(\tau)\, ;\bflambda)] - [\nabla_\bfu \bfh^c(\bfU_{\!\bflambda}(\tau))]\, \Vert_F^2\right )^{1/2} \, d\tau \, .
\end{equation}
From Proposition~\ref{proposition:2}, it can be deduced that, for $N\rightarrow +\infty$, the right-hand side member of Eq.~\eqref{eq:eq83} goes to $0$ and consequently, Eq.~\eqref{eq:eq81} shows that $\bfZ_\bflambda^N(t) \rightarrow \bfZ_\bflambda(t)$ for the mean-square convergence.
\end{proof}
\bigskip

\noindent \textbf{Iterative algorithm for calculating} $\bflambda^\psol$.
Under Proposition~\ref{proposition:1}, for $\bflambda\in\curC_{\ad,\bflambda}$, since $\Gamma(\bflambda)$ cannot be evaluated in high dimension using Eq.~\eqref{eq:eq28} due to the presence of constant $c_0(\bflambda)$ (the normalization constant), $\bflambda^\psol$ cannot directly be estimated using the gradient descent algorithm applied to the convex optimization problem  defined by Eq.~\eqref{eq:eq31}.
We will then assumed that $\bflambda^\psol$ can be calculated  as the unique solution in $\bflambda$ of equation $\nabla_{\!\bflambda}\Gamma(\bflambda)= \bfzero_{n_c}$
(see Proposition~\ref{proposition:1}-(c)) and in particular Eq.~\eqref{eq:eq31bis}), that is to say (see Eq.~\eqref{eq:eq29}), solving the following equation in $\bflambda$ on $\RR^{n_c}$,
\begin{equation}\label{eq:eq84}
E\{\bfh^c(\bfH_\bflambda)\} - \bfb^c = \bfzero_{n_c}\, .
\end{equation}
This equation is solved using the Newton iterative method \cite{Kelley2003} applied to function
$\bflambda\mapsto \nabla_{\!\bflambda}\Gamma(\bflambda)$ as proposed in \cite{Batou2013,Soize2017b}, that is to say,
\begin{equation}\label{eq:eq85}
\bflambda^{\,i+1} = \bflambda^{\,i} - [\Gamma{\,''}(\bflambda^{\,i})]^{-1} \, \nabla_{\!\bflambda}\Gamma(\bflambda^{\,i})
\quad , \quad i=0,1,\ldots ,i_\pmax \, ,
\end{equation}
with $\bflambda^{\,0} = \bfzero_{n_c}$, in which $\nabla_{\!\bflambda}\Gamma(\bflambda)$ and $[\Gamma{\,''}(\bflambda)]$ are defined by
Eqs.~\eqref{eq:eq29} and \eqref{eq:eq30}, and where $i_\pmax$ is a given integer sufficiently large.
An estimation of $\bflambda^\psol$ is chosen as
\begin{equation}\label{eq:eq86}
\bflambda^\psol = \bflambda^{i_\psol} \quad , \quad  i_\psol = \arg \min_{i =1,\ldots ,i_\pmax} \error(i)\, ,
\end{equation}
in which the error function $i\mapsto\error(i): \{1,\ldots ,i_\pmax\}\rightarrow \RR^+$ is defined as follows for taking into account the possible types of algebraic quantities (scalar, vectors, tensors) that are used for defining function $\bfh^c$. Therefore, let $M$ be an integer such that $ 1\leq M\leq n_c$ and for which $\bfh^c$ and $\bfb^c$ are written as
$\bfh^c(\bfeta) =(\bfh^{c,k_1}(\bfeta),\ldots , \bfh^{c,k_M}(\bfeta))$ and
$\bfb^c = (\bfb^{c,k_1},\ldots , \bfb^{c,k_M})$ with
$\sum_{m=1}^M k_m = n_c$.The error function is then defined by
\begin{equation}\label{eq:eq87}
\error(i) = \left ( \sum_{m=1}^M w_m \left ( \frac{\error_m(i)}{\error_m(1)}\right)^2\right )^{1/2} \, ,
\end{equation}
in which $\{w_m \geq 0\, , m=1,\ldots M\}$ are given real numbers and where
\begin{equation}\label{eq:eq88}
\error_m(i) = \frac{1}{\Vert\, \bfb^{c,k_m}\, \Vert}\, \Vert\, \bfb^{c,k_m} - E\{\bfh^{c,k_m}(\bfH_{\bflambda^{\,i}})\}\Vert \,  .
\end{equation}
%
%
%
\begin{proposition}[Rate of convergence of the sequence $\{p_{\bfH_\bflambda}\}_\bflambda$] \label{proposition:5}
Let $\bflambda^{\,i+1}$ and $\bflambda^{\, i}$ be given values of $\bflambda$ in $\curC_{\ad,\bflambda}$ and let
$p_{\bfH_\bflambda}(\, .\, ;\bflambda)$ be the pdf of $\bfH_\bflambda$ defined by Eq.~\eqref{eq:eq22}.
For $\Vert\, \bflambda^{\, i+1} - \bflambda^{\, i} \, \Vert$ sufficiently small, we have
\begin{equation}\label{eq:eq89}
\Vert\, p_{\bfH_{\bflambda^{\,i+1}}} (\, . \, ; \bflambda^{\, i+1}) - p_{\bfH_{\bflambda^{\,i}}}(\, . \, ; \bflambda^{\,i})\, \Vert_{\,L^1(\RR^\nu,\RR)}
\, \,\, \leq \,\, \Vert\, \bflambda^{\, i+1} - \bflambda^{\, i} \, \Vert \, \left ( \tr \,[\Gamma{\,''}(\bflambda^{\, i})]\right)^{1/2}
 + o\,(\Vert\, \bflambda^{\, i+1} - \bflambda^{\, i}\,\Vert ) \, ,
\end{equation}
in which $[\Gamma{\,''}(\bflambda^{\,i})]\in\MM^+_{n_c}$ is defined by Eq.~\eqref{eq:eq30} for $\bflambda = \bflambda^{\, i}$.
\end{proposition}
%
\begin{proof} (Proposition~\ref{proposition:5}).
In this proof, for simplifying the writing, $\bflambda^{\,i}$ is simply written as $\bflambda$. Proposition ~\ref{proposition:1} shows that $\Gamma$ is twice differentiable in $\curC_{\ad,\bflambda}$. For all $\bfeta$ fixed in $\RR^\nu$, the Taylor expansion of
$\bflambda^{\, i+1}\mapsto p_{\bfH_{\bflambda^{\,i+1}}} (\bfeta \, ; \bflambda^{\, i+1})$ around $\bflambda$, truncated at the first order, is written as,
\begin{equation}\label{eq:eq90}
p_{\bfH_{\bflambda^{\,i+1}}} (\bfeta \, ; \bflambda^{\, i+1}) = p_{\bfH_{\bflambda}}(\bfeta \, ; \bflambda)
+ \langle \nabla_{\!\bflambda} p_{\bfH_{\bflambda}}(\bfeta \, ; \bflambda)\, ,  \bflambda^{\, i+1}\! - \bflambda\,
\rangle + \ldots
\end{equation}
Eq.~\eqref{eq:eq28} is written as $c_0(\bflambda) =\exp\{ \langle\bflambda\, , \bfb^c\rangle - \Gamma(\bflambda)\rangle\}$ and is substituted in  Eq.~\eqref{eq:eq22} (or equivalently, in  Eq.~\eqref{eq:eq7}) of $p_{\bfH_{\bflambda}}(\bfeta \, ; \bflambda)$ yielding
$p_{\bfH_{\bflambda}}(\bfeta \, ; \bflambda) = \zeta(\bfeta)\, \exp\{-\Gamma(\bflambda) - \langle\bflambda\, , \bfh^c(\bfeta)-\bfb^c\rangle\}$. The gradient with respect to $\bflambda$ can be written as
\begin{equation}\label{eq:eq92}
\nabla_{\!\bflambda} p_{\bfH_{\bflambda}}(\bfeta \, ; \bflambda) = -\left (\nabla_{\!\bflambda}\Gamma(\bflambda)
+\bfh^c(\bfeta)-\bfb^c\right )\, p_{\bfH_{\bflambda}}(\bfeta \, ; \bflambda)\, .
\end{equation}
Let us introduce the score variable
\begin{equation}\label{eq:eq93}
\bfv(\bfeta\, ;\bflambda) = \nabla_{\!\bflambda} \log p_{\bfH_{\bflambda}}(\bfeta \, ; \bflambda) =
p_{\bfH_{\bflambda}}(\bfeta \, ; \bflambda)^{-1}\,\nabla_{\!\bflambda} p_{\bfH_{\bflambda}}(\bfeta \, ; \bflambda) \, ,
\end{equation}
which can be rewritten, using Eq.~\eqref{eq:eq92}, as
$\bfv(\bfeta\, ;\bflambda) = -\left (\nabla_{\!\bflambda}\Gamma(\bflambda)+\bfh^c(\bfeta)-\bfb^c\right )$
yielding  with the use of Eq.~\eqref{eq:eq29},
\begin{equation}\label{eq:eq95}
\bfv(\bfeta\, ;\bflambda) = -\left( \bfh^c(\bfeta)- E\{\bfh^c(\bfH_\bflambda)\}  \right )\, .
\end{equation}
Eq.~\eqref{eq:eq95} shows that
$E\{\Vert\, \bfv(\bfH_\bflambda ;\bflambda)\,\Vert^2\} = E\left\{\Vert\, \bfh^c(\bfH_\bflambda)- E\{\bfh^c(\bfH_\bflambda)\} \,\Vert^2\right\} = \tr\,[\cov\{\bfh^c(\bfH_\bflambda)\}]$
and using Eq.~\eqref{eq:eq30} yields,
\begin{equation}\label{eq:eq96}
E\{\Vert\, \bfv(\bfH_\bflambda ;\bflambda)\,\Vert^2\} = \int_{\RR^\nu} \Vert\, \bfv(\bfeta ;\bflambda)\,\Vert^2\,
p_{\bfH_{\bflambda}}(\bfeta \, ; \bflambda)\, d\bfeta = \tr\,[\Gamma{\,''}(\bflambda)]\, .
\end{equation}
Eq.~\eqref{eq:eq93} can be written as
$\nabla_{\!\bflambda} p_{\bfH_{\bflambda}}(\bfeta \, ; \bflambda) = p_{\bfH_{\bflambda}}(\bfeta \, ; \bflambda)\, \bfv(\bfeta\, ;\bflambda)$.
Consequently,
\begin{align}\nonumber
&\int_{\RR^\nu} \vert \,\langle \,\nabla_{\!\bflambda} p_{\bfH_{\bflambda}}(\bfeta \, ; \bflambda)\, ,  \bflambda^{\, i+1}\! - \bflambda\,
\rangle\,\vert \, d\bfeta
= \int_{\RR^\nu} \vert \, \langle \bfv(\bfeta \, ; \bflambda)\, ,  \bflambda^{\, i+1}\! - \bflambda\,\rangle\, \vert\,
p_{\bfH_\bflambda} (\bfeta \, ; \bflambda)\,d\bfeta \\
&\leq \Vert\, \bflambda^{\,i+1}\! - \bflambda \, \Vert \!\int_{\RR^\nu} \Vert\, \bfv(\bfeta\, ;\bflambda)\, \Vert\, p_{\bfH_{\bflambda}}(\bfeta \, ; \bflambda)\,d\bfeta \,\, \leq \,
\Vert\, \bflambda^{\, i+1}\! - \bflambda \, \Vert \,\left ( \int_{\RR^\nu} \! p_{\bfH_{\bflambda}}(\bfeta \, ; \bflambda)\,d\bfeta\right )^{1/2} \! \left ( \int_{\RR^\nu} \Vert\, \bfv(\bfeta\, ;\bflambda)\, \Vert^2 p_{\bfH_{\bflambda}}(\bfeta \, ; \bflambda)\,d\bfeta\right )^{1/2} .
\end{align}
Since $\int_{\RR^\nu} \! p_{\bfH_{\bflambda}}(\bfeta \, ; \bflambda)\,d\bfeta = 1$ and using Eq.~\eqref{eq:eq96}, we obtain
\begin{equation}\label{eq:eq97}
\int_{\RR^\nu} \vert \,\langle \,\nabla_{\!\bflambda} p_{\bfH_{\bflambda}}(\bfeta \, ; \bflambda)\, ,  \bflambda^{\, i+1}\! - \bflambda\,
\rangle\,\vert \, d\bfeta \,\, \leq \, \left ( \tr\,[\Gamma{\,''}(\bflambda)] \right )^{1/2}\,
\Vert\, \bflambda^{\, i+1}\! - \bflambda \, \Vert \, .
\end{equation}
From Eq.~\eqref{eq:eq90}, it can be deduced that
$\int_{\RR^\nu} \vert\, p_{\bfH_{\bflambda^{\,i+1}}} (\bfeta \, ; \bflambda^{\, i+1}) - p_{\bfH_\bflambda}(\bfeta \, ; \bflambda)\, \vert\, d\bfeta
\, \,\, \leq \,\, \int_{\RR^\nu} \vert \,\langle \,\nabla_{\!\bflambda} p_{\bfH_{\bflambda}}(\bfeta \, ; \bflambda)\, ,  \bflambda^{\, i+1}\! - \bflambda\, \rangle\,\vert \, d\bfeta  + o\,(\Vert\, \bflambda^{\, i+1} - \bflambda^{\,i}\,\Vert )$
that yields  Eq.~\eqref{eq:eq89} by using  Eq.~\eqref{eq:eq97}.
\end{proof}
\section{A few numerical elements for implementation of the methodology}
\label{sec:Section4}
\subsection{Choice of the integration scheme for solving the ISDE introduced in Propositions~\ref{proposition:3} and \ref{proposition:4}}
\label{sec:Section4.1}
As we have previously explained, for $\bflambda\in\curC_{\ad,\bflambda}$, the ISDE defined by Eqs.~\eqref{eq:eq67} to \eqref{eq:eq69}, must be solved for $t\in[0\, , t_s]$ (see Proposition~\ref{proposition:3}-(e)) with the initial condition at $t=0$ defined in Proposition~\ref{proposition:3}-(a), in order to generate the constrained learned set $\curD_{\bfH_{\!\bflambda}} = \{\bfeta_{\bflambda}^1,\ldots , \bfeta_{\bflambda}^N \}$ with $N\gg N_d$.  Therefore, a discretization scheme \cite{Kloeden1992,Talay1990}
must be used. The case of Hamiltonian dynamical systems has been analyzed in \cite{Talay2002} by using an implicit Euler scheme. The St\"ormer-Verlet scheme (see \cite{Hairer2003} for the deterministic case and \cite{Burrage2007} for the stochastic case) is a very efficient scheme that allows for having a long-time energy conservation for non-dissipative Hamiltonian dynamical systems.
In \cite{Soize2012b}, we have proposed to use an extension of the St\"ormer-Verlet scheme for stochastic dissipative Hamiltonian systems, that we have also used in \cite{Guilleminot2013a, Soize2015,Soize2016,Soize2020a}.
\subsection{St\"ormer-Verlet scheme and computation of the constrained learned set $\curD_{\bfH_{\bflambda^{\,i}}}$}
\label{sec:Section4.2}
Let $i$ be the index of the sequence $\{\bflambda^{\, i},i=0,1,\ldots, i_\pmax\}$ of Lagrange multipliers computed using Eq.~\eqref{eq:eq85} with $\bflambda^{\,0} = \bfzero_{n_c}$. Let $t_m = m\, \Delta t$ for $m=0,1,\ldots , M_s$ (with $M_s > 1$ an integer) be the time sampling in which $t_s = M_s \, \Delta t$ (and thus $t_{M_s} = t_s$).
Let $\Delta\bfW_{m+1}^{\wien} = \bfW^\wien(t_{m+1}) - \bfW^\wien(t_{m})$ be the Gaussian, second-order, centered, $\RR^\nu$-valued random variable  such that $E\{\Delta\bfW_{m+1}^{\wien}\otimes \Delta\bfW_{m+1}^{\wien}\} = \Delta t \, [I_\nu]$. Let $\{\theta_\ell, \ell=1,\ldots , N\}$ be $N$ independent realizations in $\Theta$. For $m=0,1,\ldots, M_s-1$, let $\Delta\WW_{m+1}^\ell = \Delta\bfW_{m+1}^{\wien}(\theta_\ell)$ be the realization $\theta_\ell$ of $\Delta\bfW_{m+1}^{\wien}$.
Following the choice of $(\bfu_0,\bfv_0)$ defined in Proposition~\ref{proposition:3}-(a), let $\bfu_0^1,\ldots, \bfu_0^N$ in $\RR^\nu$ such that $\forall\ell\in\{1,\ldots , N\}$, $\bfu_0^\ell = \bfeta_d^{j_\ell}$ in which $j_\ell\in\{1,\ldots , N_d\}$ is randomly drawn from the set $\{1,\ldots , N_d\}$ according to a uniform probability measure. Let $\bfv_0^1,\ldots ,\bfv_0^N$ in $\RR^\nu$ be $N$ independent realizations of the $\RR^\nu$-valued random variable $\bfV_G$ also defined in Proposition~\ref{proposition:3}-(a). Note that the realizations
$\Delta\WW_{m+1}^\ell$, $\bfu_0^\ell$, and $\bfv_0^\ell$, for $\ell=1,\ldots , N$ are independent of $\{\bflambda^{\,i}\}_i$.
For $i\in\{0,1,\ldots , i_\pmax\}$ and for $\ell\in\{1,\ldots , N\}$, we introduce the realizations
$\UU_m^{i,\ell} = \bfU_{\!\bflambda^{\, i}}^N(t_m\, ; \theta_\ell)$ and
$\VV_m^{i,\ell} = \bfV_{\!\bflambda^{\, i}}^N(t_m\, ; \theta_\ell)$.
For $m\in\{0,1,\ldots , M_s -1\}$, the St\"ormer-Verlet scheme applied to realization $\theta_\ell$ of  Eqs.~\eqref{eq:eq67} to \eqref{eq:eq69} yields the following recurrence,
\begin{align}
&\UU_{m+1/2}^{i,\ell}   =\UU_m^{i,\ell}  + \frac{\Delta t}{2}\, \VV_m^{i,\ell} \, , \label{eq:eq98}\\
&\VV_{m+1}^{i,\ell}     =\frac{1-\gamma}{1+\gamma}\, \VV_m^{i,\ell} + \frac{\Delta t}{1+\gamma}\,
                           \hat\bfL^N_{\bflambda^{\, i-1}}(\UU_{m+1/2}^{i,\ell})
                            + \frac{\sqrt{f_0}}{1+\gamma}\, \Delta\WW_{m+1}^\ell   \, , \label{eq:eq99} \\
&\UU_{m+1}^{i,\ell}     = \UU_{m+1/2}^{i,\ell} + \frac{\Delta t}{2} \, \VV_{m+1}^{i,\ell}   \, , \label{eq:eq100}
\end{align}
with the initial condition
\begin{equation}\label{eq:eq101}
\UU_{0}^{i,\ell}   =  \bfu_0^\ell \quad , \quad \VV_{0}^{i,\ell}   =  \bfv_0^\ell  \, ,
\end{equation}
in which $\gamma =f_0\, \Delta t /4$ and where, using Eq.~\eqref{eq:eq66},
\begin{equation}\label{eq:eq102}
\hat\bfL^N_{\bflambda^{\, i-1}}(\bfu) = \frac{1}{\zeta(\bfu)}\,\nabla_{\!\bfu}\zeta(\bfu) - [\nabla_{\!\bfu}\hat\bfh^N(\bfu\, ; \bflambda^{\, i-1})]\, \bflambda^{\, i-1}
\end{equation}
The matrix $[\nabla_{\!\bfu}\hat\bfh^N(\bfu\, ; \bflambda^{\, i-1})]$ is given by Eq.~\eqref{eq:eq104bis}, which depends on
\begin{equation}\label{eq:eq103}
\curD_{\bfH_{\!\bflambda^{\,i-1}}} = \{\bfeta_{\bflambda^{\,i-1}}^1,\ldots , \bfeta_{\bflambda^{\,i-1}}^N \} \, .
\end{equation}
It should be noted that, in Eq.~\eqref{eq:eq99}, $\hat\bfL^N_{\bflambda^{\, i-1}}$ has been used instead of $\hat\bfL^N_{\bflambda^{\, i}}$
because $\hat\bfL^N_{\bflambda^{\, i}}$ depends on $\curD_{\bfH_{\!\bflambda^{\,i}}}$ that is unknown, recalling that the aim of the recurrence defined by  Eqs.~\eqref{eq:eq98} to \eqref{eq:eq101} is precisely to calculate   $\curD_{\bfH_{\!\bflambda^{\,i}}}$ that is written as
\begin{equation}\label{eq:eq104}
\curD_{\bfH_{\!\bflambda^{\,i}}} = \{\bfeta_{\bflambda^{\,i}}^1,\ldots , \bfeta_{\bflambda^{\,i}}^N \} \quad , \quad
\bfeta_{\bflambda^{\,i}}^\ell = \bfU_{\!\bflambda^{\, i}}^N(t_s\, ; \theta_\ell) = \UU_{M_s}^{i,\ell} \, .
\end{equation}
\subsection{Explicit expression  of the gradient of the statistical surrogate model of $\bfh^c$}
\label{sec:Section4.3}
Using Definition~\ref{definition:3}, for fixed value of $\bflambda\in\curC_{\ad,\bflambda}$, the gradient $[\nabla_{\!\bfeta} \hat\bfh^N(\bfeta\,;\bflambda)]\in\MM_{\nu,n_c}$ at point $\bfeta\in\RR^\nu$ of the statistical surrogate model $\hat\bfh^N$ of $\bfh^c$ can be written as
\begin{equation}\label{eq:eq104bis}
[\nabla_{\!\bfeta} \hat\bfh^N(\bfeta\,;\bflambda)] = \sum_{\ell=1}^N \bfgamma_\bflambda^\ell\otimes \bfa_\bflambda^\ell \quad , \quad
\bfgamma_\bflambda^\ell = \nabla_{\!\bfeta} \left ( \frac{\beta_\bfeta^N(\bfeta_\bflambda^\ell)}{\sum_{\ell'=1}^N\beta_\bfeta^N(\bfeta_\bflambda^{\ell'})} \right )\, ,
\end{equation}
in which $\bfa_\bflambda^\ell = \bfh^c(\bfeta_\bflambda^\ell) \, \in \, \RR^{n_c}$ and where $\bfgamma_\bflambda^\ell \in\RR^\nu$ is explicitly calculated using Eqs.~\eqref{eq:eq39} and \eqref{eq:eq40}.
\subsection{Summary of the complete algorithm}
\label{sec:Section4.4}
The algorithm for calculating $\bflambda^\psol$ and $\curD_{\bfH^c} = \{ \bfeta_c^1,\ldots,\bfeta_c^N \}$ with
$\bfeta_c^\ell = \bfeta_{\bflambda^\psol}^\ell$ for $\ell=1,\ldots, N$ is summarized in Algorithm~\ref{algorithm:1}.
%
%
\begin{algorithm}
\caption{Algorithm for calculating $\bflambda^\psol$ and $\curD_{\bfH^c} = \{ \bfeta_c^1,\ldots,\bfeta_c^N \}$.}
\label{algorithm:1}
\begin{algorithmic}[1]
\State{\textbf{Data:}$N_d$, $\curD_d = \{ \bfeta_d^1,\ldots,\bfeta_d^{N_d} \}$, $N$, $i_\pmax$, $M_s$, $t_s$, $\Delta t$, $f_0$, $\gamma=f_0\,\Delta t / 4$}
\State{\textbf{Init:} $\,\,\Delta\WW_{m+1}^{\ell}, \ell \in\{ 1,\ldots , N \}, m\in \{1,\ldots , M_s-1\}$,
                  $\,\,\bfu_0^\ell$ and $\bfv_0^\ell$ for $\ell\in\{1,\ldots , N\}$, $\bflambda^{\,0} = \bfzero_{n_c}$ }
\For{$i=1:i_\pmax$}
    \For{$\ell=1:N\, (loop \, in \,  parallel \, computation)$}
    \State{$\curD_{\bfH_{\!\bflambda^{\,i}}} = \{\bfeta_{\bflambda^{\,i}}^1,\ldots , \bfeta_{\bflambda^{\,i}}^N \}$ from Eq.~\eqref{eq:eq104}, using Eqs.~\eqref{eq:eq98} to \eqref{eq:eq101} and $\curD_{\bfH_{\!\bflambda^{\, i-1}}}$ ($\curD_{\bfH_{\!\bflambda^{\, 0}}}$ not used for $i=1$) }
    \EndFor
    \For{$\ell=1:N\, (loop \, in \,  parallel \, computation)$}
    \State{$\bfh^c(\bfeta_{\bflambda^{\, i}}^\ell), \ell=1,\ldots , N$ using the BVP}
    \EndFor
    \State{$\nabla_{\!\bflambda}\Gamma(\bflambda^{\, i})$ and $[\Gamma{\,''}(\bflambda^{\, i})]$ using
           Eqs.~\eqref{eq:eq29} to \eqref{eq:eq30} and  $\curD_{\bfH_{\!\bflambda^{\, i}}}$ }
    \State{$\error(i)$ using Eqs.~\eqref{eq:eq87} with \eqref{eq:eq88} }
    \State{$\bflambda^{\, i+1} = \bflambda^{\, i} - \alpha_{\prelax}\, [\Gamma{\,''}(\bflambda^{\, i})]^{-1}\, \nabla_{\!\bflambda}\Gamma(\bflambda^{\, i})$ using  Eq.~\eqref{eq:eq85} with a relaxation factor $\alpha_{\prelax}\in ]0\, , 1]$ }
    \State{$\bflambda^{\, i} \leftarrow \bflambda^{\, i+1}$}
    \State{$\curD_{\bfH_{\!\bflambda^{\,i-1}}} \leftarrow \curD_{\bfH_{\!\bflambda^{\,i}}}$ }
\EndFor
\State{$\bflambda^\psol = \bflambda^{i_\ppsol}$, $i_\psol = \arg\,\min_i\error(i)\,\,$ from Eq.~\eqref{eq:eq86}}
\State{$\curD_{\bfH^c} \leftarrow \curD_{\bfH_{\bflambda^\ppsol}}$}
\end{algorithmic}
\end{algorithm}
\section{Application to stochastic homogenization without scale separation}
\label{sec:Section5}
In this section, we consider the stochastic boundary value problem associated with the stochastic homogenization of a random elastic medium without scale separation, which has been presented in Section~\ref{sec:Section1}. The physical space $\RR^3$ is referred to a Cartesian reference system whose the generic point is $\bfxi = (\xi_1,\xi_2,\xi_3)$. We consider the stochastic homogenization  of a heterogeneous  linear elastic microstructure occupying the 3D bounded open domain $\Omega = ]\, 0\, , 1\, [\times  ]\, 0\, , 1\, [\times ]\, 0\, , 0.1\, [\subset\RR^3$ (square thick plate) with boundary $\partial\Omega$. The homogenization method on $\Omega$ in the one proposed in \cite{Bornert2008} that we have already used in \cite{Soize2008,Soize2021c}. In this section, we use the convention for summation over repeated Latin indices taking values in $\{1,2,3\}$.
\subsection{Stochastic elliptic boundary value problem}
\label{sec:Section5.1}
For all $m$ and $r$ in $\{1,2,3\}$ the unknown field is the $\RR^3$-valued random field
$\{ \bfY(\bfxi) = (Y_1(\bfxi),Y_2(\bfxi),Y_3(\bfxi)),\bfxi\in\Omega\}$ defined on $(\Theta,\curT,\curP)$, indexed by $\Omega$, such that for $i=1,2,3$, and almost surely,
\begin{equation}\label{eq:eq105}
-\frac{\partial}{\partial\xi_j} \left (\CC_{ijpq}(\bfxi)\, \varepsilon_{pq}(\bfY^{mr}(\bfxi))\right ) = \bfzero_3\quad , \quad \forall\bfxi\in\Omega\, ,
\end{equation}
\begin{equation}\label{eq:eq106}
\bfY^{mr}(\bfxi) = \bfy_0^{mr}\quad , \quad \forall\bfxi\in\partial\Omega\, ,
\end{equation}
in which the strain tensor is $\varepsilon_{pq}(\bfy) = (\partial y_p / \partial \xi_q +\partial y_q / \partial \xi_p)/2$
for all $\bfy = (y_1,y_2,y_3)$. For all $\bfxi\in \partial\Omega$,
$\bfy_0^{mr} = ( y_{0,1}^{mr}, y_{0,2}^{mr}, y_{0,3}^{mr})$ is defined by
\begin{equation}\label{eq:eq107}
y_{0,j}^{mr} = (\delta_{jm}\,\xi_r +\delta_{jr}\,\xi_m)/2\quad , \quad j\in\{1,2,3\} \, ,
\end{equation}
in which $\delta_{jm}$ is the Kronecker symbol. At mesoscale, the linear  elastic heterogeneous medium is described  by the random apparent elasticity field $\{\CC(\bfxi),\bfxi\in\RR^3\}$, which is a non-Gaussian fourth-order tensor-valued random field
$\CC= \{ \CC_{ijpq} \}_{ijpq}$ with $i$, $j$, $p$, and $q$ in $\{1,2,3\}$, defined on $(\Theta,\curT,\curP)$.
The stochastic homogenization consists, for $i$, $j$, $m$, and $r$ in $\{1,2,3\}$, in analyzing at macroscale the component $\CC^\peff_{ijmr}$ of the random effective elasticity tensor $\{\CC^\peff_{ijmr}\}_{ijmr}$, which is defined by
\begin{equation}\label{eq:eq108}
\CC^\peff_{ijmr} = \frac{1}{\vert\Omega\vert} \int_\Omega \CC_{ijpq}(\bfxi)\, \varepsilon_{pq}(\bfY^{mr}(\bfxi)) \, d\bfxi \, ,
\end{equation}
in which $\bfY^{mr}$ is the $\RR^3$-valued random field that satisfies Eqs.~\eqref{eq:eq105} to \eqref{eq:eq107} and where
$\vert\Omega\vert = \int_\Omega d\bfxi$. The random effective elasticity tensor $\CC^\peff$ is symmetric and positive definite almost surely. If there was a scale separation, then the statistical fluctuations of this  tensor would be negligible.
\subsection{Prior probability model of random field $\CC$}
\label{sec:Section5.2}
The prior probability model  of $\CC$ used for generating the training set is the one presented  in \cite{Soize2021c,Soize2021d}. This is a second-order, non-Gaussian, positive-definite fourth-order tensor-valued  homogeneous random field, indexed by $\RR^3$, defined on $(\Theta,\curT,\curP)$, with a spectral random measure. This random field is parameterized as
\begin{equation}\label{eq:eq109}
\CC(\bfxi) = \cc(\bfG(\bfxi),\underline\bfz) \quad , \quad \bfxi\in\Omega\, ,
\end{equation}
in which $\{\bfG(\bfxi),\bfxi\in\RR^3\}$ is a non-Gaussian second-order, homogeneous, $\RR^{21}$-valued random field indexed by $\RR^3$, defined on $(\Theta,\curT,\curP)$, with random spectral measure, where
$\underline\bfz =(\underline z_1, \underline z_2,\underline z_3)$ is the nominal value of a $\RR^3$-valued control parameter,
and where $\cc$ is a given mapping from $\RR^{21}\times\RR^3$ into the fourth-order tensor on $\RR^3$.\\

\noindent (i) \textit{Isotropic mean model at mesoscale}. At mesoscale, the mean model is a linear, elastic, homogeneous, isotropic medium whose elasticity tensor $\underline\CC$ depends only on the Young modulus $\underline E = 1.7\times 10^{11}\, N/m^2)$ and  Poisson coefficient $\underline\nu_P = 0.24$. The corresponding bulk modulus $\underline C_{\,\bulk} = \underline E \, /(3(1-2\underline \nu))$ and
shear modulus $\underline C_{\,\shear} = \underline E \, /(2(1+\underline \nu))$ are $1.08974\times 10^{11}\, N/m^2$ and $6.85484\times 10^{10}\, N/m^2$.\\

\noindent (ii) \textit{Anisotropic statistical fluctuations at mesoscale}. The statistical fluctuations of the random medium are assumed to be  anisotropic, which means that, for all $\bfxi$ fixed in $\RR^3$, the random apparent elasticity tensor $\CC(\bfxi)$ is a full anisotropic tensor. The hyperparameters that control the  anisotropic statistical fluctuations of the random apparent elasticity field (see \cite{Soize2021c}) are:

\noindent (1) the dispersion coefficient $\underline\delta_{\,\CC} = 0.3$ that controls the level of statistical fluctuations of the random medium.

\noindent (2) the spatial correlation lengths $\underline L_{\,c1}$, $\underline L_{\,c2}$, and $\underline L_{\,c3}$ (for directions $\xi_1$, $\xi_2$, and $\xi_3$) of the random field $\{\bfG(\bfxi),\bfxi\in\RR^3\}$ and the dispersion coefficient $\underline\delta_{\,s} = 0.1$ that controls the level of uncertainties of its spectral measure (these spectral-measure uncertainties will not be controlled, which means that $\underline\delta_{\, s}$ will not be a control parameter and its value is fixed). As explained at the end of Section~\ref{sec:Section1}, three cases, SC1, SC2, and SC3, of the correlation lengths are considered for analyzing the level of scale separation and are defined in Table~\ref{table:table1}. Taking into account the definition of domain $\Omega$ and the values of the spatial correlation lengths, there will not have a scale separation and consequently, the effective elasticity tensor will exhibit statistical fluctuations.\\
\begin{table}[h]
  \caption{Values of the spatial correlation lengths $\underline L_{\,c1}$, $\underline L_{\,c2}$, and $\underline L_{\,c3}$ for cases SC1, SC2, and SC3 of scale separation.}\label{table:table1}
\begin{center}
 \begin{tabular}{|c|c|c|c|} \hline
    & $\underline L_{\,c1}$ & $\underline L_{\,c2}$ & $\underline L_{\,c3}$ \\
    \hline
    SC1   &  0.1 & 0.1 & 0.1   \\
    SC2   &  0.3 & 0.3 & 0.1   \\
    SC3   &  0.5 & 0.5 & 0.2  \\
    \hline
  \end{tabular}
\end{center}
\end{table}

\noindent (iii) \textit{Nominal value $\underline\bfz$ of the control parameter}. It is defined by
$\underline z_1 = \underline C_{\,\bulk}$, $\underline z_2 = \underline C_{\,\shear}$, and $\underline z_3= \underline\delta_{\, \CC}$.\\

\noindent (iv) \textit{Random control parameter for the probabilistic learning inference}. For estimating the posterior model using the probabilistic learning inference methodology, in addition to the prior probability model of random field $\bfG$, we introduce a $\RR^{n_w}$-valued random control parameter $\bfW=(W_1,\ldots, W_{n_w})$ defined on $(\Theta,\curT,\curP)$ and independent of $\bfG$ such that, $n_w=3$ and
\begin{equation}\label{eq:eq110}
W_1 = \log C_\bulk \quad , \quad W_2 = \log C_\shear \quad , \quad W_3 = \log \delta_\CC \, ,
\end{equation}
in which (a) $C_\bulk$ and $C_\shear$ are Gamma independent random variables (see \cite{Guilleminot2013}) whose mean values are the nominal values $\underline C_{\,\bulk}$ and $\underline C_{\,\shear}$ previously defined and for which the coefficient of variation of $C_\bulk$ is chosen as $\delta_\bulk = 0.5$ yielding $\delta_\shear = 0.25$ (note that, with the model proposed in \cite{Guilleminot2013}, $\delta_\shear$ is deduced from $\delta_\bulk$ and cannot arbitrarily be chosen); (b) $\delta_\CC$ is chosen as a uniform random variable on $[0.1\, , 0.5]$ whose mean values is $\underline\delta_{\, \CC}$.\\

\noindent It should be noted that the statistical fluctuations of $C_\bulk$, $C_\shear$, and $\delta_\CC$ are chosen sufficiently large in order that the range of fluctuations of the random effective elasticity tensor covers the experimental target (see Section~\ref{sec:Section5.5}) in order to be able to improve the prior probabilistic model with the posterior probabilistic model by solving the inverse statistical problem with the proposed probabilistic learning inference approach.
\subsection{Solution of the stochastic BVP}
\label{sec:Section5.3}
Under the hypotheses introduced for constructing random field $\CC$, Proposition~5.1 of \cite{Soize2021d} proves that for $1\leq m\leq r \leq 3$, the strong stochastic solution $\{\bfY^{mr}(\bfxi),\bfxi\in\Omega\}$ of the weak formulation of the stochastic elliptic BVP defined by Eqs.~\eqref{eq:eq105} to \eqref{eq:eq107} exists, is unique, and is a second-order random field,
\begin{equation}\label{eq:eq110bis}
E\{\,\Vert\, \bfY^{mr}(\bfxi)\,\Vert^2 \} < + \infty \quad ,\quad \forall \bfxi\in\, \Omega \, .
\end{equation}
Due to Corollary~5.1  of \cite{Soize2021d} and its proof, the random effective elasticity tensor $\CC^\peff$ is a second-order random variable,
\begin{equation}\label{eq:eq110ter}
\sum_{1\leq i\leq j\leq 3} \sum_{1\leq m\leq r \leq 3} E\{ (\CC^\peff_{ijmr})^2 \} < + \infty \, .
\end{equation}
\subsection{Stochastic computational model and random effective elasticity matrix}
\label{sec:Section5.4}
The finite element method is used for discretizing the weak formulation of the stochastic BVP. The finite element mesh of domain $\overline\Omega$  is made up of $60 \times 60 \times 6 = 21\, 600$ eight-nodes solid elements, $26\, 047$ nodes, and $78\, 141$ dofs
($25\,926$ Dirichlet conditions on $\partial\Omega$ and $n_y = 52\, 215$ the remaining dofs). There are $2^3$ integration points in each finite element, which yields $n_p= 172\, 800$ integration points for the spatial discretization of the random fourth-order tensor-valued elasticity field $\{\CC(\bfxi),\bfxi\in\Omega\}$. The discretization of random field $\CC$ is expressed as a function of a $\RR^{n_g}$-valued random variable $\bfcurG$ corresponding to the spatial discretization of random field $\bfG$ with $n_g= 21\times n_p = 3\, 628\, 800$. For $1\leq m\leq r \leq 3$, let $\bfcurY^{m r}$ be the $\RR^{n_y}$-valued random variable of the free dofs (that corresponds to the dofs of the nodes inside $\Omega$ of  the finite element discretization of random field $\{\bfY^{mr}(\bfxi),\bfxi\in\overline\Omega\}$). Therefore, the stochastic computational model can be written as the $\RR^{n_y}$-valued stochastic equation,
\begin{equation}\label{eq:eq111}
\bfcurN^{mr}(\bfcurY^{mr}\! , \bfcurG ,\bfW) = \bfzero_{n_y} \, \, a.s. \, ,
\end{equation}
which is a stochastic linear equation that can be rewritten as
$[a^{mr}(\bfcurG ,\bfW)] \, \bfcurY^{mr} - \bfb^{mr}(\bfcurG ,\bfW) = \bfzero_{n_y}$ in which, for $\bfg_d\in\RR^{n_g}$ and $\bfw_d\in\RR^{n_w}$,  $[a^{mr}(\bfg_d ,\bfw_d)]$ is a matrix in $\MM^+_{n_y}$ (thus, invertible) and where $\bfb^{mr}(\bfg_d ,\bfw_d)$ is a vector in $\RR^{n_y}$, which depends on the Dirichlet condition defined by Eq.~\eqref{eq:eq106}.
For $1\leq m\leq r\leq 3$, the finite element discretization of the right-hand side member of Eq.~\eqref{eq:eq108} yields the $\MM^+_6$-valued random effective elasticity matrix $[\CC^\peff]$ such that
$[\CC^\peff]_{\textbf{i}\textbf{j}} = \CC^\peff_{ijmr}$ in which the indices $\textbf{i}=(i,j)$ with $1\leq i\leq j\leq 3$ and $\textbf{j}=(m,r)$ with $1\leq m\leq r\leq 3$ are with values in $\{1,\ldots , 6\}$. This random matrix can be written as
\begin{equation}\label{eq:eq112}
[\CC^\peff] = [\curO(\{ \bfcurY^{mr}\! ,1\leq m\leq r\leq 3\},\bfcurG,\bfW)]\, ,
\end{equation}
in which $(\{ \bfy^{mr}\! ,1\leq m\leq r\leq 3\} ,\bfg ,\bfw)\mapsto [\curO(\{ \bfy^{mr}\! ,1\leq m\leq r\leq 3\} ,\bfg ,\bfw)]$ is a
measurable mapping from $\RR^{6\times n_y}\times \RR^{n_g}\times \RR^{n_w}$ into $\MM_6^+$. Note that, for $1\leq m\leq r\leq 3$, $\bfcurY^{mr}$ satisfying Eq.~\eqref{eq:eq111} corresponds to the strong stochastic solution of the finite element discretization of the weak formulation of the stochastic BVP (see Section~\ref{sec:Section5.3}). Due to Eqs.~\eqref{eq:eq110bis} and \eqref{eq:eq110ter}, we have,
\begin{equation}\label{eq:eq113}
E\{\,\Vert\bfcurY^{mr}\,\Vert^2\} < +\infty \quad , \quad E\{\,\Vert \,[\CC^\peff]\,\Vert^2_F\} < +\infty \, ,
\end{equation}
which proves that $\bfcurY^{mr}$ and $[\CC^\peff]$ are second-order random variables.
\subsection{Definition of the statistical moments and their targets}
\label{sec:Section5.5}
From Eq.~\eqref{eq:eq113}, the $\MM^+_6$-valued random variable  $[\CC^\peff]$ is of second-order. We can then define its first two moments.

\noindent (i) The first statistical moment of interest is the mean value
                        $[\,\underline\CC^\peff] = E\{\,[\CC^\peff]\,\} \in \MM^+_6$
of random matrix $[\CC^\peff]$ while its target counterpart is the given matrix
                        $[\,\underline\CC^\pexp] \in \MM^+_6$.
Let $\mu_\pexp = \Vert\, [\,\underline\CC^\pexp] \, \Vert_F$ be the Frobenius norm of $[\,\underline\CC^\pexp]$.
Introducing the subscript "$n$" to designate a normalization, we define the normalized quantities with respect to $\mu_\pexp$ as,
\begin{equation}\label{eq:eq114}
[\CC^\peff_n ] = \frac{1}{\mu_\pexp}\, [\CC^\peff] \quad , \quad [\,\underline\CC^\peff_{\,n}] = \frac{1}{\mu_\pexp}\,[\,\underline\CC^\peff] \quad , \quad [\,\underline\CC^\pexp_{\, n}] = \frac{1}{\mu_\pexp}\,[\,\underline\CC^\pexp]\, .
\end{equation}
The corresponding constraint equation will then be written as
\begin{equation}\label{eq:eq115}
E\{\,[\CC^\peff_n ]\,\} = [\,\underline\CC^\pexp_{\, n}] \, .
\end{equation}
\noindent (ii) The second statistical moment of interest is the coefficient of dispersion $\delta^{\,\peff}$ of random matrix $[\CC^\peff]$ and its target counterpart $\delta^{\,\pexp}$. Let $\Delta_2^\peff$ be the positive-valued random variable defined by
\begin{equation}\label{eq:eq116}
\Delta_2^\peff = \frac{1}{\Vert\, [\,\underline\CC^\peff]\,\Vert^2_F} \, \Vert\, [\CC^\peff] - [\,\underline\CC^\peff]\,\Vert^2_F \, .
\end{equation}
Consequently, $\delta^{\,\peff}$ that is defined by
$\delta^{\,\peff} = \left ( E\{ \, \Vert\, [\CC^\peff] - [\,\underline\CC^\peff]\, \Vert^2_F  /
\Vert \, [\,\underline\CC^\peff]\,\Vert^2_F \, \} \right )^{1/2}$ can be rewritten as
\begin{equation}\label{eq:eq117}
\delta^{\,\peff} = \sqrt{E\{\Delta_2^\peff\}}\, .
\end{equation}
Defining $\mu_\peff = \Vert\, [\,\underline\CC^\peff] \, \Vert_F$, the constraint equation that is defined by
\begin{equation}\label{eq:eq118}
\delta^{\,\peff} = \delta^{\,\pexp}\, ,
\end{equation}
can be rewritten, using Eqs.~\eqref{eq:eq114} to \eqref{eq:eq117}, as
\begin{equation}\label{eq:eq119}
E\{\, \Vert\, [\CC^\peff_n] - [\,\underline\CC^\peff_{\,n}]\,\Vert^2_F \,\} = \left( \frac{\mu_\peff}{\mu_\pexp}\, \delta^{\,\pexp}\right )^2\, .
\end{equation}
while Eq.~\eqref{eq:eq116} yields
\begin{equation}\label{eq:eq120}
\Delta_2^\peff =  \left(\frac{\mu_\pexp}{\mu_\peff}\right )^2\, \Vert\, [\CC^\peff_n] - [\,\underline\CC^\peff_{\,n}]\,\Vert^2_F \,\} \, .
\end{equation}
It should be noted that, if $\delta^{\,\peff}$ goes to zero, then the statistical fluctuations of $[\CC^\peff]$ goes to zero because, due to the Tchebychev inequality, $[\CC^\peff]$ goes in probability to its mean value $[\,\underline\CC^\peff]$ (this would be the case of a scale separation). For the three considered cases of scale separation, the numerical values of $\mu_\peff$ for the training set (and thus denoted by $\mu_{\peff,d}$) and $\mu_\pexp$ for the target are given in Table~\ref{table:tableA1}.
\begin{table}[h]
  \caption{For cases SC1, SC2, and SC3, values of $\mu_{\peff,d}$ and $\mu_\pexp$.}\label{table:tableA1}
\begin{center}
 \begin{tabular}{|c|c|c|c||c|c|} \hline
    & $\underline L_{\,c1}$ & $\underline L_{\,c2}$ & $\underline L_{\,c3}$ & $\mu_{\peff,d}\times 10^{11}$ & $\mu_{\pexp}\times 10^{11}$\\
    \hline
    SC1   &  0.1 & 0.1 & 0.1 & 4.2106 & 4.6317  \\
    SC2   &  0.3 & 0.3 & 0.1 & 4.1925 & 4.6549  \\
    SC3   &  0.5 & 0.5 & 0.2 & 4.1943 & 4.6706  \\
    \hline
  \end{tabular}
\end{center}
\end{table}
\subsection{Training set computed with the prior probability model and its normalization}
\label{sec:Section5.6}
The stochastic computational model defined in Section~\ref{sec:Section5.4} is used  for generating the training set related to the random variable
$\bfX =(\, \{\bfcurY^{mr}\! ,1\leq m\leq r\leq 3\},\bfcurG,\bfW)$ with  values in $\RR^{n_x} = \RR^{6\times n_y}\times\RR^{n_g}\times\RR^{n_w}$ with $n_x = 6\, n_y + n_g + n_w = 3\, 942\, 093$.
The Monte Carlo numerical simulation method is used with $N_d=50$ independent realizations and the prior probability model of $\bfcurG$ and $\bfW$. We then obtain the training set
$\{\bfx_d^j,j=1,\ldots, N_d\}$ with
$\bfx_d^j = ( \,\{\bfy_d^{mr,j}\! , 1\leq m\leq r\leq 3\},\bfg_d^j,\bfw_d^j)$ in which
$\bfy_d^{mr,j}\in\RR^{n_y}$ is the solution (see Eq.~\eqref{eq:eq111}) of
\begin{equation}\label{eq:eq121}
\bfcurN^{mr} (\bfy_d^{mr,j},\bfg_d^j,\bfw_d^j) = \bfzero_{n_y}\, ,
\end{equation}
that is to say of the linear equation $[a^{mr}(\bfg_d^j,\bfw_d^j)] \, \bfy_d^{mr,j} = \bfb^{mr}(\bfg_d^j,\bfw_d^j)$ (see Section~\ref{sec:Section5.4}).
The statistical moments defined in Section~\ref{sec:Section5.5} can then be computed using the $N_d$ independent realizations $\{\,[\CC_d^{\peff,j}],j=1,\ldots , N_d\}$ of the $\MM_6^+$-valued random matrix $[\CC_d^\peff]$ defined by Eq.~\eqref{eq:eq112} (a subscript "d" is introduced to designate the computation done with the training set, based on the prior probability model).
Since $N_d\ll n_x$, the normalization of $\bfX$ is performed  using the principal component analysis as follows.
Let $\bfx_{\rm{ctr}}^j = \bfx_d^j - \underline\bfx$ with $\underline\bfx  = (1/N_d)\, \sum_{j=1}^{N_d} \bfx_d^j \in \RR^{n_x}$.
Let $[x_{\rm{ctr}}] = [\bfx_{\rm{ctr}}^1 \ldots \bfx_{\rm{ctr}}^{N_d}]$ be the matrix in $\MM_{n_x,\, N_d}$ and let $[\Phi]\,[S]\,[\Psi]^T = [x_{\rm{ctr}}]$ be the thin SVD \cite{Golub1993} (economy size SVD) of matrix $[x_{\rm{ctr}}]$. The diagonal entries of diagonal matrix $[S]$ are the singular values  $S_1\geq \ldots \geq S_{N_d-1} > S_{N_d} = 0$ that are in decreasing order  and  we have $S_{N_d} = 0$. The matrix $[\Phi]$ is in $\MM_{n_x,\nu}$ with $\nu=N_d-1$ and $[\Phi]^T\, [\Phi]=[I_\nu]$. Then random vector $\bfX$ can be written as
\begin{equation}\label{eq:eq122}
\bfX = \underline\bfx + [\Phi] \, [\kappa]^{1/2}\, \bfH \, ,
\end{equation}
in which $[\kappa]$ is the diagonal matrix such that $\kappa_\alpha = [\kappa]_{\alpha\alpha} = S_\alpha^2/(N_d-1)$, and where $\bfH$ is the $\RR^\nu$-valued random variable whose $N_d$ independent realizations are computed by
\begin{equation}\label{eq:eq123}
\bfeta_d^j = [\kappa]^{-1/2}\,[\Phi]^T (\bfx_d^j - \underline \bfx) \quad , \quad j=1,\ldots , N_d \, .
\end{equation}
Note that $\{\kappa_\alpha\}_\alpha$ are the eigenvalues of the covariance matrix of $\bfX$ estimated with $\{\bfx_d^1,\ldots , \bfx_d^{N_d}\}$.
Random vector $\bfH$ is then normalized. The empirical estimation of its mean value and its covariance matrix are given by Eq.~\eqref{eq:eq10} (centered and identity matrix). From Eq.~\eqref{eq:eq122}, it can be deduced that, for $1\leq m \leq r\leq 3$, we have
\begin{align}
 \bfcurY^{mr}  = \underline\bfy^{mr} + [\Phi_y^{mr}] \, [\kappa]^{1/2} \, \bfH \quad & , \quad \underline\bfy^{mr}
                \in \RR^{n_y} \quad , \quad [\Phi_y^{mr}] \in\MM_{n_y,\nu}                                        \, , \label{eq:eq124} \\
 \bfcurG  = \underline \bfg + [\Phi_g]\,  \, [\kappa]^{1/2} \, \bfH \quad & , \quad \underline\bfg \in \RR^{n_g}
                 \quad , \quad [\Phi_g] \in\MM_{n_g,\nu}                                                     \, , \label{eq:eq125} \\
 \bfW  = \underline \bfw + [\Phi_w]\,  \, [\kappa]^{1/2} \, \bfH \quad & , \quad \underline\bfw \in \RR^{n_w}
                 \quad , \quad [\Phi_w] \in\MM_{n_w,\nu}                                                     \, . \label{eq:eq126}
\end{align}
The training set $\curD_d$ introduced in Definition~\ref{definition:1} is written as
\begin{equation}\label{eq:eq127}
\curD_d = \{\bfeta_d^1,\ldots ,\bfeta_d^{N_d}\} \quad ,\quad \bfeta_d^j \in\RR^\nu\, .
\end{equation}
Figure~\ref{fig:figure1} displays the distribution of the eigenvalues $\kappa_\alpha$ for the 3 cases, SC1, SC2, and SC3 of scale separation, computed with the training set and the prior probability model.
\begin{figure}[h]
\centering
\includegraphics[width=5.5cm]{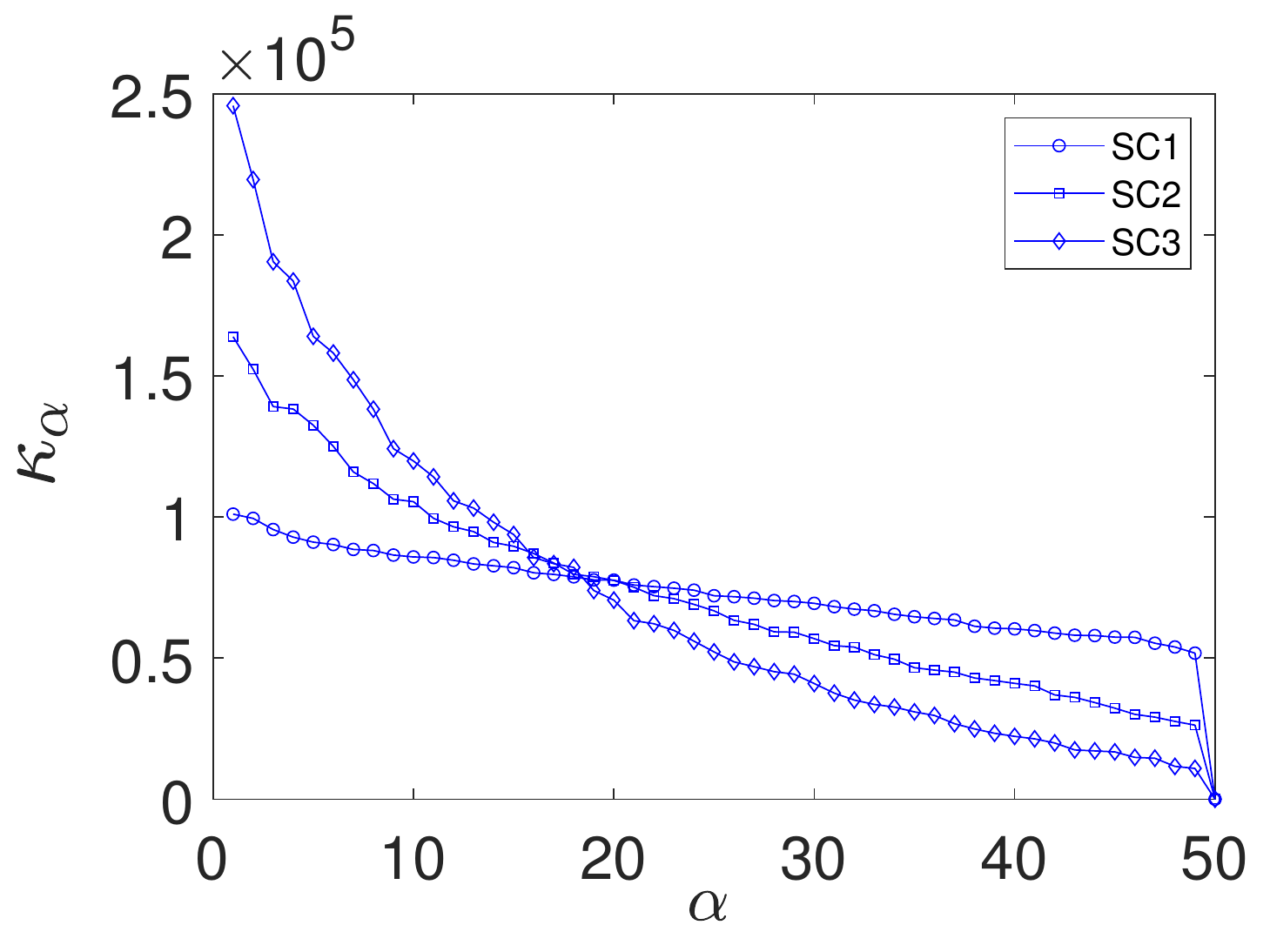}
\caption{For cases, SC1, SC2, and SC3, distribution of the eigenvalues $\kappa_\alpha$.}
\label{fig:figure1}
\end{figure}
\subsection{Definition of the random normalized residue induced by the use of the constrained learned set}
\label{sec:Section5.7}
The construction of the random normalized residue is based on a similar approach of the one that has been used in \cite{Soize2021a}.
Let us consider the constrained learned set
$\curD_{\bfH_{\bflambda^{\,i}}} = \{ \bfeta_{\bflambda^{\,i}}^1,\ldots , \bfeta_{\bflambda^{\,i}}^N \}$ generated with Algorithm~\ref{algorithm:1} for iteration $i$. Using Eqs.~\eqref{eq:eq124} to \eqref{eq:eq126}, for $\ell\in\{1,\ldots , N\}$, the corresponding realizations
$\bfy_{\bflambda^{\,i}}^{mr,\ell} \in\RR^{n_y}$ for $1\leq m\leq r\leq 3$, $\bfg_{\bflambda^{\,i}}^{\ell}\in\RR^{n_g}$, and
$\bfw_{\bflambda^{\,i}}^{\ell}\in\RR^{n_w}$, are computed by
\begin{align}
 \bfy_{\bflambda^{\,i}}^{mr,\ell} & = \underline\bfy^{mr} + [\Phi_y^{mr}] \, [\kappa]^{1/2} \,
                                                                                \bfeta_{\bflambda^{\,i}}^\ell  \, , \label{eq:eq128} \\
 \bfg_{\bflambda^{\,i}}^{\ell}  & = \underline \bfg + [\Phi_g]\,  \, [\kappa]^{1/2} \,
                                                                                \bfeta_{\bflambda^{\,i}}^\ell  \, , \label{eq:eq129} \\
 \bfw_{\bflambda^{\,i}}^{\ell}  & = \underline \bfw + [\Phi_w]\,  \, [\kappa]^{1/2} \,
                                                                                \bfeta_{\bflambda^{\,i}}^\ell  \, . \label{eq:eq133}
\end{align}
For $1\leq m\leq r\leq 3$ and for $\ell\in\{1,\ldots , N\}$ the realization $\bfcurR_{\bflambda^{\,i}}^{\,mr,\ell}$ of the $\RR^{n_y}$-valued random residue are computed using Eq.~\eqref{eq:eq111} and Eqs.~\eqref{eq:eq128} to \eqref{eq:eq133},
$\bfcurR_{\bflambda^{\,i}}^{\,mr,\ell} = \bfcurN^{mr}(\bfy_{\bflambda^{\,i}}^{mr,\ell},\bfg_{\bflambda^{\,i}}^{\ell}, \bfw_{\bflambda^{\,i}}^{\ell})$.
We define the realization $\hat\rho_{\bflambda^{\,i}}^\ell$ of the random residue $\hat\rho_{\bflambda^{\,i}}$ by
\begin{equation}\label{eq:eq135}
\hat\rho_{\bflambda^{\,i}}^\ell  =\frac{1}{\sqrt{6\, n_y}} \left ( \sum_{1\leq m\leq r \leq 3}
                    \Vert\, \bfcurR_{\bflambda^{\,i}}^{\,mr,\ell} \,\Vert^2 \right )^{1/2} \, .
\end{equation}
Finally, we define the realization $\rho_{\bflambda^{\,i}}^\ell$ of the random normalized residue $\rho_{\bflambda^{\,i}}$ by
\begin{equation}\label{eq:eq136}
\rho_{\bflambda^{\,i}}^\ell  =\frac{\hat\rho_{\bflambda^{\,i}}^\ell }{\hat{\underline\rho}_{\,0}}\, ,
\end{equation}
in which $\hat{\underline\rho}_{\,0}$ is the estimation of $E\{\,\rho_{\bflambda^{1}}\}$ using the constrained learned set
$\curD_{\bfH_{\bflambda^{1}}}$ of the first iteration $i=1$. Since for $i=1$, $\bflambda^{\, i-1} = \bflambda^{\,0} = \bfzero_{n_c}$,
 $ \rho_{\bflambda^{1}}$ is the random normalized residue of the constrained learned set without taking into account the constraints and we have, for $i=1$,
$E\{\, \rho_{\bflambda^{1}}\} = {E\{\,\hat\rho_{\bflambda^{1}}\} }/{\hat{\underline\rho}_{\,0}} = 1$.
\subsection{Defining function $\bfh^c$ related to the constraints and defining the targets represented by $\bfb^c$}
\label{sec:Section5.8}
We define the function $\bfeta\mapsto\bfh^c(\bfeta): \RR^\nu\rightarrow \RR^{n_c}$ related to the constraints (see Eq.~\eqref{eq:eq2}) and we define the target represented by vector $\bfb^c$ given in $\RR^{n_c}$. Three constraints are introduced, the second-order moment of the random normalized residue (see Section~\ref{sec:Section5.7}) and two statistical moments: the normalized mean value of the random effective elasticity matrix $[\CC^\peff]$ and its coefficient of dispersion (see Section~\ref{sec:Section5.5}). Below we consider iteration $i$, and then $\bflambda^1,\ldots , \bflambda^{\, i-1},\bflambda^{\,i}$ are known.\\

\noindent (i) The random normalized residue $\rho_{\bflambda^{\,i}}$ whose realizations are defined by Eq.~\eqref{eq:eq136}, is an implicit function of $\bfH_{\bflambda^{\,i}}$, that we can rewrite as $\rho_{\bflambda^{\,i}}(\bfH_{\bflambda^{\,i}})$. The second-order moment of the normalized random residue is then written as $E\{(\rho_{\bflambda^{\,i}}(\bfH_{\bflambda^{\,i}}))^2 \}$ yielding
$h_\rho^c(\bfH_{\bflambda^{\,i}}) = \left ( \rho_{\bflambda^{\,i}}(\bfH_{\bflambda^{\,i}}) \right )^2$.
 Therefore, using the notation of Eq.~\eqref{eq:eq84}, we have
$E\{\,h_\rho^c(\bfH_{\bflambda^{\,i}})\} = b^c_\rho$,
in which we choose $b_\rho^c=1$ (this value is close to the value of the second-order moment of the random normalized residue  without constraint) and where $\bfeta\mapsto h_\rho^c(\bfeta): \RR^\nu\rightarrow \RR$ is a positive-valued implicit function.\\

\noindent (ii) The second constraint is given by Eq.~\eqref{eq:eq115}. Transforming the upper triangular  matrix of $[\CC_n^\peff] \in\MM_6^+$  in a $\RR^{21}$-vector, Eq.~\eqref{eq:eq115} is rewritten, similarly to Eq.~\eqref{eq:eq84}, as
$E\{\,\bfh_\CC^c(\bfH_{\bflambda^{\,i}})\} = \bfb_\CC^c$,
in which $\bfeta\mapsto \bfh_\CC^c(\bfeta): \RR^\nu\rightarrow \RR^{21}$ is an implicit function and where $\bfb_\CC^c\in\RR^{21}$ is the reshaping of the upper triangular matrix of $[\,{\underline\CC}_{\,n}^{\pexp}]\in\MM_6^+$ defined by Eq.~\eqref{eq:eq114}.\\

\noindent (iii) The last constraint is given by Eq.~\eqref{eq:eq119} that is rewritten, using Eq.~\eqref{eq:eq84}, as
$E\{\,h_\delta^c(\bfH_{\bflambda^{\,i}})\} = b^c_\delta$,
in which $b_\delta^c = ( \mu_\eff \, \delta^{\,\pexp} / \mu_\pexp)^2$ and where $\bfeta\mapsto h_\delta^c(\bfeta): \RR^\nu\rightarrow \RR$ is a positive-valued implicit function.\\

\noindent (iv) Finally, for given $\bflambda$, and in particular for $\bflambda = \bflambda^\psol$ yielding $\bfH^c= \bfH_\bflambda$, the constraint is defined by Eq.~\eqref{eq:eq84} with
\begin{equation}\label{eq:eq141}
\bfeta\mapsto\bfh^c(\bfeta) = (h_\rho^c(\bfeta), \bfh_\CC^c(\bfeta), h_\delta^c(\bfeta)): \RR^\nu\rightarrow\RR^{n_c} = \RR\times\RR^{21}\times\RR\, ,
\end{equation}
\begin{equation}\label{eq:eq142}
b^c = (b_\rho^c, \bfb_\CC^c, b_\delta^c) \in \RR^{n_c} = \RR\times\RR^{21}\times\RR\, ,
\end{equation}
in which $n_c=23$. Using the mathematical developments presented in \cite{Soize2021d}, it can be verified that $\bfh^c$ satisfies Hypothesis~\ref{hypothesis:1}. We will also assume that Eq.~\eqref{eq:eq7bis} holds.
Indeed, verifying \textit{a priori} that the components of $\bfh^c$ are algebraically independent  is very difficult for the considered problem. This hypothesis will indirectly be verified  by checking that the solution $\bflambda^\psol$ is well identified (see Proposition~\ref{proposition:1}-(c) and Section~\ref{sec:Section5.9}).
\subsection{Error function and convergence analysis of the sequence of MCMC generator}
\label{sec:Section5.9}
The error function is defined by Eqs.~\eqref{eq:eq87} and \eqref{eq:eq88}. The constraint on the random normalized residue (see Section~\ref{sec:Section5.8}-(i)) is introduced with $b_\rho^c=1$ in order to avoid the increasing of $E\{\,h_\rho^c(\bfH^c)\}$ with respect to $E\{\,h_\rho^c(\bfH_{\bflambda^{1}})\}$ (first iteration without constraints effects).
This constraint, which is taken into account in Algorithm~\ref{algorithm:1} for computing $\{\bflambda^i, i=1,\ldots, i_\pmax\}$, is not taken into account in the error function (see Eq.~\eqref{eq:eq87}) to identify the index $i_\psol$ (see Eq.~\eqref{eq:eq86}) of  the optimal value $\bflambda^\psol = \bflambda^{i_\psol}$ of $\bflambda^i$. Consequently, Eq.~\eqref{eq:eq87} is written, for $i\in \{1,\ldots , i_\pmax\}$,
\begin{equation}\label{eq:eq143}
\error(i) = \sqrt{\left( {\error_\CC(i)}\,/\,{\error_\CC(1)}\right )^2 + \left( {\error_\delta(i)}\,/\,{\error_\delta(1)}\right )^2} \, ,
\end{equation}
\begin{equation}\label{eq:eq144}
\error_\CC(i) = \Vert \, \bfb_\CC^c - E\{\,\bfh_\CC^c(\bfH_{\bflambda^{\,i}})\}\, \Vert \, / \,\Vert\, \bfb_\CC^c\Vert
\quad , \quad
\error_\delta(i) = \vert \, b_\delta^c - E\{\,h_\delta^c(\bfH_{\bflambda^{\,i}})\}\, \vert \, / \, b_\delta^c \, .
\end{equation}
For the three cases, SC1, SC2,and SC3, Fig.~\ref{fig:figure2} displays the error function $i\mapsto\error(i)$  defined by Eq.~\eqref{eq:eq143}, computed with the constrained learned set
$\curD_{\bfH_{\bflambda^{\, i}}}$ for $N=1000$, $2000$, $6000$, and $10\, 000$.
It can be seen that convergence is reached for  $N=10\, 000$ and that, at convergence, function $i\mapsto\error(i)$ is relatively smooth (that is not the case for $N=1000$). These graphs show a good illustration  of the convergence  of the sequence in $N$ of the MCMC generator using the statistical surrogate model $\hat\bfh^N$ of $\bfh^c$ (see Proposition~\ref{proposition:4}).
When the convergence is reached for $N=10\, 000$, Table~\ref{table:table2} gives the value of $i_\psol$ such that $\bflambda^\psol= \bflambda^{\, i_\psol}$ (see Eq.~\eqref{eq:eq86}) and the corresponding value $\error(i_\psol)$ of the error.
\begin{figure}[h]
\centering
\includegraphics[width=5.4cm]{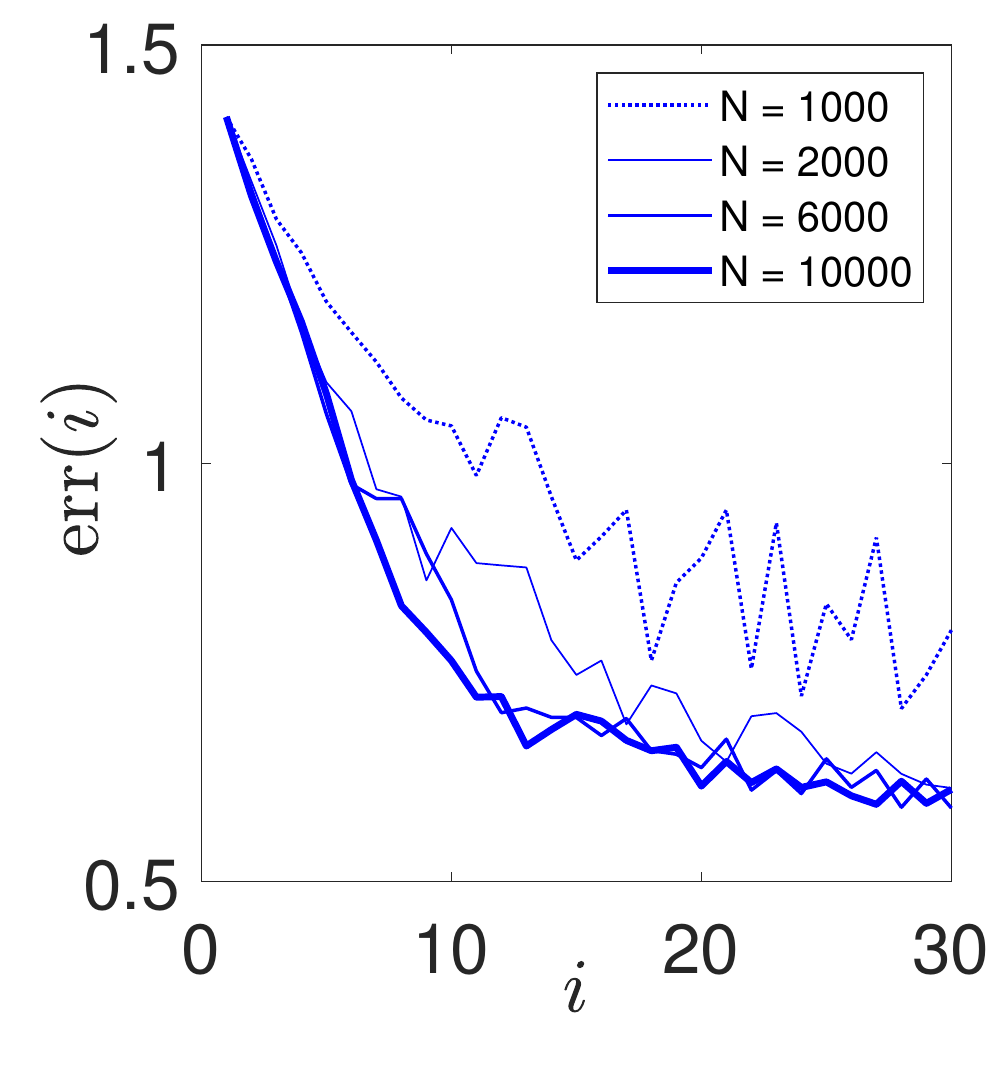}  \includegraphics[width=5.4cm]{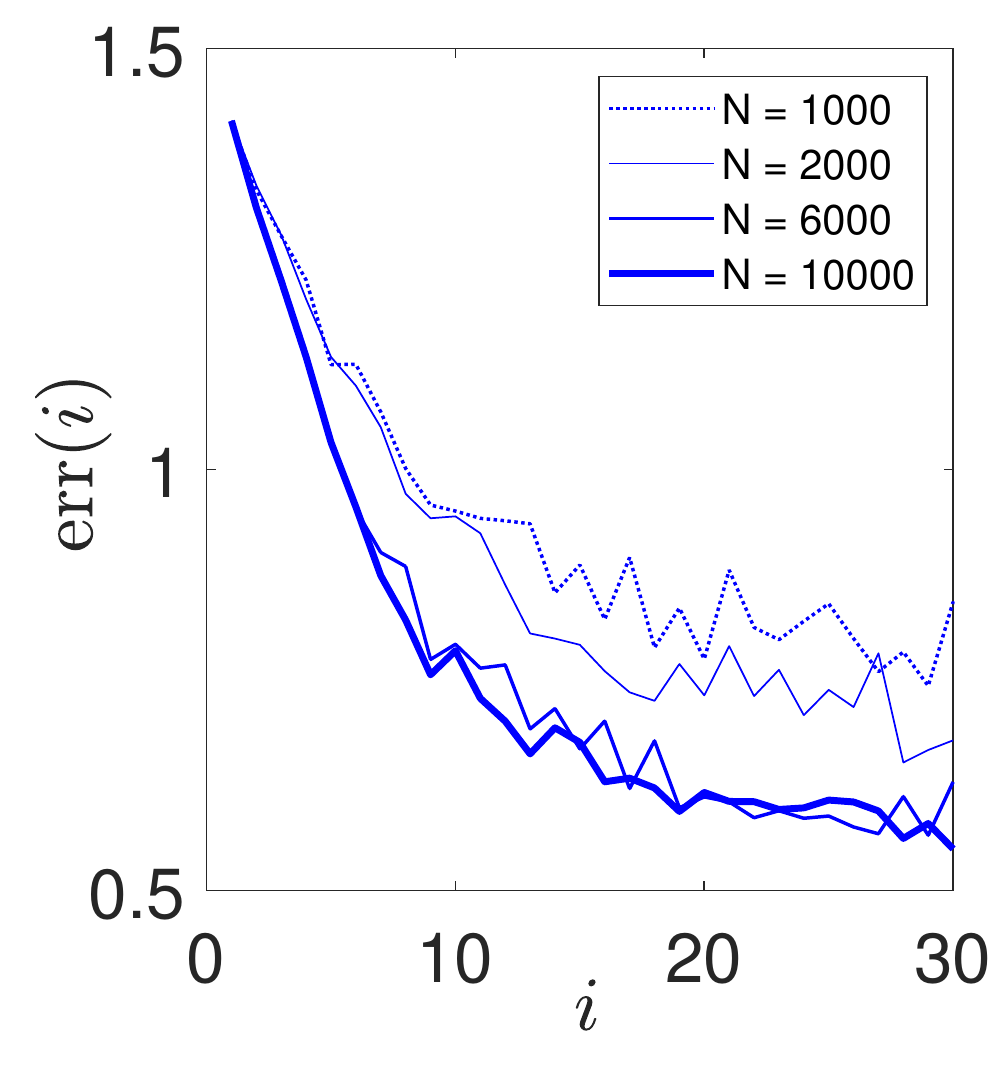} \includegraphics[width=5.4cm]{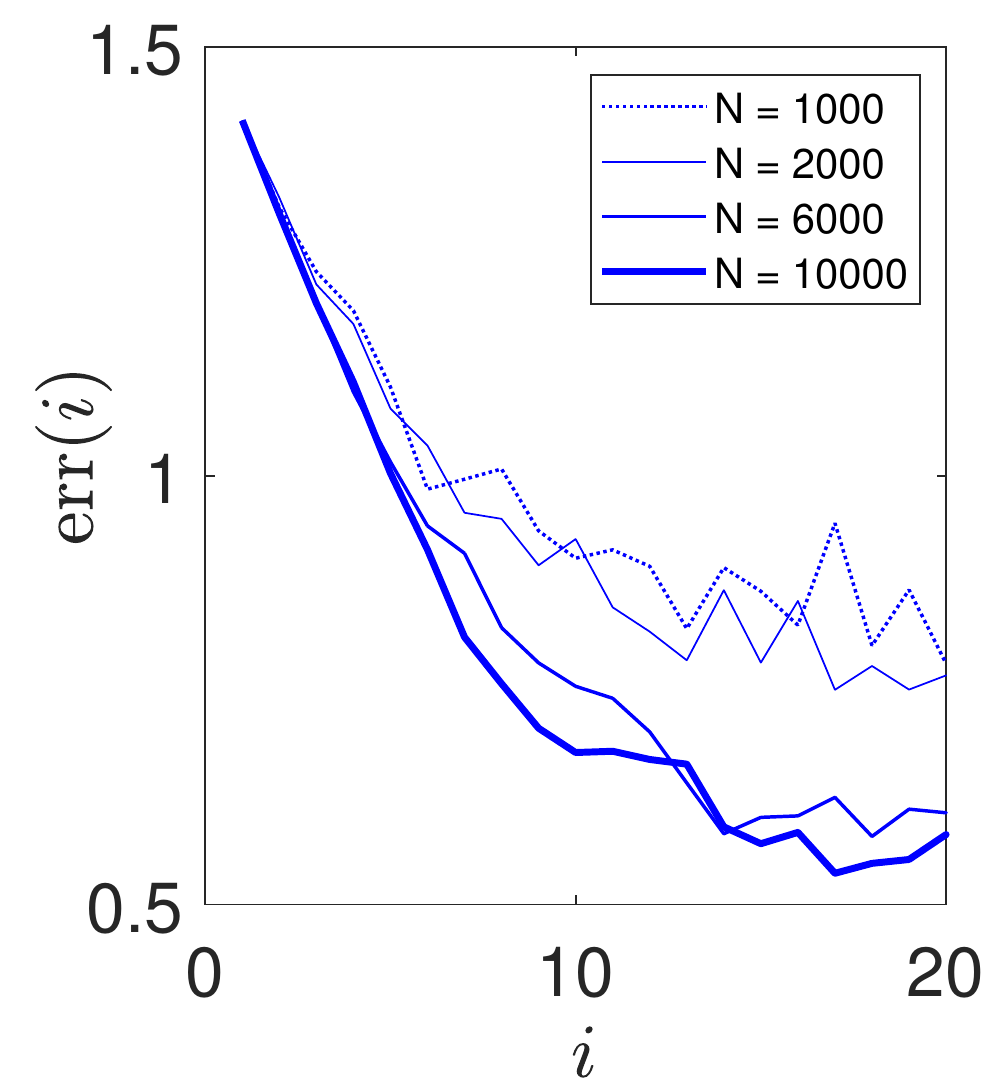}
\caption{For scale separation, SC1 (left figure), SC2 (central figure), and SC3 (right figure), graph of error function $i\mapsto\error(i)$ for $N=1000$, $2000$, $6000$, and $10\, 000$.}
\label{fig:figure2}
\end{figure}
\begin{table}[h]
  \caption{For cases SC1, SC2, and SC3 of scale separation, value $\error(i_\psol)$ of the error function for the solution
  $\bflambda^\psol = \bflambda_{i_\psol}$ computed with the constrained learned set for $N=10\, 000$.}\label{table:table2}
\begin{center}
 \begin{tabular}{|c|c|c|c|} \hline
    Case              & SC1     & SC2     &   SC3    \\
    \hline
    $i_\pmax$         &  30     &  30     &   20     \\
    $i_\psol$         &  27     &  30     &   17     \\
    $\error(i_\psol)$ &  0.5924 &  0.5495 &   0.5364 \\
    \hline
  \end{tabular}
\end{center}
\end{table}
\subsection{Residue and posterior second-order moments of the random effective elasticity matrix estimated with the constrained learned set}
\label{sec:Section5.10}
The posterior statistics of the random normalized residue  and the random effective elasticity matrix are estimated with the constrained learned set $\curD_{\bfH^c} = \curD_{\bfH_{\bflambda^{\psol}}}$ as a function of $N$ for the three cases, SC1, SC2, and SC3.

\noindent (i) The posterior second-order moment of the random normalized residue is $E\{\,\rho_c^2\}$ and compared to $1$, see Section~\ref{sec:Section5.8}-(i).

\noindent (ii) The posterior mean value is $[\,\underline{\CC}_{\,c}^\peff] = E\{[\CC_{\, c}^\peff]\}$ and is compared to $[\,\underline\CC^\pexp]$ (see Eq.~\eqref{eq:eq115} for the normalized version). We also consider the Frobenius norm $\Vert \, [\,\underline{\CC}_{\,c}^\peff]\, \Vert_F$ that is compared to $\Vert \, [\,\underline{\CC}^\pexp]\, \Vert_F$.

\noindent (iii) For the posterior dispersion coefficient of $[\CC_{\, c}^\peff]$, instead of comparing $\delta_c^{\,\peff}$ to $\delta^{\,\pexp}$ (see Eq.~\eqref{eq:eq118}),  it is more efficient to use the maximum likelihood consisting in comparing $\delta_{c,\pML}^{\,\peff}$ to $\delta^{\,\pexp}$ in which
\begin{equation}\label{eq:eq146}
\delta_{c,\pML}^{\,\peff} = \sqrt{\delta_{2,\pML}^{\,\peff}} \quad , \quad \delta_{2,\pML}^{\,\peff} = \max_{\delta_2} \, p_{\Delta_{2,C}^\peff}(\delta_2)\, ,
\end{equation}
where $p_{\Delta_{2,C}^\peff}$ is the pdf of the random variable $\Delta_{2,c}^\peff$ defined by Eq.~\eqref{eq:eq120}, at convergence $\bflambda = \bflambda^\psol$ (introduction of subscript "$c$"). For the three cases, SC1, SC2, and SC3,
Fig.~\ref{fig:figure3} displays the graph of the posterior pdf $\delta_2\mapsto p_{\Delta_{2,C}^\peff}(\delta2)$ of random variable $\Delta_{2,c}^\peff$, estimated with the constrained learned set  for $N=10\, 000$ and its prior counterpart estimated with the training set (constructed using the prior probability model).
Fig.~\ref{fig:figure4} (left figure) displays the graph of the Frobenius norm $N\mapsto \Vert\, [\,\underline\CC_{\, c}^\peff(N)]\, \Vert_F$ of the posterior mean value $[\,\underline{\CC}_{\,c}^\peff(N)] = E\{\, [\CC_{\,c}^\peff(N)]\}$ of the random effective elasticity matrix $[\CC_{\,c}^\peff(N)]$  as a function of $N$ and estimated using the constrained learned set, while  Fig.~\ref{fig:figure4} (right figure) shows the graph of the maximum likelihood $N\mapsto \delta_{c,\pML}^{\,\peff}(N)$ of the coefficient of dispersion of $[\CC_{\,c}^\peff(N)]$ defined by Eq.~\eqref{eq:eq146}.
Fig.~\ref{fig:figure5} shows the graph of the posterior pdf $r\mapsto p_{\rho_c}(r)$ of the random normalized residue $\rho_c$, estimated with the constrained learned set for $N=10\,000$ and its counterpart for the estimation performed with the training set (constructed using the prior probability model).

\noindent (iv) For cases SC1, SC2, and SC3, Table~\ref{table:table3} gives the posterior statistics computed with the constrained learned set for $N = 10\, 000$ (subscript "c"), the prior statistics computed with the training set (subscript "d"), and the targets (superscript "exp"), for the second-order moment of the random normalized residue, for the Frobenius norm of the mean value of the random effective elasticity matrix, and for the coefficient of dispersion of this random matrix.
For the same three cases, Table~\ref{table:table4} gives the values of the entries
of the mean matrices $[\, \underline\CC_{\, d}^\peff]$ computed with the training set, $[\,\underline\CC_{\,c}^\peff]$ computed with the constrained learned set for $N=10\, 000$, and $[\,\underline\CC^\pexp]$ for the targets. Note that entries $(4,5)$, $(4,6)$, and $(5,6)$, which are small with respect to the other entries, are not given.
\begin{figure}[h]
\centering
\includegraphics[width=4.4cm]{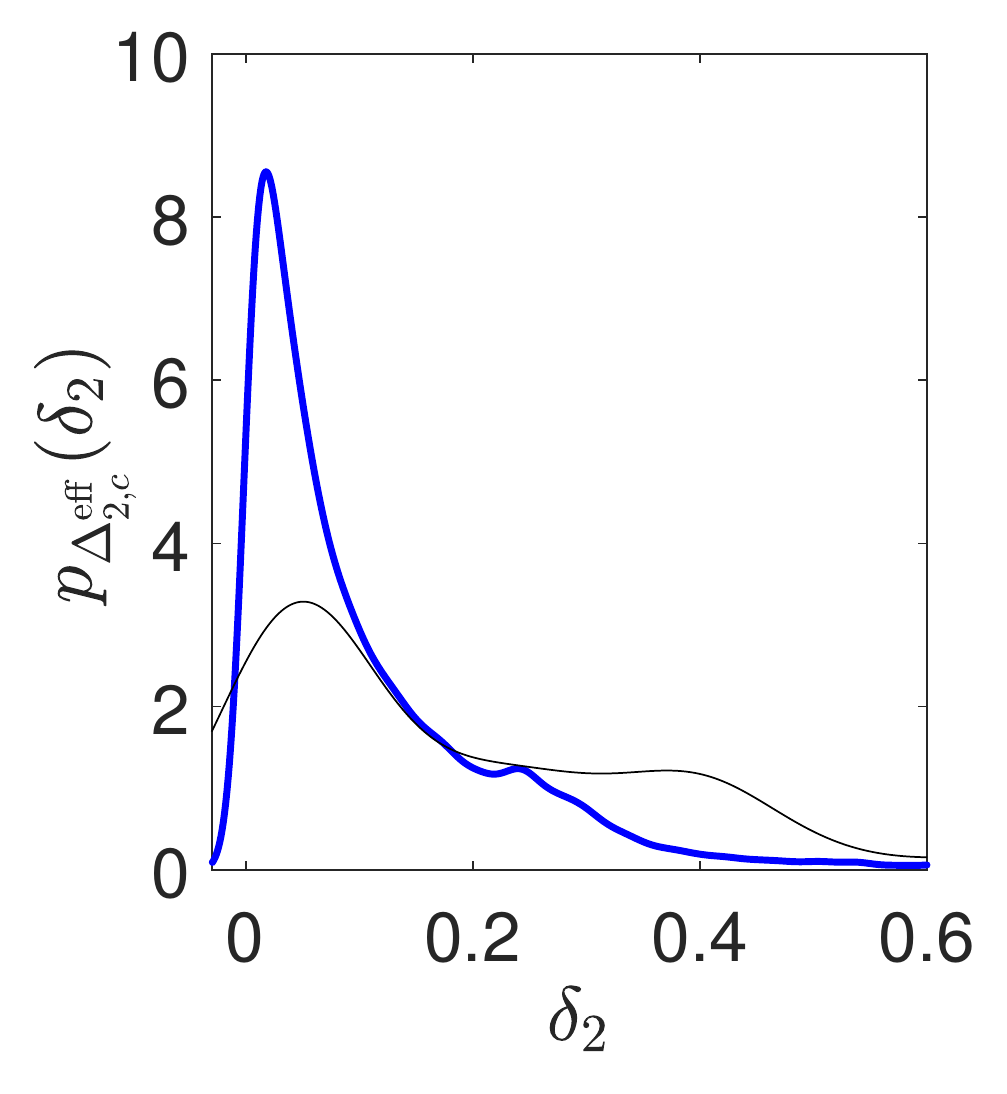}  \includegraphics[width=4.4cm]{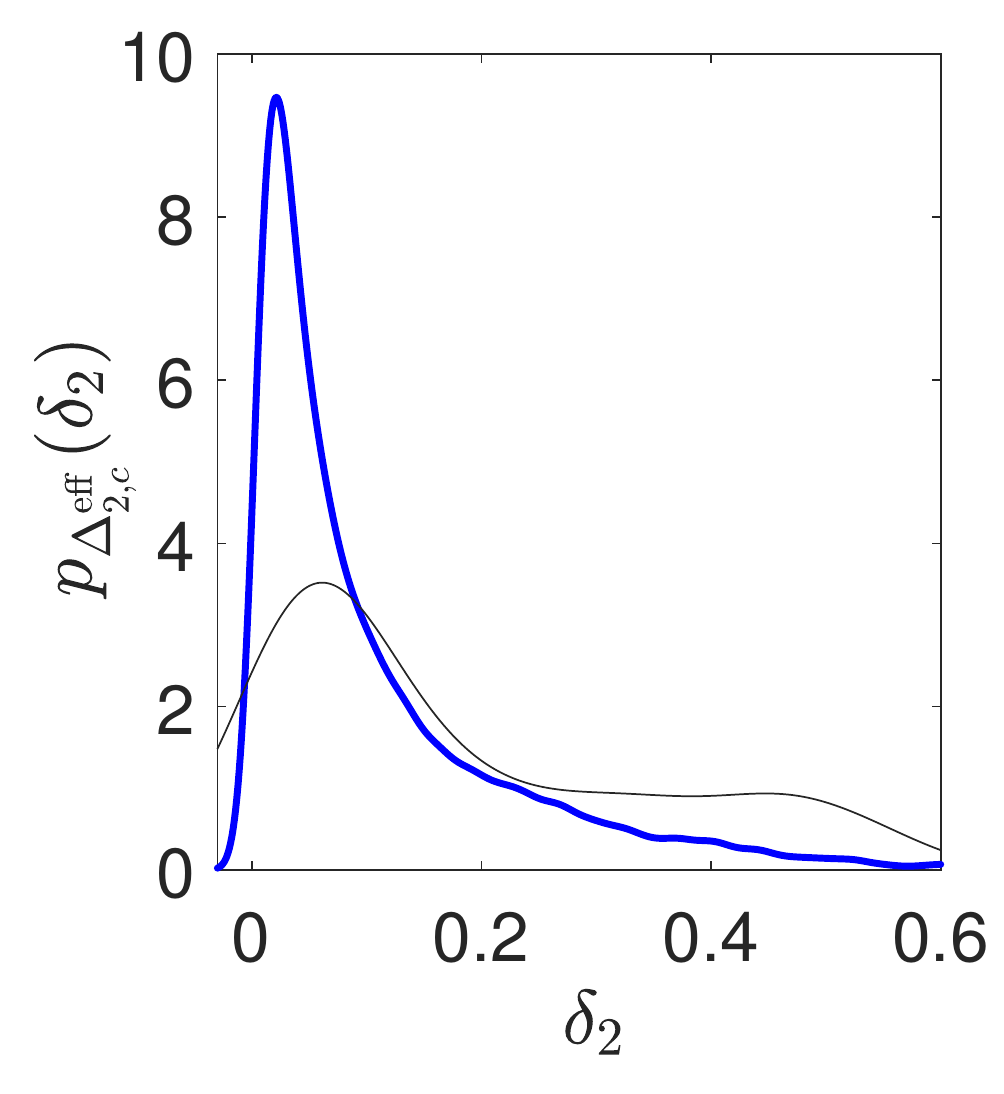} \includegraphics[width=4.4cm]{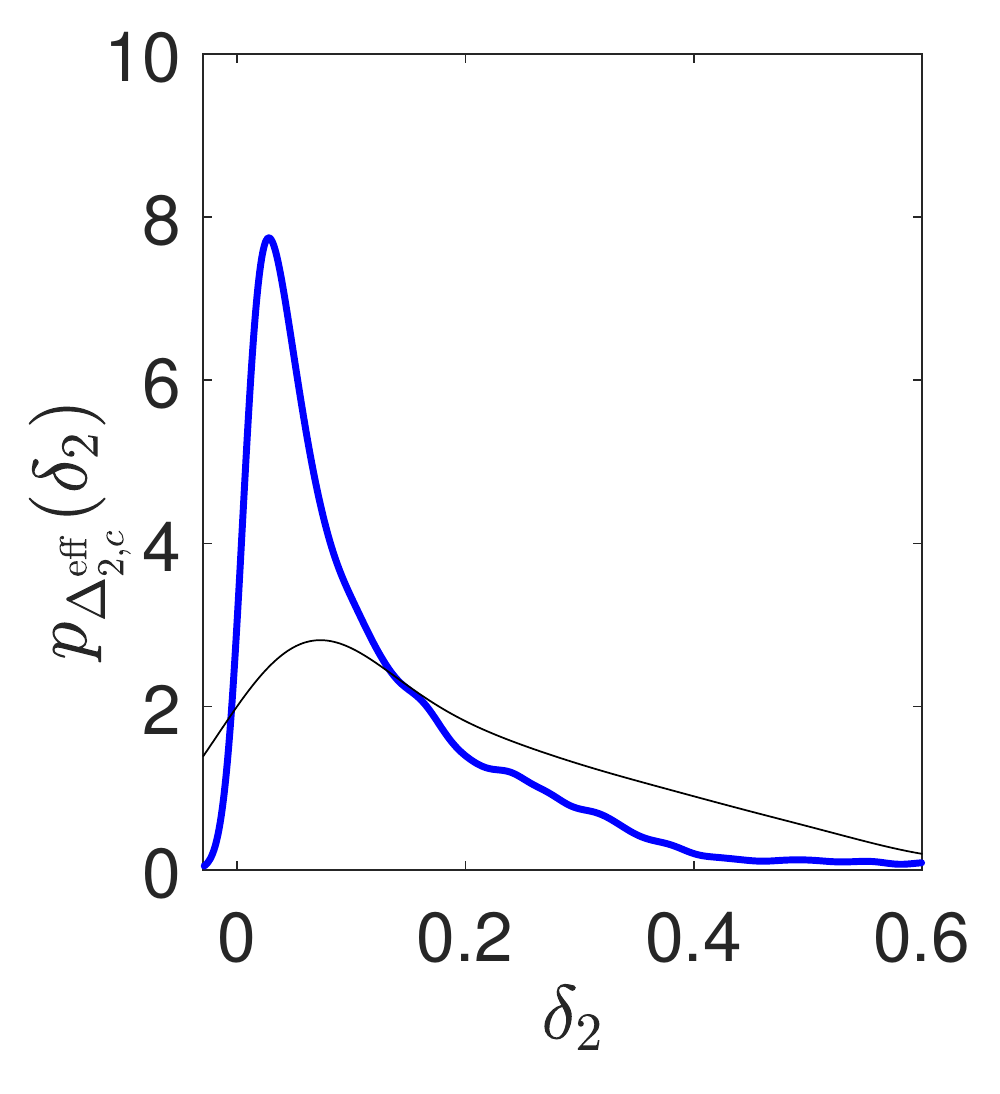}
\caption{For scale separation, SC1 (left figure), SC2 (central figure), and SC3 (right figure), graph of the posterior pdf of random variable $\Delta_{2,c}^\peff$ (thick blue line) estimated with the constrained learned set for $N=10\, 000$ and its prior counterpart (thin black line) corresponding to an estimation with the training set (constructed using the prior probability model).}
\label{fig:figure3}
\end{figure}
\begin{figure}[h]
\centering
\includegraphics[width=5.0cm]{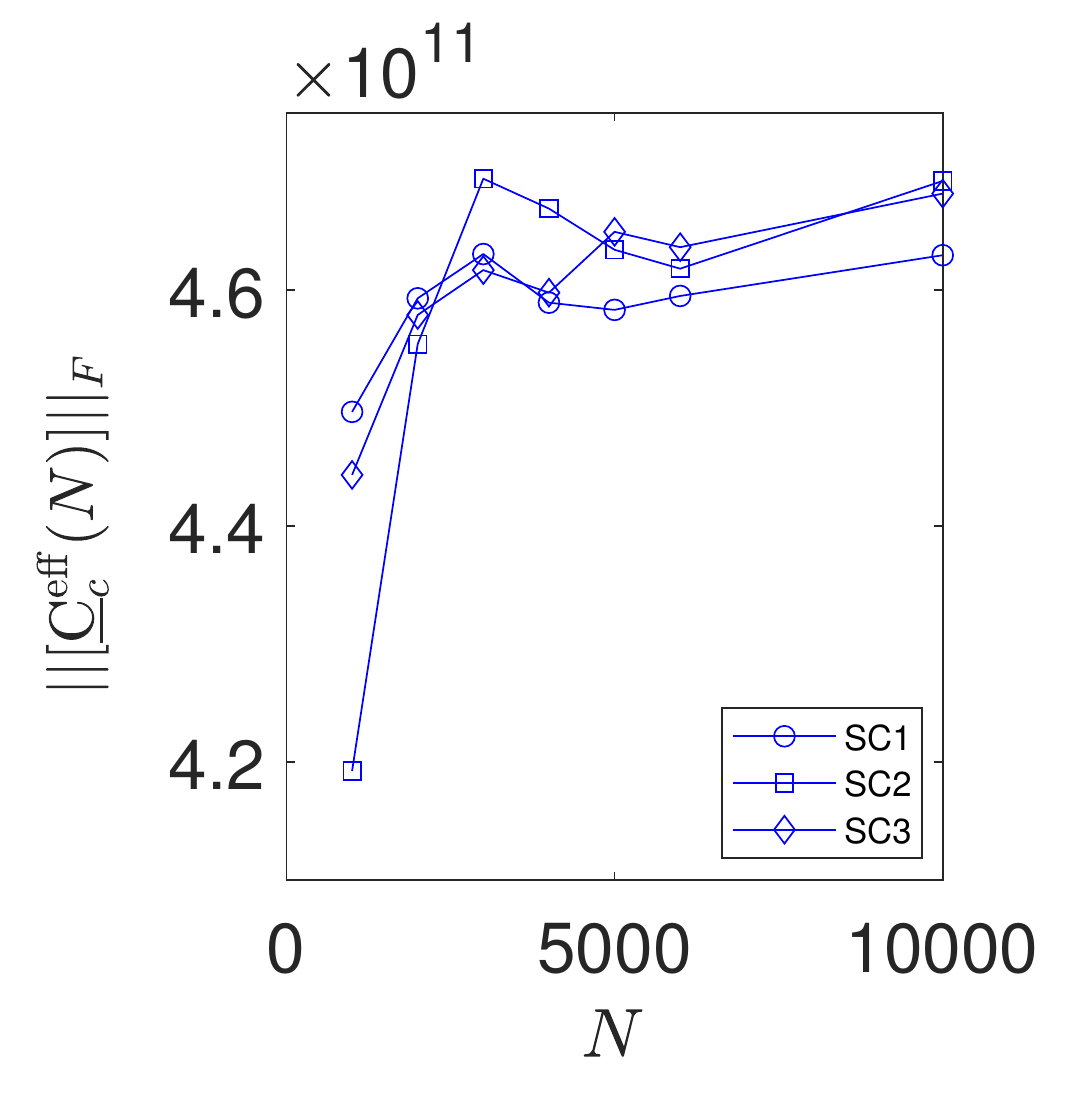}  \includegraphics[width=5.0cm]{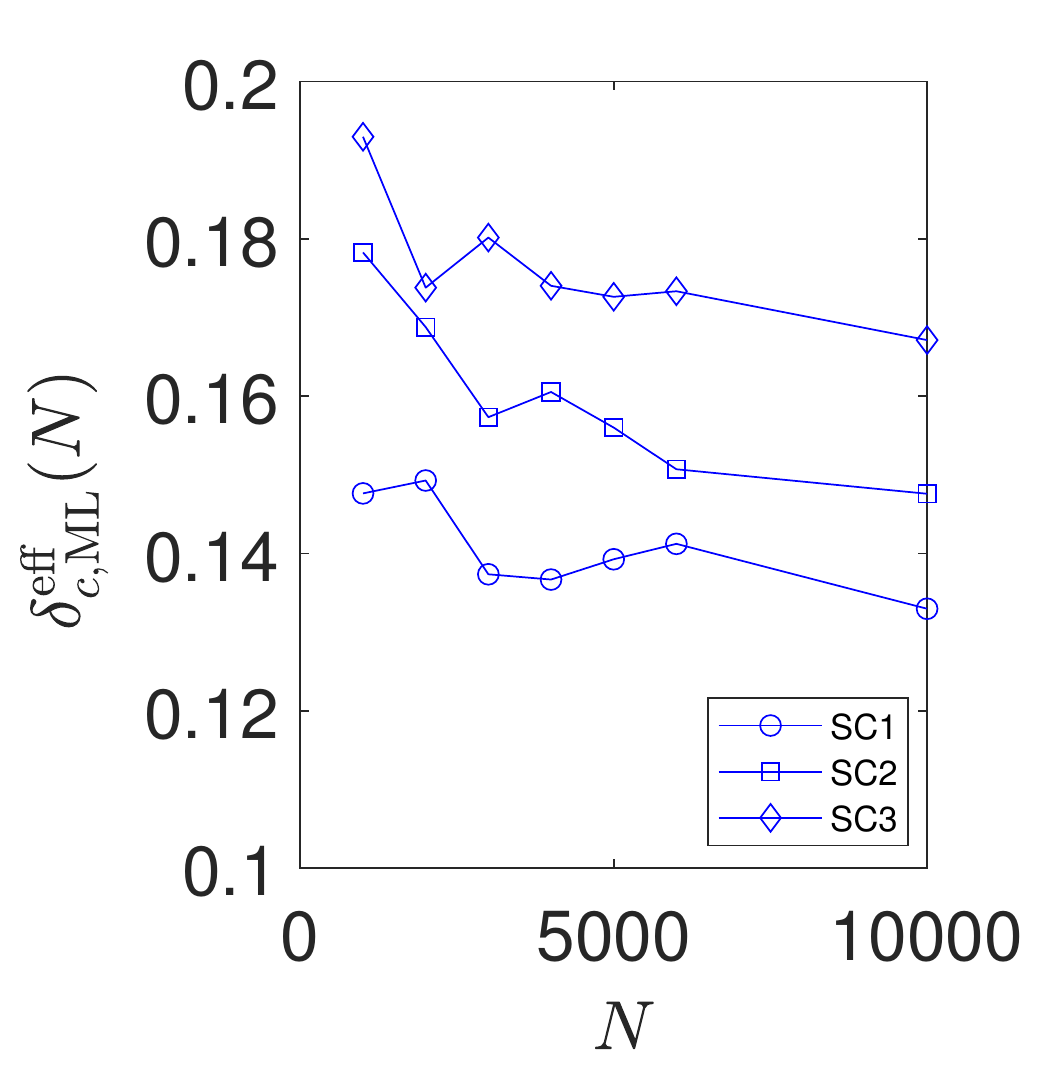}
\caption{For scale separation, SC1, SC2, and SC3, graph of $N\mapsto \Vert \,[ \,\underline\CC_{\,c}^\peff(N)] \, \Vert_F$ (left figure) and
$N\mapsto \delta_{c,\pML}^{\,\peff}(N)$ (right figure), computed with the constrained learned set.}
\label{fig:figure4}
\end{figure}
\begin{figure}[h]
\centering
\includegraphics[width=4.4cm]{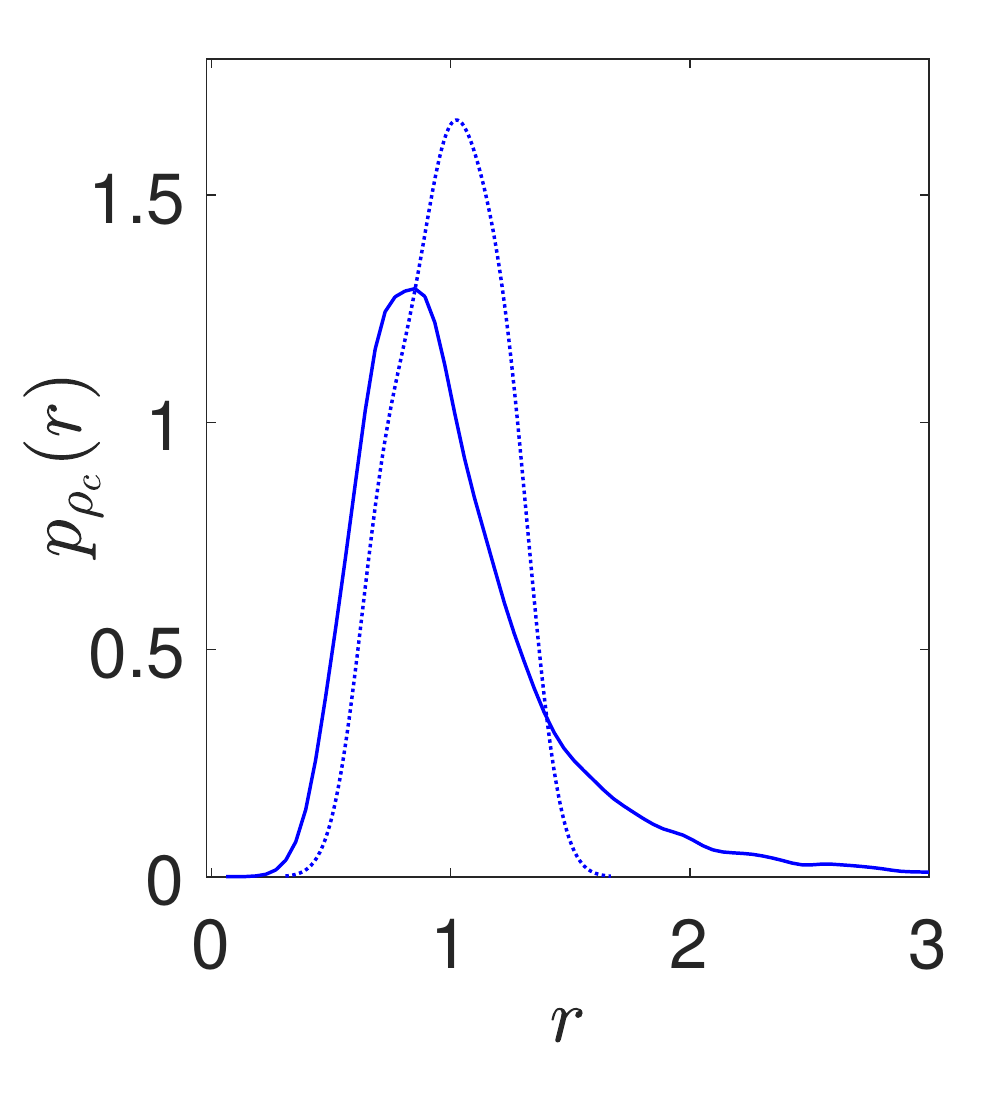}  \includegraphics[width=4.4cm]{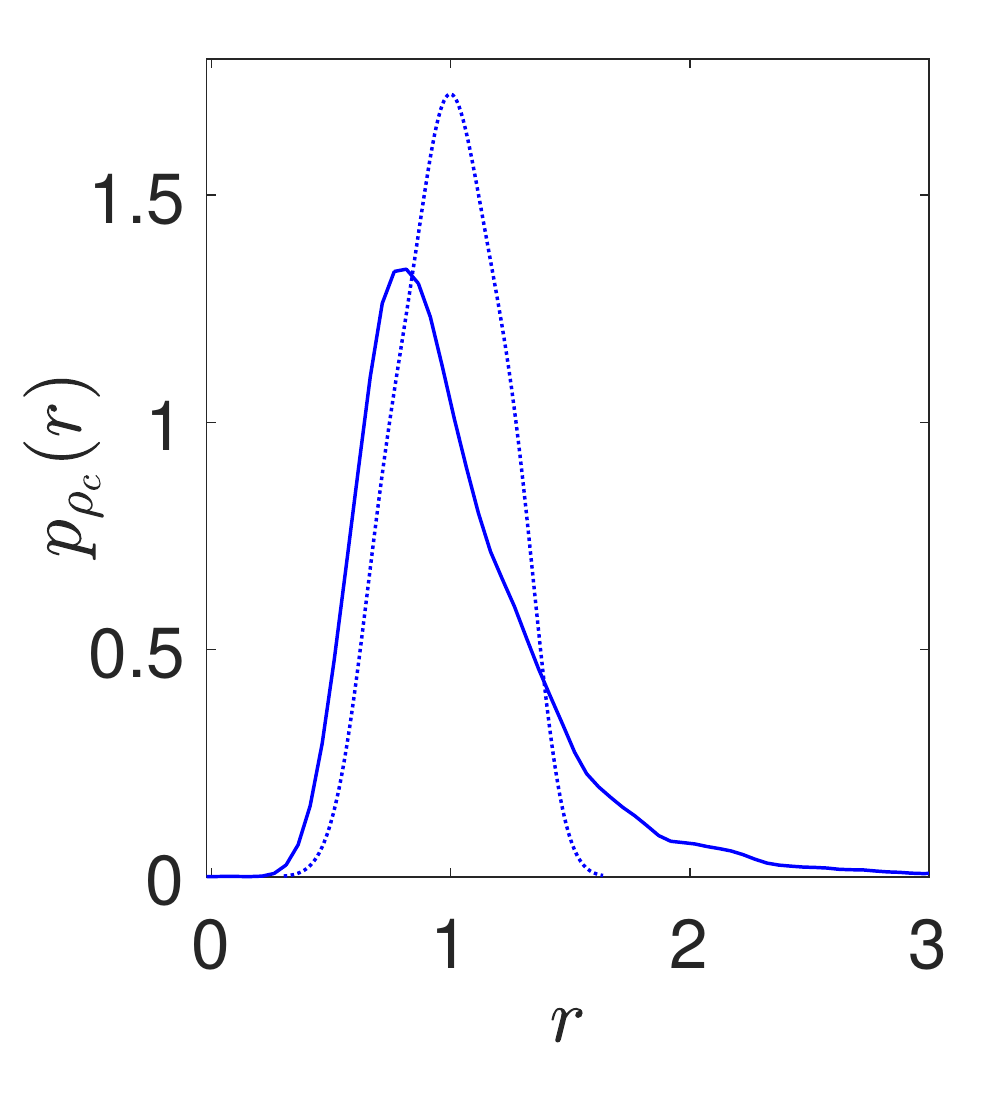} \includegraphics[width=4.4cm]{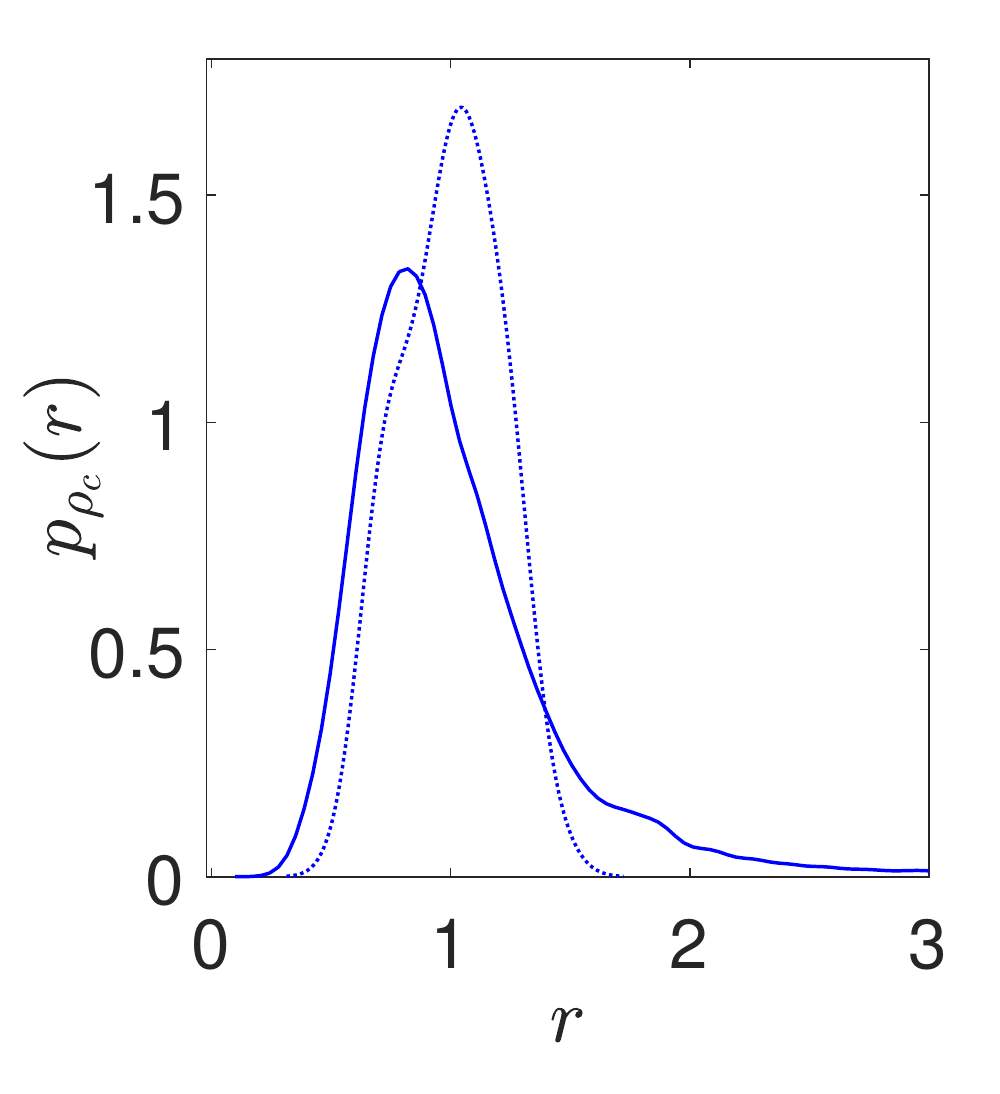}
\caption{For scale separation, SC1 (left figure), SC2 (central figure), and SC3 (right figure), graph of the posterior pdf of random normalized residue $\rho_c$ (solid line) estimated with the constrained learned set for $N=10\, 000$ and its prior counterpart (dotted line) corresponding to an estimation with the training set (constructed using the prior probability model).}
\label{fig:figure5}
\end{figure}
\begin{table}[h]
  \caption{For cases SC1, SC2, and SC3, posterior statistics computed  with the constrained learned set for $N=10\, 000$ (subscript "c"), prior statistics (subscript "d") computed with the training set, and targets.}\label{table:table3}
\begin{center}
 \begin{tabular}{|c|c|c|c|}
    \hline
                                                                         & SC1     & SC2     &   SC3    \\
    \hline
    $E\{\rho_c^2\}$                                                      &  1.2938 &  1.2687 &   1.2413 \\
    $b_\rho^c$                                                           &  1      &  1      &   1      \\
    \hline
    $\Vert \,[\, \underline\CC_{\, d}^\peff] \, \Vert_F \times 10^{11}$  &  4.2106 &  4.1925 &  4.1943  \\
    $\Vert \,[\, \underline\CC_{\, c}^\peff] \, \Vert_F \times 10^{11}$  &  4.6294 &  4.6923 &  4.6816  \\
    $\Vert \,[\, \underline\CC^\pexp] \, \Vert_F \times 10^{11}$         &  4.6317 &  4.6549 &  4.6706  \\
    \hline
    $\delta^{\,\peff}_{d,\pML}$                                              &  0.2257 &  0.2469 &  0.2701  \\
    $\delta^{\,\peff}_{c,\pML}$                                              &  0.1329 &  0.1476 &  0.1671  \\
    $\delta^{\,\pexp}$                                                       &  0.0946 &  0.1374 &  0.1825  \\
    \hline
  \end{tabular}
\end{center}
\end{table}
\begin{table}[h]
  \caption{For cases SC1, SC2, and SC3, values of $[\,\underline\CC_{\, d}^\peff]_{ij}$ computed with the training set, $[\,\underline\CC_{\, c}^\peff]_{ij}$ computed with the constrained learned set for $N=10\, 000$, and $[\,\underline\CC^\pexp]_{ij}$
  for the targets.}\label{table:table4}
\begin{center}
 \begin{tabular}{|c||c|c|c||c|c|c||c|c|c|} \hline
                         & \multicolumn{3}{c||}{SC1}   & \multicolumn{3}{c||}{SC2}  & \multicolumn{3}{c|}{SC3}  \\
    \hline
    Entries of  &        &        &        &        &        &        &        &        &        \\
    $(6\times 6)$        & $[\,\underline\CC_{\, d}^\peff]$ & $[\,\underline\CC_{\, c}^\peff]$ & $[\,\underline\CC^\pexp]$
                         & $[\,\underline\CC_{\, d}^\peff]$ & $[\,\underline\CC_{\, c}^\peff]$ & $[\,\underline\CC^\pexp]$
                         & $[\,\underline\CC_{\, d}^\peff]$ & $[\,\underline\CC_{\, c}^\peff]$ & $[\,\underline\CC^\pexp]$    \\
     matrix              &        &        &        &        &        &        &        &        &        \\
    \hline
    $(1,1)$ & 2.0904 & 2.2600 & 2.2652 & 2.0792 & 2.2914 & 2.2810 & 2.0465 & 2.2751 & 2.2946 \\
    $(1,2)$ & 0.7427 & 0.8809 & 0.8753 & 0.7269 & 0.8982 & 0.8809 & 0.7140 & 0.8874 & 0.8824 \\
    $(1,3)$ & 0.7458 & 0.8804 & 0.8745 & 0.7324 & 0.8983 & 0.8800 & 0.7381 & 0.8999 & 0.8826 \\
    $(2,2)$ & 2.0832 & 2.2603 & 2.2668 & 2.0786 & 2.2946 & 2.2846 & 2.0917 & 2.2830 & 2.2822 \\
    $(2,3)$ & 0.7451 & 0.8802 & 0.8734 & 0.7486 & 0.8950 & 0.8754 & 0.7471 & 0.9045 & 0.8808 \\
    $(3,3)$ & 2.0839 & 2.2647 & 2.2680 & 2.0777 & 2.2841 & 2.2697 & 2.1038 & 2.2928 & 2.2812 \\
    $(4,4)$ & 0.6702 & 0.6909 & 0.6958 & 0.6785 & 0.6985 & 0.7003 & 0.6835 & 0.6976 & 0.7027 \\
    $(5,5)$ & 0.6714 & 0.6903 & 0.6949 & 0.6732 & 0.6950 & 0.6960 & 0.6726 & 0.6872 & 0.6980 \\
    $(6,6)$ & 0.6713 & 0.6924 & 0.6960 & 0.6727 & 0.6958 & 0.6970 & 0.6749 & 0.6933 & 0.6991 \\
    \hline
  \end{tabular}
\end{center}
\end{table}
\subsection{Posterior probability model of parameters}
\label{sec:Section5.11}
The prior probability model concerns the $\RR^{n_g}$-valued random variable $\bfcurG$ that corresponds to the spatial discretization of the $\RR^{21}$-valued random field $\bfG$, and the $\RR^3$-valued random variable $\bfW$ that is related (see Eq.~\eqref{eq:eq110}) to the random bulk modulus $C_\bulk$ and the random shear modulus $C_\shear$, which control the elasticity tensor of the mean isotropic model at mesoscale, and to the dispersion coefficient $\delta_\CC$ that controls the level of anisotropic statistical fluctuations of the random apparent elasticity field at mesoscale (see Section~\ref{sec:Section5.2}). For cases SC1, SC2, and SC3, Fig.~\ref{fig:figure6} displays the posterior pdf $c\mapsto p_{C_\bulk}(c)$ of $C_\bulk$ (left figure), $c\mapsto p_{C_\shear}(c)$ of $C_\shear$ (central figure), and
$c\mapsto p_{\delta_\CC}(c)$ of $\delta_\CC$ (right figure),
estimated with the constrained learned set for $N=10\, 000$, and their prior counterparts estimated with the training set constructed using the prior probability model. It should be noted that for each one of the random variables $C_\bulk$, $C_\shear$, and $\delta_\CC$, its prior probability model is the same for the three cases and consequently, does not depend on the case contrary to its posterior probability model that depends on it.
\begin{figure}[h]
\centering
\includegraphics[width=4.4cm]{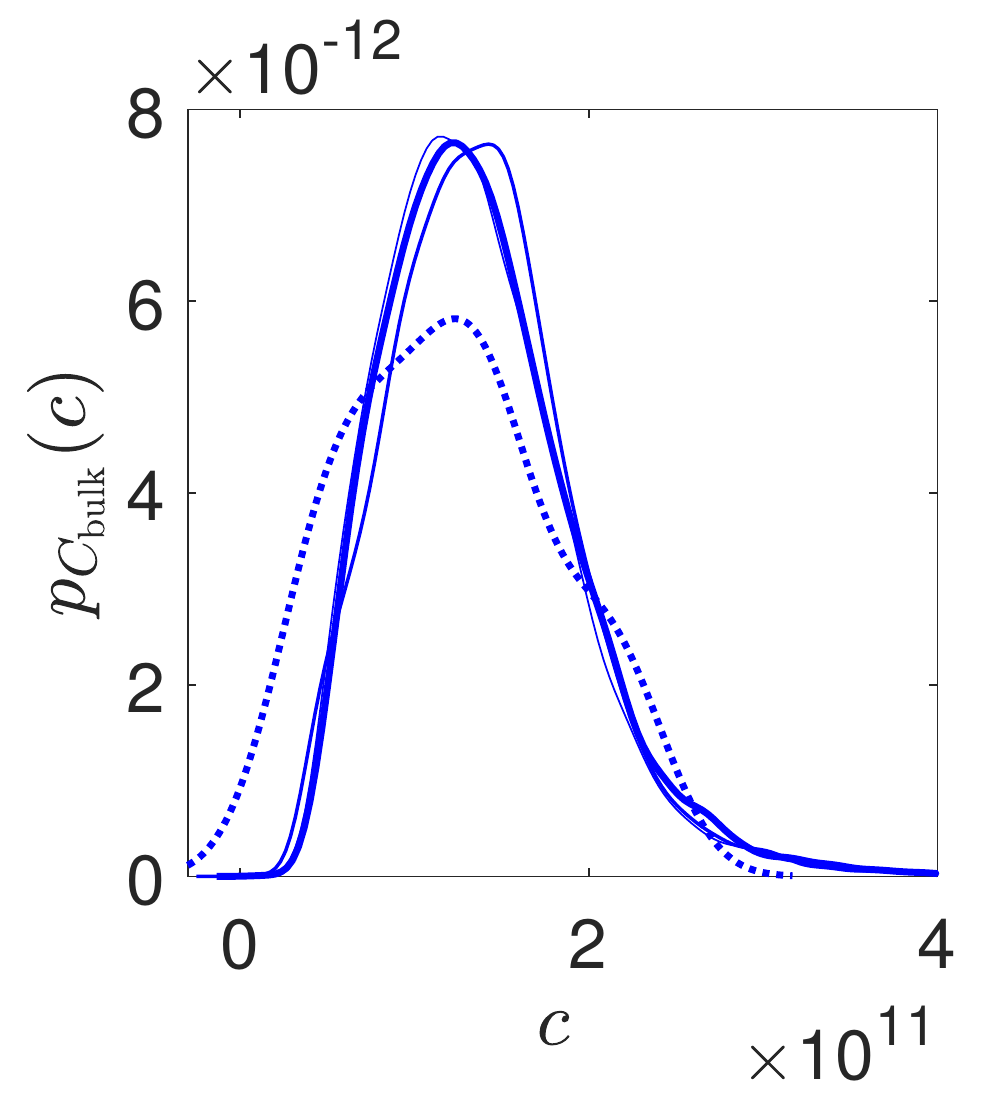}  \includegraphics[width=4.4cm]{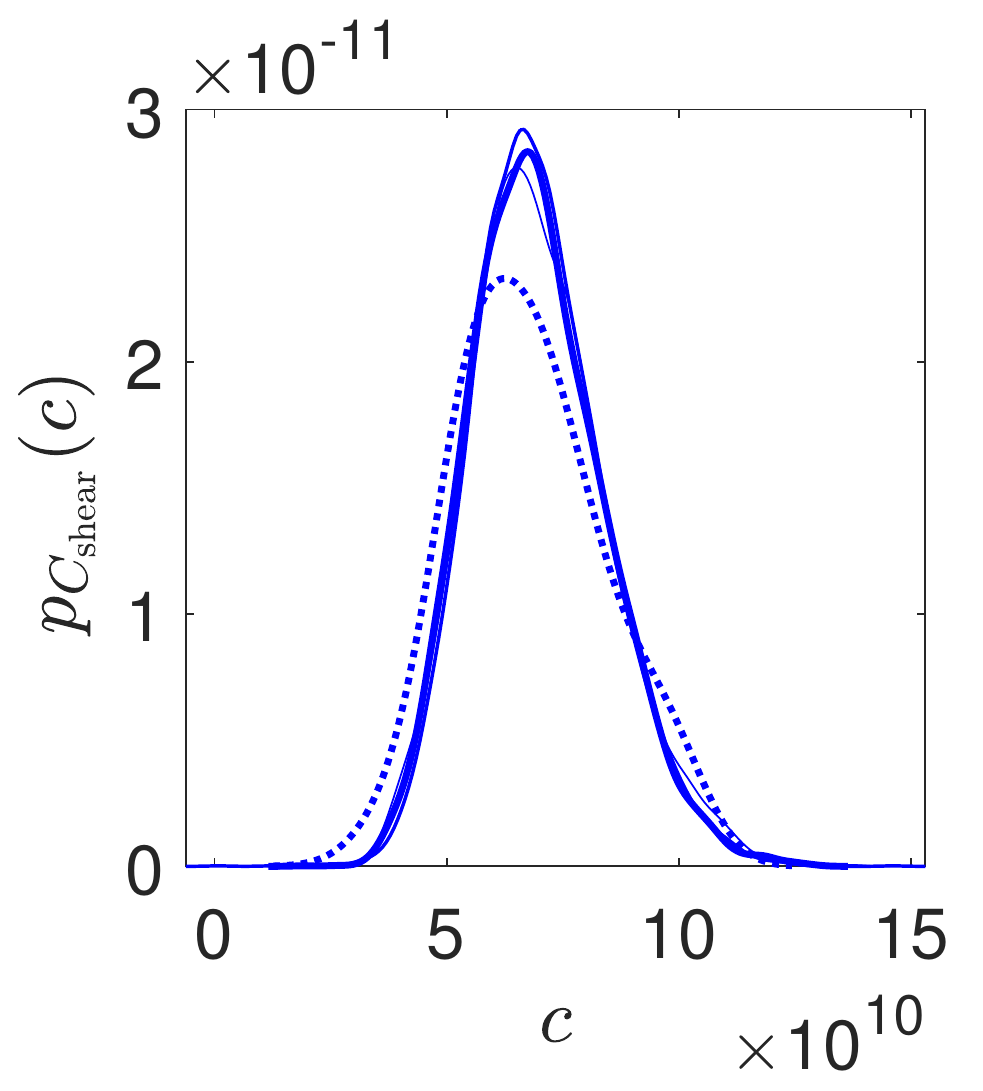} \includegraphics[width=4.4cm]{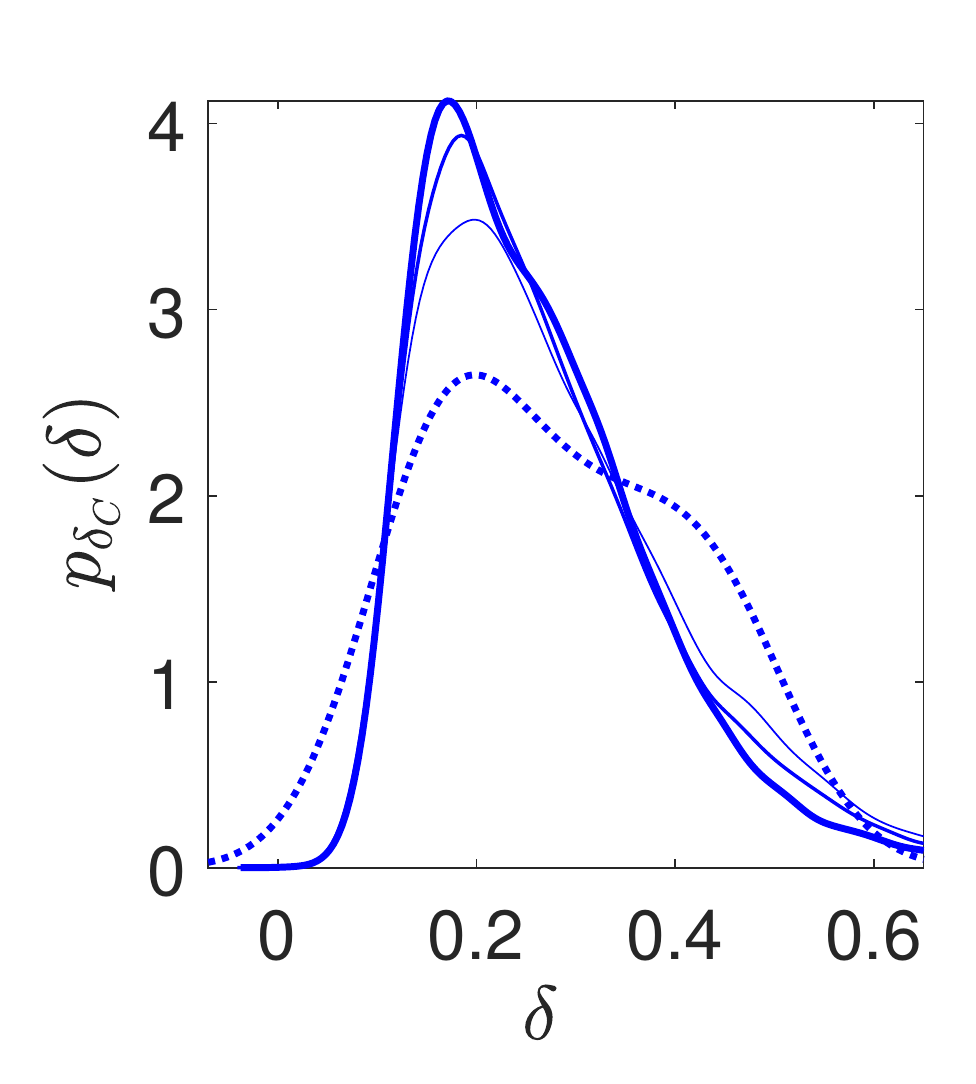}
\caption{Graphs of the posterior pdf $c\mapsto p_{C_{\rm{bulk}}}(c)$ (left figure), $c\mapsto p_{C_{\rm{shear}}}(c)$ (central figure),
and $\delta\mapsto p_{\delta_\CC}(\delta)$ (right figure), estimated with the constrained learned set for $N=10\, 000$,
for cases, SC1 (thin solid line), SC2 (med solid line), and SC3 (thick solid line), and the corresponding prior pdf estimated with the training set constructed using the prior probability model (dotted line).}
\label{fig:figure6}
\end{figure}
\subsection{Discussion about the presented results}
\label{sec:Section5.12}
\noindent (i) The results obtained with the posterior model (see Tables~\ref{table:table3} and \ref{table:table4}) show that the constrained learned set significantly improves the prior probability model used for generating the training set. The comparison of the posterior statistics with the targets are good.

\noindent (ii) As explained in Section~\ref{sec:Section5.9} the residue constraint is taken into account for the generation of the constrained learned set, but does not intervene in the estimation of the optimal value $\bflambda^\psol$ of $\bflambda$ using the error function. Nevertheless and as expected, Fig.~\ref{fig:figure5} and Table~\ref{table:table3} show that the residue is controlled and stayed small with respect to the reference (the training) for the optimal solution.

\noindent (iii) The convergence of the sequence of MCMC generators with respect to the number of points generated in the constrained learned set is good as shown by Figs.~\ref{fig:figure2} and \ref{fig:figure4} in accordance to Proposition~\ref{proposition:4}.

\noindent (iv) The dispersion of the target, measured by the value of $\delta^{\,\pexp}$, is smaller than the one of the prior probability model (with which the training set has been constructed). So, it was expected that the posterior random effective elasticity matrix be less dispersed than the one exhibited by the prior probability model. The results confirm this point as it can be seen in Tables~\ref{table:table3} and \ref{table:table4} and also in Figs.~\ref{fig:figure3} and  \ref{fig:figure6}.

\noindent (v) Comment (ii) is extended as follows. It can be seen that the target is "well" reached for the mean value of the random effective elasticity matrix (see Table~\ref{table:table4}), while it is  "less well" reached for its coefficient of dispersion computed using the maximum likelihood.
This can be explained by the fact that the second-order moment  $E\{\rho_c^2\}$ of the random normalized residue of the equation is kept small during the probabilistic learning of the posterior probability measure, what prevents the dispersion coefficient from reaching its target.
We have carried out numerical tests without imposing the constraint related to $E\{\rho_c^2\}$, which should stay close to 1. We have observed that the dispersion coefficient "reasonably" reached its target but that $E\{\rho_c^2\}$ did not stay close to $1$ but took on significant values greater than $1$. There is indeed a choice of objective between (1) correctly satisfying the constraint on the dispersion coefficient while degrading the value of $E\{\rho_c^2\}$ or (2) preserving a small value of $E\{\rho_c^2\}$ to the detriment of perfectly reaching the target for the dispersion coefficient. We have chosen to present the compromise consisting in taking into account the constraint on the residue during the probabilistic learning process, but the error function that we have chosen to identify the optimal value $\bflambda^\psol$ of $\bflambda$ does not take into account the constraint on the residue.

\noindent (vi) A last comment concerns the effects of no scale separation. As expected, for the three cases SC1, SC2, and SC3, Table~\ref{table:table3} shows that the coefficient of dispersion  is significant and increases with the spatial correlation lengths of the random apparent elasticity field at mesoscale, inducing statistical fluctuations of the effective elasticity tensor at macroscale. It should be noted that, even for the case SC1, for which homogenization in the plane of the plate (domain $\Omega$) is guaranteed (the correlation lengths $\underline L_{C1}$ and $\underline L_{C2}$ being much lower than $1$), this is not the case for the correlation length $\underline L_{C3}$ that is equal to the thickness of the plate. Consequently, there is no  homogenization at the macroscopic scale and the effective elasticity tensor remains random and is not deterministic.
\section{Conclusions}
\label{sec:Section6}
In this paper, we have presented a general methodology to estimate a posterior probability model for a stochastic boundary value problem from a prior probability model. The given targets are statistical moments for which the underlying realizations are not available. Under these conditions, it has been proposed to use the Kullback-Leibler divergence minimum principle for estimating the posterior probability measure, given the prior probability measure and the constraints related to the targets of  the statistical moments. We have proposed the construction of a statistical surrogate model of the implicit mapping that represents the constraints. The constrained learned set, which defines the posterior model, is constructed using only a training set constituted of a small number of points. A mathematical analysis of the proposed methodology has been presented. We have defined the required mathematical hypotheses, which have allowed us to prove the  convergence of introduced approximations. We have also given all the necessary numerical elements, which facilitate the implementation of the methodology in a parallel computing framework.

The application presented to illustrate the proposed theory is also, as such, a contribution to the three-dimensional stochastic homogenization of heterogeneous linear elastic media in the case of a non-separation of the microscale and macroscale, that is to say, when there are significant statistical fluctuations in the effective elasticity tensor at macroscale. The prior stochastic model of the elasticity tensor field at the mesoscopic scale is an advanced model, recently proposed, which takes into account uncertainties on its spectral measure. In addition to the statistical moments of the random effective elasticity tensor, for the construction of the posterior probability measure by probabilistic learning inference, the second-order moment of the random normalized residue of the stochastic partial differential equation has been added as a constraint. This constraint guarantees that the algorithm seeks to bring the statistical moments closer to their targets while preserving a small residue. The results obtained are those which were expected and give a very good illustration of the theory developed for a non-trivial application.
%
%

\end{document}